\documentclass[10pt]{article}

\usepackage[utf8]{inputenc}

\usepackage{epsf}
\usepackage{amsmath}

\allowdisplaybreaks

\usepackage[showframe=false]{geometry}
\usepackage{changepage}

\usepackage{epsfig}
\usepackage{amssymb}

\usepackage{amsthm}
\usepackage{setspace}
\usepackage{cite}
\usepackage{mcite}

\usepackage[inline]{trackchanges}

\usepackage{algorithmic}  
\usepackage{algorithm}

\usepackage{shadow}
\usepackage{fancybox}
\usepackage{fancyhdr}

\usepackage{color}
\usepackage[usenames,dvipsnames,svgnames,table]{xcolor}

\definecolor{mypurple}{rgb}{.4,.0,.5}

\usepackage{xcolor}

\usepackage{color}

\definecolor{darkgreen}{rgb}{0, 0.4,0}

\definecolor{purplebrown}{rgb}{0.5,0.1,0.6}

\definecolor{ultclupcol}{rgb}{0.1,0.5,0.5}

\definecolor{mytrycolor}{rgb}{0.5,0.7,0.2}

\definecolor{ultclupcola}{rgb}{.5,0,.5}

\newcommand{\bl}[1]{\textcolor{blue}{#1}}

\newcommand{\prp}[1]{\textcolor{purple}{#1}}
\newcommand{\yellow}[1]{\textcolor{yellow}{#1}}

\definecolor{shadebrown}{rgb}{0.1,0.1,0.9}
\definecolor{lightblue}{rgb}{0.2,0,1}

\usepackage{fancybox}
\usepackage{graphicx}
\usepackage{epstopdf}
\usepackage{epsfig}
\usepackage{wrapfig}
\usepackage{subfigure}

\usepackage{xcolor}

\usepackage{tcolorbox}
\tcbuselibrary{skins}

\usepackage{amsmath}
    
    \usepackage{graphicx}
    \usepackage{dcolumn}
    \usepackage{bm}
    \usepackage{amsmath}

\newtcbox{\xmybox}{on line,
arc=7pt,
before upper={\rule[-3pt]{0pt}{10pt}},boxrule=0pt,
boxsep=0pt,left=6pt,right=6pt,top=0pt,bottom=0pt,enhanced, coltext=blue, colback=white!10!yellow}

\newtcbox{\xmyboxa}{on line,
arc=7pt,
before upper={\rule[-3pt]{0pt}{10pt}},boxrule=0pt,
boxsep=0pt,left=6pt,right=6pt,top=0pt,bottom=0pt,enhanced, colback=white!10!yellow}

\newtcbox{\xmyboxb}{on line,
arc=7pt,
before upper={\rule[-3pt]{0pt}{10pt}},boxrule=1pt,colframe=darkgreen!100!blue,
boxsep=0pt,left=6pt,right=6pt,top=0pt,bottom=0pt,enhanced, colback=white!10!yellow}

\newtcbox{\xmyboxc}{on line,
arc=7pt,
before upper={\rule[-3pt]{0pt}{10pt}},boxrule=.7pt,colframe=blue!100!blue,
boxsep=0pt,left=6pt,right=6pt,top=0pt,bottom=0pt,enhanced, coltext=blue, colback=white!10!yellow}

\newtcbox{\xmytboxa}{on line,
arc=7pt,
before upper={\rule[-3pt]{0pt}{10pt}},boxrule=.0pt,colframe=pink!50!yellow,
boxsep=0pt,left=6pt,right=6pt,top=0pt,bottom=0pt,enhanced, coltext=white, colback=blue!40!red}

\newtcbox{\xmytboxb}{on line,
arc=7pt,
before upper={\rule[-3pt]{0pt}{10pt}},boxrule=.0pt,colframe=pink!50!yellow,
boxsep=0pt,left=6pt,right=6pt,top=0pt,bottom=0pt,enhanced, coltext=white, colback=white!40!green}

\usepackage[hyphens]{url}

\usepackage[colorlinks=true,
            linkcolor=black,
            urlcolor=blue,
            citecolor=purple]{hyperref}

\usepackage{breakurl}

\makeatletter
\newcommand\subsubsubsection{\@startsection{paragraph}{4}{\z@}{-2.5ex\@plus -1ex \@minus -.25ex}{1.25ex \@plus .25ex}{\normalfont\normalsize\bfseries}}
\newcommand\subsubsubsubsection{\@startsection{subparagraph}{5}{\z@}{-2.5ex\@plus -1ex \@minus -.25ex}{1.25ex \@plus .25ex}{\normalfont\normalsize\bfseries}}
\makeatother

\def\y{{\bf y}}

\def\x{{\bf x}}

\def\x{{\mathbf x}}

\def\u{{\bf u}}

\def\x{{\bf x}}
\def\y{{\bf y}}

\def\c{{\bf c}}

\def\tr{\mbox{Tr}}

\def\tr{{\rm tr}\,}

\def\diag{{\rm diag}\,}

\def\be{\begin{equation}}
\def\ee{\end{equation}}
\def\ba{\left[\begin{array}}
\def\ea{\end{array}\right]}

\def\u{{\bf u}}

\def\x{{\bf x}}
\def\y{{\bf y}}

\def\c{{\bf c}}

\def\1{{\bf 1}}

\def\0{{\bf 0}}

\def\vecw{\mbox{vec}}
\def\rankw{\mbox{rank}}
\def\diag{\mbox{diag}}

\def\bU{\bar{U}}

\def\bV{\bar{V}}

\def\mR{{\mathbb R}}
\def\mC{{\mathbb C}}
\def\mN{{\mathbb N}}

\def\mP{{\mathbb P}}

\def\calV{{\cal V}}
\def\calU{{\cal U}}

\def\lp{\left (}
\def\rp{\right )}

\newtheorem{theorem}{Theorem}
\newtheorem{corollary}{Corollary}

\newtheorem{lemma}{Lemma}

\begin{document}

\begin{singlespace}


\title {\textbf{Causal Inference (C-inf) --- \emph{closed} form \emph{worst} case typical phase transitions}
}
\author{
\textsc{Agostino Capponi \footnote{e-mail: {\tt ac3827@columbia.edu}}}\quad  \textsc{Mihailo Stojnic \footnote{e-mail: {\tt flatoyer@gmail.com}}} \quad 
\\
{Department of Industrial Engineering and Operations Research}\\
{Columbia University, New York, NY 10027, USA}
 }
\date{}
\maketitle

\centerline{{\bf Abstract}} \vspace*{0.1in}

In this paper we establish a mathematically rigorous connection between  \bl{\textbf{\emph{Causal inference (C-inf)}}} and the \bl{\textbf{\emph{low-rank recovery (LRR)}}}. 
Using Random Duality Theory (RDT) concepts developed in \cite{StojnicICASSP10var,StojnicISIT2010binary,StojnicICASSP09} and novel mathematical strategies related to free probability theory, we obtain the \bl{\textbf{\emph{exact explicit}}} typical (and achievable) worst case phase transitions (PT). These PT precisely separate scenarios where causal inference via LRR is possible from those where it is not. We supplement our mathematical analysis with numerical experiments that confirm the theoretical predictions of PT phenomena, and further show that the two closely match for fairly small sample sizes. We obtain simple closed form representations for the resulting PTs, which highlight direct relations between the low rankness of the target C-inf matrix and the time of the treatment. 
 Hence, our results can be used to determine the range of C-inf's typical applicability.

\vspace*{0.25in} \noindent {\bf Index Terms: Causal inference; Random duality theory; Algorithms; Matrix completion; Sparsity}.

\end{singlespace}

\section{Introduction}
\label{sec:back}

The \bl{\textbf{\emph{Causal inference (C-inf)}}} discipline deals with design of methods for estimating causal effects in panel data settings, where a subset of the units are exposed to a treatment during some time periods. The goal is estimating counterfactual outcomes, i.e., that outcome for those units were the treatment not been applied for that period of time.

Casual inference plays
a key role in decision making, and is essential in business decisions, network design, medical sciences, and many others. It allows answering questions of the form: What would happen to a data center's latency if a new congestion control algorithm were not used?  What would have been the systolic blood pressure of a patient if the new drug were not given to her?

This problem of estimating the counterfactual appears in many disciplines, including economics, finance, health, and social sciences (see, e.g. \cite{RoseRub83,Rub06,ImbRub15,ADHsynth10,DoudImb16,Xucinf17,HerRob10}), and computer science (see, e.g. \cite{PearlBar19,PearlCausBook09,PearlSMack18,PearlSurv09}). The increasing availability of big data through digital services and smart sensors calls makes it possible to design efficient algorithmic techniques to address these fundamental questions in counterfactual estimation.

The econometrics and social sciences communities have proposed three main approaches to causal inference: 1) the unconfoundedness (see, e.g. \cite{RoseRub83,ImbRub15}); 2) the synthetic control (see, e.g. \cite{ADHsynth10,DoudImb16,Abadsynth19}); and 3) the matrix completion (see, e.g. \cite{ABDIK21, Agarwal2021,KallusNIPS}. Matrix completion methods build upon the foundation works of \cite{	CR09matcomp,Rechtmatcomp11,CPmatcomp10}). Perhaps unexpectedly, all three methods heavily rely on mathematical, statistical, and ultimately algorithmic concepts with very deep roots in information theory. Our work is positioned within the third line of work that mathematically resembles the \bl{\textbf{\emph{matrix completion (MC)}}} problem.

Our main contribution is to  develop a rigorous connection between \bl{\textbf{\emph{Causal inference (C-inf)}}} from observational data and the \bl{\textbf{\emph{low-rank recovery (LRR)}}} problem in compressed sensing. Methodologically, we integrate Random Duality Theory (RDT) concepts developed in \cite{StojnicICASSP10var,StojnicISIT2010binary,StojnicICASSP09} with novel mathematical strategies related to free probability theory, and manage to obtain the \bl{\textbf{\emph{exact explicit}}} typical (and achievable) worst case phase transitions (PT). These PT provide a precise separation between regions of the parameter space where the counterfactual can be perfectly estimated via LRR from those regions where this is not possible for large sample sizes. Our numerical study complements our theoretical predictions by showing that theory and numerical simulation closely match even if the sample size is fairly small.

\section{Mathematics of causal inference}
\label{sec:mbcinf}
To put our contribution into a proper context, we begin by presenting the explicit \prp{\textbf{\emph{causal inference (C-inf) $\leftrightarrow$ matrix completion (MC)}}}  connection.

To introduce the most basic mathematical description of the main C-inf concepts we adopt the standard \emph{matrix completion (MC)} terminology. As is well known the matrix completion problems belong to a class of the so-called \emph{low-rank recovery problems (LRR)} which themselves belong to a broader class of the so-called \emph{structured recovery (SR)} problems. In its most general way the description of the structured recovery problems starts by assuming that one has access to the observation vector $\y\in\mR^m$ given by the following
\begin{eqnarray}\label{eq:linsys1a}
\y_i=f_i(A_{i,:}\x_{sol}),
\end{eqnarray}
where $A\in\mR^{m\times n}$ is a known system matrix (with rows $A_{i,:},1\leq i\leq m$), $\x_{sol}\in\mR^{n}$ is the unknown vector, and $f_i(\cdot): \mR\longrightarrow \mR$ are known real functions. In the treatments that will be of our interest here the functions $f_i(\cdot)$ will be assumed to be identically linear, i.e. it will be assumed that $f_i(x)=x$. Then the structured recovery assumes utilizing the $\x_{sol}$ 's \emph{a priori known} structure to eventually recover it. Most often, that boils down practically to finding efficient algorithmic designs that take as inputs the observation vector $\y$ and the system matrix $A$ and successfully output $\x_{sol}$ or its a sufficiently close approximation. Moreover, under efficient algorithms one typically views those that run preferably provably in polynomial time. What typically differentiates between various forms of the SR problems are the structures that $\x_{sol}$ possesses as well as the structure of the system matrix $A$. Along the same lines, what typically differentiates the level of success (or usefulness) that some of these forms might achieve is the ultimate theoretical and practical capabilities of the algorithms employed for their solving.

The structured recovery problems have been the subject of an extensive research for a better half of the last century and many beautiful results appeared regarding their various theoretical and practical aspects. However, it is the emergence of \bl{\emph{\textbf{compressed sensing (CS)}}} about 20 years ago that almost singlehandedly made a key transformational change in how these problems and the surrounding research are perceived. As a result of such a perceptional change, the interest in the structured recovery, its importance, popularity, and ultimate usefulness skyrocketed to the heights unseen ever before. It is, of course, hard to pinpoint exactly what could be the secret behind the compressed sensing meteoric rise to success, but the overall conceptual simplicity probably contributed to some degree. We refer for more on the CS invention, simplicity, and further developments to e.g. \cite{DonohoPol,CRT,BayMon10,DonMalMon09,StojnicCSetam09,StojnicUpper10,StojnicICASSP10var,StojnicICASSP10block,SPH,StojnicJSTSP09,StojnicISIT2010binary}. At the same time, since (mathematically speaking) the CS is not precisely the main topic of our interest here, we, before proceeding further, just briefly mention that it basically assumes the above mentioned structured recovery setup with the sparsity of $\x_{sol}$ being the underlying structuring.

\subsection{Low rank recovery (LRR)}
\label{sec:lrr}

While the mathematical problems typically seen in compressed sensing will not exactly match the key mathematical problems of this paper, the above mentioned low-rank recovery (LRR) ones will. To introduce the LRR problems, we first note that almost everything mentioned above for the generic structured recovery remains in place in the LRR scenarios. In fact, there will be only one key difference compared to the standard CS setup. The $\x_{sol}$ (which is a sparse vector in the CS context) will within the LRR considerations be viewed as a vectorized unknown matrix $X_{sol}$ and consequently the imposed \emph{a priori known} structure will be the low-rankness of $X_{sol}$. This, of course, is nothing but the sparsity of the vector of the singular values of $X_{sol}$. In other words, in compressed sensing the unknown vector $\x_{sol}$ itself is sparse, whereas in the LRR the vector of the singular values of the unknown matrix $X_{sol}$ is sparse.

In more mathematical terms, the LRR can then be described in the following way. One first starts with the above mentioned vectorizing of matrix $X_{sol}$
\begin{eqnarray}\label{eq:posmcxsol}
\x_{sol}=\vecw(X_{sol}),
\end{eqnarray}
with $\vecw(\cdot)$ stacking its matrix argument columns one after another (starting from the very first one) into a column vector. Assuming $X_{sol}\in\mR^{n\times n}$, trivial dimensional adjustments then give $A\in\mR^{m\times n^2}$. The low-rankness ($\rankw(X_{sol})=k\leq n$) and the corresponding sparsity are imposed via the singular value decomposition (SVD)
\begin{eqnarray}\label{eq:mcsvd1}
X=U\Sigma V^T,
\end{eqnarray}
where
\begin{eqnarray}\label{eq:mcsvd2}
\sigma(X)\triangleq\diag(\Sigma) \quad \mbox{and} \quad U^TU=I_{n\times n} \quad \mbox{and} \quad V^TV=I_{n\times n},
\end{eqnarray}
with $I_{n\times n}$ being the identity matrix of size $n\times n$ and $\diag(\cdot)$ being the operator that extracts the diagonal from its matrix argument and puts it into a column vector. When clear from the context, we will abbreviate and write just $\sigma$ instead of $\sigma(X)$. Moreover, $\sigma_i(X)$ will stand for the $i$-th component of $\sigma(X)$ (with $\sigma_i$ often being its shorter version). As is of course well known, the elements of $\sigma\in\mR^n$ are precisely the above mentioned singular values of $X$ and the number of nonzero such elements is precisely the rank of $X_{sol}$. It is probably obvious, but we state it to ensure the overall clarity, that in typical SVD treatments in the mathematical literature one will often find in (\ref{eq:mcsvd2}) $I_{k\times k}$ as a replacement for $I_{n\times n}$. In other words, one will often find that the underlying identity matrix is of size $k\times k$ instead of size $n\times n$. The reason for our choice is to ensure a complete parallelism between the LRR and the compressed sensing. Basically, this choice makes the underlying vector of the singular values visibly sparse (by this definition it will automatically have $n-k$ zeros) and as such more in parallel with the corresponding sparse vector structure one typically finds in the compressed sensing setup. To further maintain the parallelism with compressed sensing, we also find it useful to introduce $\ell_p^*(X)$ as the so-called $\ell_p^*$ quasi-norm of $X$
\begin{eqnarray}\label{eq:mclinsys1b1}
\ell_p^*(X)\triangleq \ell_p(\sigma(X)), p\in\mR_+,
\end{eqnarray}
where, $\ell_p(\cdot)$ is the standard vector $\ell_p$ (quasi-) norm, and to note that the following useful matrix-vector limiting $\ell_p(\cdot)$ connections also hold
\begin{eqnarray}\label{eq:mclinsys1c}
\ell_0^*(X_{sol})\triangleq \ell_0(\sigma(X_{sol}))=\|\sigma(X_{sol})\|_0=\lim_{p\longrightarrow 0}\|\sigma(X_{sol})\|_p= \lim_{p\longrightarrow 0}\ell_p(\sigma(X_{sol}))=\lim_{p\longrightarrow 0}\ell_p^*(X_{sol}).
\end{eqnarray}
One can then restate (\ref{eq:linsys1a}) adapted to fit the LRR context as
\begin{eqnarray}\label{eq:mclinsys1b}
\y_i=A_{i,:}\x_{sol}=A_{i,:}\vecw(X_{sol}) \quad \mbox{where} \quad \ell_0^*(X_{sol})=\ell_0(\sigma(X_{sol}))=\|\sigma(X_{sol})\|_0=k, k\in \mN,
\end{eqnarray}
and $A_{i,:}$ being the $i$-th row of $A$. Finally one can summarize the LRR mechanism as:

{\small \begin{tcolorbox}
[beamer,sidebyside,title=\textbf{Generic LRR},lower separated=false, fonttitle=\bfseries,
coltext=black , interior style={left color=orange!10!white, right color=blue!10!white},title style={left color=black, right color=red!70!orange!30!white}]
\begin{tcolorbox}
[beamer,title=\textbf{Observations -- forming $\y$ from $X_{sol}$},lower separated=false, fonttitle=\bfseries,
coltext=black , colback=yellow!70!orange!40!white,title style={left color=black, right color=red!70!orange!30!white}]
Given $A$ and $X_{sol}$ create $\y$ as
\begin{eqnarray}\label{eq:lrr1}
\y=A\vecw(X_{sol}), \ell_0^*(X_{sol})=k, k\in \mN.
\end{eqnarray}
 \end{tcolorbox}
\tcblower
\begin{tcolorbox}
[beamer,title=\textbf{Observations -- forming $\y$ from $X_{sol}$},lower separated=false, fonttitle=\bfseries,
coltext=black , colback=yellow!70!orange!40!white,title style={left color=black, right color=red!70!orange!30!white}]
Given $\y$ and $A$ from (\ref{eq:lrr1}) can one efficiently recover $X_{sol}$ back?
 \end{tcolorbox}
\end{tcolorbox}
}

It is not that difficult to see that $\ell_0^*(X_{sol})$ basically serves as a counting function that counts the number of the nonzero singular values of $X_{sol}$. Its analogy with the function $\ell_0(\x_{sol})$, that appears in the compressed sensing setup and counts the number of the nonzero elements of the unknown sparse vector, is rather obvious. Moreover, both such an analogy and the underlying sparsity that enables it suggest an algorithmic way that one can try to employ to ultimately recover $X_{sol}$. By the above definition, the LRR is the inverse problem for (\ref{eq:lrr1}) and can be posed as the so-called $\ell_0^*$-minimization optimization problem. Since such a minimization is a highly non-convex problem it is not known to be solvable in polynomial time. To design polynomially solvable heuristics one then, following into the compressed sensing footsteps, introduces the tightest convex norm relaxation concept and replaces the $\ell_0^*$- with the $\ell_1^*$-minimization.

\begin{center}
\tcbset{beamer,lower separated=false, fonttitle=\bfseries,width=3.4in, coltext=black ,
colback=yellow!70!orange!40!white,title style={left color=cyan!40!black!80!purple, right color=red!60!yellow!40!orange!80!white},
width=(\linewidth-4pt)/4, equal height group=AT,before=,after=\hfill,fonttitle=\bfseries}
\begin{tcolorbox}[title={\small$\ell_0^*$-minimization (LRR/MC/C-inf -- \yellow{VMT})}, width=3.15in]
\vspace{-.15in}
\begin{eqnarray}\label{eq:genmcl0pos}
\min_{X} & & \ell_0^*(X) \nonumber \\
\hspace{-.0in} \mbox{subject to} & & \y=A\vecw(X).\hspace{.4in}
\end{eqnarray}
\vspace{-.32in}
\end{tcolorbox}
\begin{tcolorbox}[title={\small$\ell_1^*$-minimization (LRR/MC/C-inf -- \yellow{VMT})}, width=3.15in]
\vspace{-.15in}
\begin{eqnarray}\label{eq:genmcl1pos}
\min_{X} & & \ell_1^*(X) \nonumber \\
 \mbox{subject to} & & \y=A\vecw(X). \hspace{.5in}
\end{eqnarray}
\vspace{-.32in}
\end{tcolorbox}
\end{center}

More on the history, usefulness, and applicability of the above introduced generic low rank recovery (LRR) via the tightest convex norm relaxation can be found in the introductory papers \cite{RFPrank,SAT05}. We also point out that we might on occasion refer to the above LRR description as the ``\emph{vectorized matrix terminology}" (\textbf{VMT}). Below, we wll also find it useful to introduce the very same LRR via the corresponding, so-called, ``\emph{masking matrix terminology}" (\textbf{MMT}).

\subsection{Matrix completion (MC) -- a special case of LRR}
\label{sec:mc}

The above is of course a generic LRR description. The matrix completion is a particular type of the LRR or, in technical terms, a special case of the above described LRR mathematical framework. In the matrix completion scenario one deals with a very particular type of system matrix $A$. Instead of (\ref{eq:mclinsys1b}) one then has
\begin{equation}\label{eq:mc1}
\y_i=A\vecw(X_{sol}) \quad \mbox{where} \quad \ell_0^*(X_{sol})=\ell_0(\sigma(X_{sol}))=\|\sigma(X_{sol})\|_0=k, k\in \mN
\quad \mbox{and} \quad \forall p\in\mR_+,\|A_{i,:}\|_p=1.
\end{equation}
This practically means that each row of system matrix $A$ has exactly one element equal to one and all other elements equal to zero. In a way, one can think of $A$ as being a cardinality $m$ subset of rows of $I_{n^2\times n^2}$. The reason for the appearance of such a matrix $A$ is of course the origin of the matrix completion itself. Basically, the matrix $A$ essentially emulates a ``mask" that one puts on matrix $X$ which allows reading out only $m$ of its $n^2$ elements. More on the origin of the matrix completion, its importance, and different related algorithmic considerations can be found in the introductory papers \cite{CR09matcomp,SAT05} as well as in many further studies that followed later on (see, e.g. \cite{CT10matcomp,KMO10matcomp,KMO10matcomp1,Klopp14matcomp,KLT11matcomp,NW11matcomp,NW12matcomp,RT11matcomp,MHT10}).

To ensure the completeness of the overall presentation and to be in alignment with the standard representation of the matrix completion problem usually seen throughout the literature we reformulate (\ref{eq:mc1}) through the above mentioned masking matrices terminology. Let $M\in\mR^{n\times n}$ be a masking matrix such that
\begin{equation}
M_{i,j}=\begin{cases}
          1, & \mbox{$(i,j)$-th element of $X_{sol}$ is observed}  \\
          0, & \mbox{otherwise}.
        \end{cases}  \label{eq:mc2}
\end{equation}
One can then describe creating the linear observations in the matrix completion as the following
\begin{equation}
Y=M \circ X_{sol},
  \label{eq:mc3}
\end{equation}
where $\circ$ stands for the component-wise multiplication. Clearly, ones in matrix $M$ allow reading out corresponding elements of $X_{sol}$ while zeros block (mask) them. To fit in the above linear description in (\ref{eq:mc1}) one would then create $A$ by removing all the zero rows from $\diag^{-1}(\vecw(M))I_{n^2\times n^2}$ ($\diag^{-1}(\cdot)$ creates the diagonal matrix with the elements on the main diagonal equal to its vector argument and, in a way, is the inverse of the earlier introduced $\diag(\cdot)$). Consequently, $\y$ would be obtained as $\y=AX_{sol}$.

\subsection{Causal inference (C-inf) -- a special case of MC}
\label{sec:cinf}

Finally, the causal inference, which will be the main topic of our interest, is yet another special case of the above low rank recovery framework. In fact, to be a bit more precise, the matrix completion is a special case of the above LRR and the causal inference is a special case of the matrix completion itself. The connection between the matrix completion (MC) and the causal inference (C-inf) was established in \cite{ABDIK21}. Moreover, a very nice additional connections to the unconfoundedness and the synthetic control were established in \cite{ABDIK21} as well.
Here, we will also work within the context of the matrix completion-causal inference connection. To that end we start by making this connection mathematically more precise. Namely, if one thinks of the matrix completion as a way of ``masking" $X$ and reading out the unmasked elements, then the causal inference does exactly the same thing while additionally imposing a particular structure on the mask itself. Let, as above, $M\in\mR^{n\times n}$ be a masking matrix. Then for a fixed $l\leq n$ one, in a causal inference context, has
\begin{equation}
M\triangleq M^{(l)} \quad \mbox{and} \quad M_{i,j}=M_{i,j}^{(l)}=\begin{cases}
          1, & \mbox{if } \min(i,j)\leq l \\
          0, & \mbox{otherwise}.
        \end{cases}  \label{eq:cinf1}
\end{equation}
The above is the so-called \bl{\emph{\textbf{block causal inference}}} setup. Figures \ref{fig:Mmc} and \ref{fig:Mnlockcinf} showcase the key difference between the generic matrix completion and the causal inference. In the generic matrix completion the mask matrix $M$ can have zeros and ones located within the matrix in a basically arbitrary way. On the other hand, in the causal inference setup the structure of $M$ is somewhat particular. In general, in a row of the matrix $M$ the first zero can appear not later than the first zeros appeared in any of the previous rows. After a zero all other remaining elements in the same row must also be zero. For the block case of our interest here it is as shown in Figure  \ref{fig:Mnlockcinf}.

\begin{figure}[htb]
\centering
\centerline{\epsfig{figure=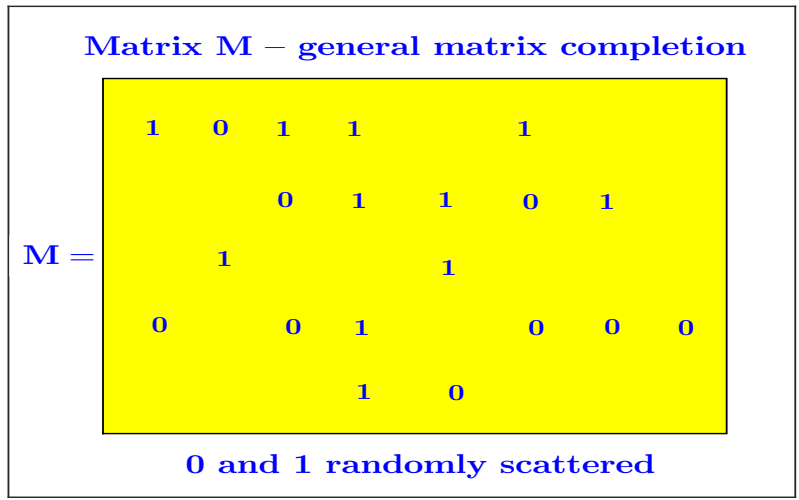,width=13.5cm,height=8cm}}
\caption{Matrix $M$ -- general matrix completion (\bl{\textbf{MC}}) setup}
\label{fig:Mmc}
\end{figure}

\begin{figure}[htb]
\centering
\centerline{\epsfig{figure=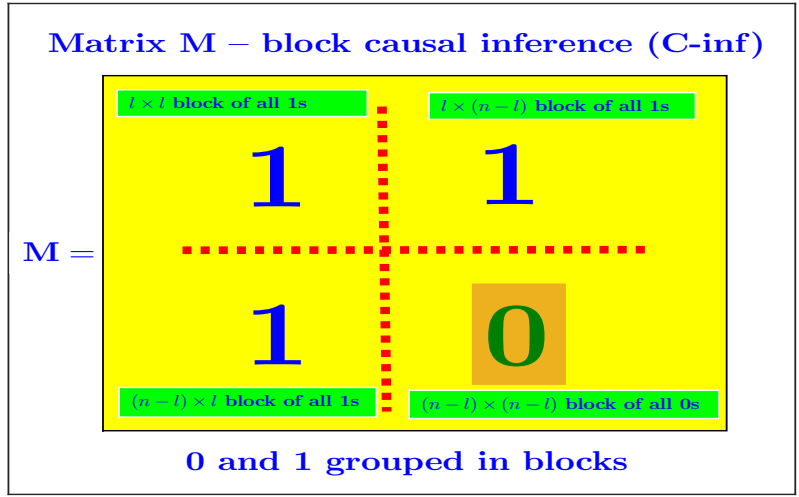,width=13.5cm,height=8cm}}
\caption{Matrix $M\triangleq M^{(l)}$ -- block causal inference (\bl{\textbf{C-inf}}) setup}
\label{fig:Mnlockcinf}
\end{figure}

\noindent The rationale for the use of the block causal inference is the most easily understood if one views things in the time domain. Namely, if the columns of the masking matrix $M$ represent time axis then the observations related to ceratin rows of the matrix will not be available after a fixed point in time. In the block scenario this point is fixed across the affected rows. However, it does not necessarily need to be fixed (for more in this direction we refer to \cite{ADHsynth10} (in particular, the California tobacco example), \cite{XPlatfac10} (in particular, the latent factor modeling in the context of the simultaneous/staggered treatment adoption), and to \cite{AthImb18,AthSte02,ShaTou19} (in particular, the health care applications) as excellent references for understanding the need of various C-inf scenarios). As this is the introductory paper, where we present the overall methodology, we selected the block causal inference scenario as probably the most representative and well-known one. In some of our companion papers we will show how the methodology that we are introducing here can be utilized to handle other C-inf scenarios as well.

Under this casual inference assumption one then has the following for the collection of the observation $Y$ in the matrix completion
\begin{equation}
Y=M \circ X = M^{(l)}\circ X.
  \label{eq:cinf2}
\end{equation}
We will more often than not avoid specifying superscript to make writing easier. From the context it will be clear if it should be there and,
 if so, what its value should be. To fit in the linear description of (\ref{eq:mc1}) one would create $A$ in the following way: 1) start by choosing the first $ln$ rows of $I_{n^2\times n^2}$; and 2) then for any next set of $n$ rows of $I_{n^2\times n^2}$ continue by choosing its subset of first $l$ rows while skipping the remaining $n-l$. Similarly, $\y$ would be obtained by stacking the columns of $Y$ and skipping the elements $Y_{i,j}$ where $\min(i,j)>l$. As mentioned above, we will find it useful later on to have (\ref{eq:genmcl0pos}) and (\ref{eq:genmcl1pos}) rewritten in the ``\emph{masking matrix terminology}" (\textbf{MMT}) as well;

\begin{center}
\tcbset{beamer,lower separated=false, fonttitle=\bfseries,width=3.4in, coltext=black ,
colback=yellow!70!orange!40!white,title style={left color=cyan!40!black!80!purple, right color=red!60!yellow!40!orange!80!white},
width=(\linewidth-4pt)/4, equal height group=AT,before=,after=\hfill,fonttitle=\bfseries}
\begin{tcolorbox}[title=$\ell_0^*$-minimization (C-inf -- \yellow{MMT}), width=3.1in]
\vspace{-.15in}
\begin{eqnarray}\label{eq:genmcl0posmmt}
\min_{X} & & \ell_0^*(X) \nonumber \\
\hspace{-.0in} \mbox{subject to} & & Y=M\circ X.\hspace{.4in}
\end{eqnarray}
\vspace{-.32in}
\end{tcolorbox}
\begin{tcolorbox}[title=$\ell_1^*$-minimization (C-inf -- \yellow{MMT}), width=3.2in]
\vspace{-.15in}
\begin{eqnarray}\label{eq:genmcl1posmmt}
\min_{X} & & \ell_1^*(X) \nonumber \\
 \mbox{subject to} & & Y=M\circ X. \hspace{.5in}
\end{eqnarray}
\vspace{-.32in}
\end{tcolorbox}
\end{center}

\subsection{\prp{\textbf{C-inf $\leftrightarrow$ MC}} connection via counterfactuals}
\label{sec:cinfmccf}

To understand where such a structure may come from, it is useful to connect it to the context of  treatment/units/times following the terminology in \cite{ABDIK21}. In such  context, one assumes that the matrix $X$ contains observations about a certain set of, say, $n$ units {(e.g. individuals, subpopulations, and geographic regions)} over a period of say, $n$, time instances. Assuming that the rows of $X$ correspond to the units and the columns to the time instances, one wants to estimate the effects that a certain treatment may have on the treated units. A subset of the units (say those that correspond to the rows $i>l$) is then at time $l$ exposed to an irreversible treatment (once the treatment starts its effects can not be reversed). {Examples of treatments include health therapies, socio-economic policies, and taxes.} To be able to appropriately assess the resulting treatment effects, in addition to having the values of $X$ after the treatment, one would need to have the access to the so-called \bl{\textbf{\emph{counterfactuals}}} -- the values of the treated units -- had the treatment not been applied. Switching back to the  matrix completion terminology, one would need to estimate (a presumably low rank) $X$ while not having access to its portion covered by the block-mask $M=M^{(l)}$. In other words, one would need to solve (\ref{eq:genmcl0posmmt}) with $M=M^{(l)}$.

Of course, as the change in the structure of $A$ or $M$ does not prevent the utilization of the above $\ell_0^*\longrightarrow\ell_1^*$ relaxation concept, one typically employs it as a provably polynomial heuristic strategy to solve the matrix completion/causal inference problems approximately. While it is somewhat intuitive that, as the closest convex norm relaxation, the above $\ell_1^*$-minimization might produce a matrix similar to the unknown $X_{sol}$ it is perhaps quite surprising that in certain scenarios it actually perfectly recovers the exact $X_{sol}$. Potential existence of such an $\ell_0^*-\ell_1^*$ equivalence is rather remarkable phenomenon and determining when or how often it happens is pretty much the key task in the mathematical analysis of the convex norm relaxation based LRR algorithms. Along the same lines, answering this very same question will be precisely the main mathematical contribution of this paper.

A lot of work will need to be performed, however, before we get to the point where we can say a few more concrete words about the ultimate mathematical contributions. We will utilize to a large degree some of the \bl{\textbf{\emph{Random Duality Theory (RDT)}}} concepts presented and discussed in details in a long line of work \cite{StojnicCSetam09,StojnicUpper10,StojnicICASSP10var,StojnicICASSP10knownsupp,StojnicICASSP09,StojnicISIT2010binary,StojnicRegRndDlt10,StojnicGenLasso10}). On top of that, quite a few additional mathematical concepts will be needed as well. While we will try to explain all the needed mathematical tools in sufficient detail, a solid level of familiarity with the RDT might be helpful.

Before moving to a more thorough discussion particularly related to the mathematical analysis of the causal inference we first briefly digress to address a seemingly paradoxical situation regarding the level of simplicity/difficulty of the above LRR on the one side and the MC/C-inf, as its special cases, on the other.

\subsection{Special cases are not necessarily simpler}
\label{sec:scnotsimpler}

The above introduction of the matrix completion and the causal inference concepts through the generic LRR mechanism might portrait them as subproblems of a more general class of recovery problems. Moreover, one then may be tempted to believe that as such they can be both \emph{solved} and \emph{analyzed} in the very same way as the generic LRR problems. That would basically mean that all the results that one could conceivably create for the LRR would automatically translate to hold in a similar form in the MC and the C-inf scenarios. The part related to the \emph{solving} of these problems is indeed true. The same algorithm (say $\ell_1^*$-minimization) that can be used as a heuristic for the generic LRR can be used for the MC and the C-inf as well. On the other hand, the part related to the \emph{analysis} could not be further from the truth. Not only would not the results obtained for the LRR directly translate to the MC and the C-inf but they actually often might need to be proven in a completely different way.

The key to fully understanding this paradox is in distinguishing the additional structuring of the unknown vector $X_{sol}$ from the additional structuring of the system matrix $A$. When the unknown vectors/matrices are additionally structured it is quite likely that the corresponding performance analyses of the underlying algorithms are translatable (see, e.g. the companion paper \cite{Cinfidealpc22}). On the other hand when the system matrices $A$ are additionally structured then not only that the corresponding analyses might be difficult to translate but they also might actually need to be completely replaced. Along the same lines, since the MC and the C-inf are the special cases of the LRR with regard to the additional structuring of $A$, it is not a priori clear that the ability to analytically handle the corresponding generic LRR in any way guarantees the existence of such an ability when it comes to analytically handling the MC and the C-inf. We will see some aspects of this reasoning already in one of the sections that will follow later on.

\section{Causal inference -- $\ell_0^*-\ell_1^*$ relaxation equivalence}
\label{sec:cinfreleqv}

As mentioned earlier, solving the generic LRR (and consequently the C-inf as its a special case) might be difficult due to a highly non-convex objective function in (\ref{eq:genmcl0posmmt}). Various heuristics can be employed depending on the practical scenarios that one can face. In the mathematically most challenging so-called linear regime, the above mentioned $\ell_1^*$-minimization relaxation heuristic is typically viewed as the best known provably polynomial one. We adopt the same view in what follows and take it as a current benchmark for the algorithmic handling of the C-inf. As mentioned above, a rather remarkable feature of this heuristic is that sometimes it can actually solve the underlying problems exactly. When that happens we say that the following $\ell_0^*-\ell_1^*$-equivalence phenomenon occurs.

\begin{center}
\tcbset{beamer,lower separated=false, fonttitle=\bfseries,
coltext=black , colback=yellow!70!orange!40!white ,title style={left color=black, right color=red!70!orange!30!white}}
\begin{tcolorbox}[title=$\ell_0^*-\ell_1^*$-equivalence (C-inf):, width=6.43in]
\vspace{-.0in}
Let $X_{sol}$ be the solution of (\ref{eq:genmcl0posmmt}) or (\ref{eq:genmcl0pos}) and let $\hat{X}$ be a solution of (\ref{eq:genmcl1posmmt}) or (\ref{eq:genmcl1pos}) and set
\begin{eqnarray*}\label{eq:genmcl1poseqvrmse}
\mathbf{RMSE}\triangleq\|\vecw(\hat{X})-\vecw(X_{sol})\|_2.
\end{eqnarray*}
\vspace{-.3in}
\begin{eqnarray}\label{eq:genmcl1poseqv}
\mbox{If and only if } \prp{(\hat{X}=X_{sol} \mbox{ and } \mathbf{RMSE}=0)} \quad \mbox{then} \quad \prp{(\ell_0^{*}-\mbox{minimization} \Longleftrightarrow \ell_1^{*}-\mbox{minimization})}.\quad
\end{eqnarray}
\vspace{-.3in}
\end{tcolorbox}
\end{center}

The above basically means that when the $\ell_0^*-\ell_1^*$-equivalence happens the optimization problems in (\ref{eq:genmcl0posmmt}) and (\ref{eq:genmcl1posmmt}) are equivalent and as such replaceable by each other. That would, of course, be an ideal scenario where it would be basically possible to replace the non-convex optimization problem with the convex one without losing anything in terms of the accuracy of the obtained solutions. Since the mere existence of such a scenario is already a remarkable phenomenon we will in this paper be interested in uncovering the underlying intricacies that enable for it ro happen. Moreover, as it will turn out that its occurrence is not an anomaly but rather a consequence of a generic property, we will then raise the bar accordingly and attempt to provide not only the proof of the existence but also a complete analytical characterization of this property. This will include a full characterization as to how often and in what scenarios it might happen. To do so we will combine the Random Duality Theory (RDT) tools from \cite{StojnicCSetam09,StojnicUpper10,StojnicICASSP10var,StojnicICASSP10knownsupp,StojnicICASSP09,StojnicISIT2010binary,StojnicRegRndDlt10,StojnicGenLasso10} and several advanced sophisticated probabilistic concepts that we will introduce along the way in the sections that follow below.

In the rest of this section we will focus on some algebraic $\ell_0^*-\ell_1^*$-equivalence preliminaries conceptually borrowed from the RDT. We start things off with a generic LRR $\ell_0^*-\ell_1^*$-equivalence result (the result is basically an adaptation of the general CS equivalence condition result from \cite{StojnicCSetam09,StojnicUpper10,StojnicICASSP09} to the corresponding one for the $\ell_1$ norm of the singular/eigenvalues (similar adaptation can also be found in \cite{OH10})).

\begin{theorem}(\textbf{\bl{$\ell_0^*-\ell_1^*$-equivalence condition (LRR)}} -- \textbf{general}  $X$)
Consider a $\bU\in\mR^{n\times k}$ such that $\bU^T\bU=I_{k\times k}$ and a $\bV\in\mR^{n\times k}$ such that $\bV^T\bV=I_{k\times k}$ and a  $\rankw-k$  matrix $X_{sol}=X\in\mR^{n\times n}$  with all of its columns belonging to the span of $\bU$ and all of its rows belonging to the span of $\bV^T$. Also, let the orthogonal spans $\bU^{\perp}\in\mR^{n\times (n-k)}$ and $\bV^{\perp}\in\mR^{n\times (n-k)}$ be such that $U\triangleq \begin{bmatrix}
    \bU & \bU^{\perp}
   \end{bmatrix}$ and $V\triangleq \begin{bmatrix}
    \bV & \bV^{\perp}
   \end{bmatrix}$ and
\begin{equation}\label{eq:cinfthm0}
U^TU\triangleq \begin{bmatrix}
    \bU & \bU^{\perp}
   \end{bmatrix}^T\begin{bmatrix}
    \bU & \bU^{\perp}
   \end{bmatrix}=I_{n\times n} \quad \mbox{and} \quad
 V^TV \triangleq\begin{bmatrix}
    \bV & \bV^{\perp}
   \end{bmatrix}^T\begin{bmatrix}
    \bV & \bV^{\perp}
   \end{bmatrix}=I_{n\times n}.
 \end{equation}
For a given matrix $A\in\mR^{m\times n^2}$ ($m\leq n^2$) assume that $\y=A\vecw(X)=A\vecw(X_{sol})\in \mR^m$ and let $\hat{X}$ be the solution of (\ref{eq:genmcl1posmmt}). If
\begin{equation}
(\forall W\in \mR^{n\times n} | A\vecw(W)=\0_{m\times 1},W\neq \0_{n\times n}) \quad  -\tr(\bU^TW\bV)< \ell_1^*((\bU^{\perp})^TW\bV^{\perp}),
\label{eq:cinfthm1}
\end{equation}
then
\begin{equation}
\ell_0^*\Longleftrightarrow \ell_1^* \quad \mbox{and}\quad  \textbf{\emph{RMSE}}=\|\mbox{\emph{vec}}(\hat{X})-\mbox{\emph{vec}}(X_{sol})\|_2=0,\label{eq:cinfthm1a}
\end{equation}
and the solutions of (\ref{eq:genmcl0posmmt}) (or (\ref{eq:genmcl0pos})) and (\ref{eq:genmcl1posmmt}) (or (\ref{eq:genmcl1pos})) coincide. Moreover, if
\begin{equation}
(\exists  W\in \mR^{n\times n} | A\vecw(W)=\0_{m\times 1},W\neq \0_{n\times n}) \quad  -\tr(\bU^TW\bV)\geq \ell_1^*((\bU^{\perp})^TW\bV^{\perp}),
\label{eq:cinfthm2}
\end{equation}
then there is an $X$ from the above set of matrices with columns belonging to the span of $\bU$ and rows belonging to the span of $\bV$ such that the solutions of (\ref{eq:genmcl0posmmt}) (or (\ref{eq:genmcl0pos})) and (\ref{eq:genmcl1posmmt}) (or (\ref{eq:genmcl1pos})) are different.
\label{thm:cinfthm1}
\end{theorem}
\begin{proof}
  The proof is a trivial adaptation of the proof for symmetric matrices given in Appendix \ref{sec:appA}.
\end{proof}

The condition in the theorem relates matrix $W$ to the null-space of matrix $A$ and as such is VMT based. In the MC and C-inf cases that are of our interest here, it is more convenient to deal with its an MMT analogue. Recalling on the proof of Theorem \ref{thm:cinfthm1} from Appendix \ref{sec:appA} and the origin and role of matrix $W$ within that proof, one has that stating that $W$ belongs to the null-space of $A$ is basically equivalent to stating that $M\circ W=\0_{n\times n}$. In other words, one has the equivalence between the following two sets
\begin{equation}
(W\in \mR^{n\times n} | A\vecw(W)=\0_{m\times 1},W\neq \0_{n\times n})
 \Longleftrightarrow (W\in \mR^{n\times n} | M\circ W=\0_{n\times n},W\neq \0_{n\times n}).
 \label{eq:cinfanl1}
\end{equation}
Continuing further in the spirit of the RDT the following corollary of the above theorem can be established as well.
\begin{corollary}(\textbf{\bl{$\ell_0^*-\ell_1^*$-equivalence condition via masking matrix (MC/C-inf)}} -- \textbf{general}  $X$)
Assume the setup of Theorem \ref{thm:cinfthm1} with $X_{sol}$ being the unique solution of (\ref{eq:genmcl0posmmt}) (or (\ref{eq:genmcl0pos})). Let the masking matrix $M\in\mR^{n\times n}$ have $m$ ones and $(n^2-m)$ zeros and let $A$ be generated via $M$, i.e. let $A$ be the matrix obtained after removing all the zero rows from $\diag^{-1}(\vecw(M))I_{n^2\times n^2}$. If and only if
\begin{equation}
\min_{W,W^TW=1,M\circ W=\0_{n\times n}}  \tr(\bU^TW\bV)+\ell_1^*((\bU^{\perp})^TW\bV^{\perp})\geq 0,
\label{eq:cinfcor1}
\end{equation}
then
\begin{equation}
\ell_0^*\Longleftrightarrow \ell_1^* \quad \mbox{and}\quad  \textbf{\emph{RMSE}}=\|\mbox{\emph{vec}}(\hat{X})-\mbox{\emph{vec}}(X_{sol})\|_2=0,\label{eq:cinfcor1a}
\end{equation}
and the solutions of (\ref{eq:genmcl0posmmt}) (or (\ref{eq:genmcl0pos})) and (\ref{eq:genmcl1posmmt}) (or (\ref{eq:genmcl1pos})) coincide.
 \label{cor:cinfcor1}
\end{corollary}
\begin{proof}
  Follows immediately as a combination of (\ref{eq:cinfthm1}), (\ref{eq:cinfthm2}), and (\ref{eq:cinfanl1}).
\end{proof}

\noindent \textbf{Remark:} Carefully comparing the conditions in (\ref{eq:cinfthm1}) and (\ref{eq:cinfcor1}) one can observe that a strict inequality is loosened up a bit at the expense of the uniqueness assumption. With a little bit of extra effort one may avoid this. However, to make writings below substantially easier we will work with a non-strict inequality.

To analyze the optimization problem in (\ref{eq:cinfcor1}) we follow into the footsteps of \cite{StojnicCSetam09} where similar optimization problems were handled on multiple occasions through a very generic Lagrangian mechanism. The first step of such a procedure is writing down explicitly the optimization from (\ref{eq:cinfcor1})
\begin{eqnarray}
f_{pr}(M;U,V) \triangleq \min_{W} & &  \tr(\bU^TW\bV)+\ell_1^*((\bU^{\perp})^TW\bV^{\perp}) \nonumber \\
\mbox{subject to} & & M\circ W=\0_{n\times n} \nonumber \\
& & W^TW=1.
\label{eq:cinfanl2}
\end{eqnarray}
One can then write the corresponding Lagrangian and the Lagrange dual function to obtain
\begin{eqnarray}
{\cal L}(W,\Lambda,\Theta,\gamma) & \triangleq &  \tr(\bU^TW\bV)+\ell_1^*((\bU^{\perp})^TW\bV^{\perp})+\Theta (M\circ W)+\gamma \lp\tr(W^TW)-1\rp \nonumber \\
& = & \max_{\Lambda,\Lambda^T\Lambda\leq I} \lp \tr(\bU^TW\bV)+\tr(\Lambda((\bU^{\perp})^TW\bV^{\perp}))+\Theta (M\circ W)+\gamma \lp\tr(W^TW)-1\rp \rp\nonumber \\
& = & \max_{\Lambda,\Lambda^T\Lambda\leq I} \lp \tr\lp (\bV\bU^T+\bV^{\perp}\Lambda(\bU^{\perp})^T+\Theta\circ M)W\rp + \gamma\tr(W^TW)-\gamma \rp,\label{eq:cinfanl3}
\end{eqnarray}
and
\begin{eqnarray}
g(\Theta,\gamma) & \triangleq & \min_{W} {\cal L}(W,\Lambda,\Theta,\gamma) \nonumber \\
 & = & \min_{W} \max_{\Lambda,\Lambda^T\Lambda\leq I} \lp \tr\lp (\bV\bU^T+\bV^{\perp}\Lambda(\bU^{\perp})^T+\Theta\circ M)W\rp + \gamma\tr(W^TW)-\gamma \rp.\label{eq:cinfanl4}
\end{eqnarray}
Utilizing the Lagrangian duality one then further has
\begin{eqnarray}
g(\Theta,\gamma) & = & \min_{W} \max_{\Lambda,\Lambda^T\Lambda\leq I} \lp \tr\lp (\bV\bU^T+\bV^{\perp}\Lambda(\bU^{\perp})^T+\Theta\circ M)W\rp + \gamma\tr(W^TW)-\gamma \rp \nonumber \\
& \geq & \max_{\Lambda,\Lambda^T\Lambda\leq I} \min_{W}  \lp \tr\lp (\bV\bU^T+\bV^{\perp}\Lambda(\bU^{\perp})^T+\Theta\circ M)W\rp + \gamma\tr(W^TW)-\gamma \rp \nonumber \\
& = & \max_{\Lambda,\Lambda^T\Lambda\leq I} \lp -\frac{1}{4\gamma} \tr\lp (\bV\bU^T+\bV^{\perp}\Lambda(\bU^{\perp})^T+\Theta\circ M)(\bV\bU^T+\bV^{\perp}\Lambda(\bU^{\perp})^T+\Theta\circ M)^T\rp  -\gamma \rp.\nonumber \\\label{eq:cinfanl5}
\end{eqnarray}
Moreover,
{\small\begin{eqnarray}
f_{pr}(M;U,V) & \geq  & \max_{\Theta,\gamma} g(\Theta,\gamma) \nonumber \\
 & = & \max_{\Lambda,\Lambda^T\Lambda\leq I,\Theta,\gamma} \lp -\frac{1}{4\gamma} \tr\lp (\bV\bU^T+\bV^{\perp}\Lambda(\bU^{\perp})^T+\Theta\circ M)(\bV\bU^T+\bV^{\perp}\Lambda(\bU^{\perp})^T+\Theta\circ M)^T\rp  -\gamma \rp.\nonumber \\\label{eq:cinfanl5a}
\end{eqnarray}}We proceed by further optimizing over $\gamma$.
{\small\begin{eqnarray}
f_{pr}(M;U,V)
& \geq & \max_{\Lambda,\Lambda^T\Lambda\leq I,\Theta,\gamma} \lp -\frac{1}{4\gamma} \tr\lp (\bV\bU^T+\bV^{\perp}\Lambda(\bU^{\perp})^T+\Theta\circ M)(\bV\bU^T+\bV^{\perp}\Lambda(\bU^{\perp})^T+\Theta\circ M)^T\rp  -\gamma \rp \nonumber \\
& = & \max_{\Lambda,\Lambda^T\Lambda\leq I,\Theta}  -\sqrt{\tr\lp (\bV\bU^T+\bV^{\perp}\Lambda(\bU^{\perp})^T+\Theta\circ M)(\bV\bU^T+\bV^{\perp}\Lambda(\bU^{\perp})^T+\Theta\circ M)^T\rp}\nonumber \\
& = & -\min_{\Lambda,\Lambda^T\Lambda\leq I,\Theta}  \sqrt{\tr\lp (\bV\bU^T+\bV^{\perp}\Lambda(\bU^{\perp})^T+\Theta\circ M)(\bV\bU^T+\bV^{\perp}\Lambda(\bU^{\perp})^T+\Theta\circ M)^T\rp}. \label{eq:cinfanl6}
\end{eqnarray}}

We now particularize the above to the block causal inference scenario. In the block C-inf scenario the matrix $M$ is as defined in (\ref{eq:cinf1}). For the concreteness and to facilitate writings later on we will introduce matrix $I^{(l)}$ and also keep in mind the following characterization of matrix $M$

\begin{center}
\tcbset{beamer,lower separated=false, fonttitle=\bfseries,
coltext=black , colback=yellow!70!orange!40!white ,title style={left color=black, right color=red!70!orange!30!white}}
\begin{tcolorbox}[title=$M$ matrix in causal inference (C-inf):, width=6.2in]
\vspace{-.0in}
\begin{equation}\label{eq:cinfanl2a}
  M\triangleq M^{(l)}\triangleq \1_{n\times 1}\1_{n\times 1}^T- I^{(l)}(I^{(l)})^T   \1_{n\times 1}\1_{n\times 1}^T   I^{(l)}(I^{(l)})^T \quad \mbox{and} \quad I^{(l)} \triangleq \begin{bmatrix}
        \0_{l\times (n-l)} \\
        I_{(n-l)\times (n-l)}
      \end{bmatrix}.
\end{equation}
\vspace{-.15in}
\end{tcolorbox}
\end{center}

\noindent In the above definition/representation of $M$, $\1$/$\0$ stand for vectors and matrices of all ones/zeros with the dimensions specified in their subscripts. To avoid overwhelming the notation, we may on occasion skip specifying the underlying dimensions of these vectors. However, they should be easy to infer from the overall context. Also, $l$ is, of course, adjusted so that $M$ has $m$ nonzero elements, i.e. $m$ elements equal to one which basically amounts to having the following identity to hold
\begin{eqnarray}
   m & = & n^2-(n-l)^2. \label{eq:cinfanl2b}
\end{eqnarray}

Using the above C-inf $M$ in (\ref{eq:cinfanl6}) the only terms that will be left in $\lp\bV\bU^T+\bV^{\perp}\Lambda(\bU^{\perp})^T\rp$ after the optimization are those that correspond to the zero elements of $M$, i..e. the only ones where the presence of (and consequently the optimization over) $\Theta$ can not have an effect. This basically implies
\begin{equation}
f_{pr}(M;U,V)
    \geq-\min_{\Lambda,\Lambda^T\Lambda\leq I}  \sqrt{\tr\lp \lp (I^{(l)})^T \lp \bV\bU^T+\bV^{\perp}\Lambda(\bU^{\perp})^T\rp I^{(l)}\rp \lp (I^{(l)})^T \lp \bV\bU^T+\bV^{\perp}\Lambda(\bU^{\perp})^T\rp I^{(l)}\rp^T\rp}.\label{eq:cinfanl7}
\end{equation}
For the overall success of the whole machinery one would need that $\Lambda$ in the above optimization can be chosen so that the overall optimum is nonnegative. This is then sufficient to establish the following alternative to Corollary \ref{cor:cinfcor1}.

\begin{corollary}(\textbf{\bl{$\ell_0^*-\ell_1^*$-equivalence condition via masking matrix (C-inf)}} -- \textbf{general}  $X$)
Assume the setup of Theorem \ref{thm:cinfthm1} and Corollary \ref{cor:cinfcor1}. Let $I^{(l)}$ be as in (\ref{eq:cinfanl2a}). Then
\begin{center}
 \begin{tcolorbox}[beamer,title=\textbf{C-inf \yellow{perfectly succeeds: $\ell_0^*\Longleftrightarrow \ell_1^* \quad \mbox{and}\quad  \textbf{\emph{RMSE}}=\|\mbox{\emph{vec}}(\hat{X})-\mbox{\emph{vec}}(X_{sol})\|_2=0$}},lower separated=false, colback=yellow!95!green!40!white,
colframe=red!75!blue!60!black,fonttitle=\bfseries,width=6in]
\begin{equation}
\mbox{ If and only if} \quad \exists \Lambda| \Lambda^T\Lambda\leq I \quad \mbox{and} \quad (I^{(l)})^T \lp \bV\bU^T+\bV^{\perp}\Lambda(\bU^{\perp})^T\rp I^{(l)}=0\label{eq:cinfcor2eq1}
\end{equation}
 \end{tcolorbox}
\end{center}
 \label{cor:cinfcor2}
\end{corollary}
\begin{proof}
The ``if" part follows from Corollary \ref{cor:cinfcor1}, (\ref{eq:cinfanl7}), and the above discussion. The ``only if" part follows after noting that all the above inequalities in (\ref{eq:cinfanl5})-(\ref{eq:cinfanl7}) are written for generic instructional purposes. Due to the underlying convexity and the strong duality they all actually can be replaced with equalities as well.
\end{proof}

For the time being we will assume $k\leq l$ (later on this assumption will be rigorously justified). From (\ref{eq:cinfcor2eq1}) one then easily has
\begin{equation}
\Lambda=((I^{(l)})^T \bV^{\perp})^{-1} (I^{(l)})^T \bV \bU^T I^{(l)}((\bU^{\perp})^T I^{(l)})^{-1}  \quad \Longrightarrow \quad (I^{(l)})^T \lp \bV\bU^T+\bV^{\perp}\Lambda(\bU^{\perp})^T\rp I^{(l)}=0, \label{eq:cinfanl8}
\end{equation}
where $(\cdot)^{-1}$ stands for the pseudo-inverse. To make writing a bit easier we can set
\begin{eqnarray}
\Lambda_{opt} & \triangleq & \Lambda_V\Lambda_U^T \nonumber \\
\Lambda_V & \triangleq & ((I^{(l)})^T \bV^{\perp})^{-1} (I^{(l)})^T \bV \nonumber \\
\Lambda_U & \triangleq & ((I^{(l)})^T \bU^{\perp})^{-1} (I^{(l)})^T \bU. \label{eq:cinfanl9}
\end{eqnarray}
Let $\lambda_{max}(\cdot)$ be the maximum eigenvalue of its symmetric matrix argument. After combining (\ref{eq:cinfcor2eq1})-(\ref{eq:cinfanl9}) we conclude that if
\begin{eqnarray}
\lambda_{max}(\Lambda_{opt}^T\Lambda_{opt}) \leq 1, \label{eq:cinfanl10}
\end{eqnarray}
then (\ref{eq:cinfcor2eq1}) will be satisfied. After basic algebraic transformations (\ref{eq:cinfanl10}) can also be rewritten as
\begin{eqnarray}
\lambda_{max}(\Lambda_{opt}^T\Lambda_{opt})=\lambda_{max}(\Lambda_{opt}\Lambda_{opt}^T)=\lambda_{max}(\Lambda_V\Lambda_U^T\Lambda_U\Lambda_V^T)
=\lambda_{max}(\Lambda_V^T\Lambda_V\Lambda_U^T\Lambda_U)\leq 1. \label{eq:cinfanl11}
\end{eqnarray}
From (\ref{eq:cinfanl11}) it is rather clear that the spectrum of $\Lambda_V^T\Lambda_V\Lambda_U^T\Lambda_U$ as well as the spectra of $\Lambda_V^T\Lambda_V$ and $\Lambda_U^T\Lambda_U$ play an important role in the $\ell_0^*-\ell_1^*$-equivalence. We first observe a worst case bound. Namely, since
\begin{eqnarray}
\lambda_{max}(\Lambda_{opt}^T\Lambda_{opt})=\lambda_{max}(\Lambda_V^T\Lambda_V\Lambda_U^T\Lambda_U)
\leq \lambda_{max}(\Lambda_V^T\Lambda_V)\lambda_{max}(\Lambda_U^T\Lambda_U)), \label{eq:cinfanl12}
\end{eqnarray}
one has that if the individual spectra of $\Lambda_V^T\Lambda_V$ and $\Lambda_U^T\Lambda_U$ do not exceed one then the $\ell_0^*-\ell_1^*$-equivalence holds. Given the obvious importance of these spectra we will below look at them in more detail. Clearly, due to symmetry we need to focus on only one of them. To that end we start by observing
\begin{eqnarray}
((I^{(l)})^T \bV^{\perp})^{-1} & = & (\bV^{\perp})^T I^{(l)} \lp (I^{(l)})^T\bV^{\perp}(\bV^{\perp})^T I^{(l)}\rp^{-1}, \label{eq:cinfanl13}
\end{eqnarray}
From (\ref{eq:cinfanl13}) one quickly finds
\begin{eqnarray}
\lp ((I^{(l)})^T \bV^{\perp})^{-1} \rp^T ((I^{(l)})^T \bV^{\perp})^{-1} & = & \lp (I^{(l)})^T\bV^{\perp}(\bV^{\perp})^T I^{(l)}\rp^{-1}. \label{eq:cinfanl14}
\end{eqnarray}
We can then write
\begin{eqnarray}
 (I^{(l)})^T \bV \bV^T I^{(l)} & = &  (I^{(l)})^T \lp I - \bV^{\perp} (\bV^{\perp})^T \rp I^{(l)}. \label{eq:cinfanl15}
\end{eqnarray}
From (\ref{eq:cinfanl9}) we have
\begin{eqnarray}
Q_1\triangleq \Lambda_V^T\Lambda_V & = &  \lp \lp ((I^{(l)})^T \bV^{\perp})^{-1} (I^{(l)})^T \bV\rp^T \lp ((I^{(l)})^T \bV^{\perp})^{-1} (I^{(l)})^T \bV \rp \rp \nonumber \\
& = &    \bV^T I^{(l)}\lp ((I^{(l)})^T \bV^{\perp})^{-1} \rp^T ((I^{(l)})^T \bV^{\perp})^{-1}   \lp(I^{(l)})^T \bV \rp.\label{eq:cinfanl16}
\end{eqnarray}
Now, we will find it more convenient to work with a slightly change version of matrix $Q_1$. Namely, after combining (\ref{eq:cinfanl9}), (\ref{eq:cinfanl14}), and (\ref{eq:cinfanl15}) we obtain
\begin{eqnarray}
Q & \triangleq &  \lp \lp ((I^{(l)})^T \bV^{\perp})^{-1} \rp^T ((I^{(l)})^T \bV^{\perp})^{-1}   \lp(I^{(l)})^T \bV \bV^T I^{(l)}\rp \rp \nonumber \\
& = &   \lp \lp (I^{(l)})^T\bV^{\perp}(\bV^{\perp})^T I^{(l)}\rp^{-1}  \lp (I^{(l)})^T \lp I - \bV^{\perp} (\bV^{\perp})^T \rp I^{(l)}\rp \rp\nonumber \\
& = &   \lp (I^{(l)})^T\bV^{\perp}(\bV^{\perp})^T I^{(l)}\rp^{-1} -I. \label{eq:cinfanl16a0}
\end{eqnarray}
Clearly, all the nonzero eigenvalues of $Q_1$ and $Q$ are identical. When $k\leq n-l$ then $Q$ has all the eigenvalues of $Q_1$ plus $n-l-k$ extra zeros. On the other hand, when $k\geq n-l$ then $Q_1$ has all the eigenvalues of $Q$ plus $k-(n-l)$ extra zeros. Since adding or removing zeros from the spectra will not change any of their features of our interests here, instead of working directly with $Q_1$, we can work with $Q$. In particular, we have
\begin{eqnarray}
\lambda_{max}(Q_1) = \lambda_{max}(\Lambda_V^T\Lambda_V)
& = & \lambda_{max} \lp \lp (I^{(l)})^T\bV^{\perp}(\bV^{\perp})^T I^{(l)}\rp^{-1}\rp -1 = \lambda_{max}(Q). \label{eq:cinfanl16a}
\end{eqnarray}

We are now in position to establish a spectral alternative to Corollary \ref{cor:cinfcor2}.

\begin{corollary}(\textbf{\bl{$\ell_0^*-\ell_1^*$-equivalence condition via mask-modified bases spectra (C-inf)}} -- \textbf{general}  $X$)
Assume the setup of Theorem \ref{thm:cinfthm1} and Corollaries \ref{cor:cinfcor1} and \ref{cor:cinfcor2} with $k\leq l$. Let $\lambda_V$ and $\lambda_U$ be defined as in (\ref{eq:cinfanl9}). Then
\begin{center}
 \begin{tcolorbox}[beamer,title=\textbf{C-inf \yellow{perfectly succeeds: $\ell_0^*\Longleftrightarrow \ell_1^* \quad \mbox{and}\quad  \textbf{\emph{RMSE}}=\|\mbox{\emph{vec}}(\hat{X})-\mbox{\emph{vec}}(X_{sol})\|_2=0$}},lower separated=false, colback=yellow!95!green!40!white,
colframe=red!75!blue!60!black,fonttitle=\bfseries,width=6in]
\begin{equation}
\mbox{If and only if} \quad \lambda_{max}(\Lambda_V^T\Lambda_V\Lambda_U^T\Lambda_U)\leq 1.\label{eq:cinfcor3eq1}
\end{equation}
 \end{tcolorbox}
\end{center}
Moreover, if
\begin{eqnarray}
\lp\lambda_{max} \lp \lp (I^{(l)})^T\bV^{\perp}(\bV^{\perp})^T I^{(l)}\rp^{-1}\rp -1\rp
\lp\lambda_{max} \lp \lp (I^{(l)})^T\bU^{\perp}(\bU^{\perp})^T I^{(l)}\rp^{-1}\rp -1\rp \leq 1,\label{eq:cinfcor3eq2}
\end{eqnarray}
then again $\ell_0^*\Longleftrightarrow \ell_1^*$ and $\textbf{\emph{RMSE}}=\|\mbox{\emph{vec}}(\hat{X}-\mbox{\emph{vec}}(X_{sol})\|_2=0$ and the C-inf perfectly succeeds as well.
\label{cor:cinfcor3}
\end{corollary}
\begin{proof}
The first part follows from Corollaries \ref{cor:cinfcor1} and \ref{cor:cinfcor2}, (\ref{eq:cinfanl8}), (\ref{eq:cinfanl9}), (\ref{eq:cinfanl11}), the above discussion and some additional considerations while the second part follows by additional taking into account (\ref{eq:cinfanl12}) and (\ref{eq:cinfanl16a}). We below present all the details split into three parts: the first two relate to the equivalence condition (equation (\ref{eq:cinfcor3eq1})) and third one to (\ref{eq:cinfcor3eq2}).

\noindent \underline{\bl{\textbf{\emph{1) $\Longrightarrow$ -- The ``if part" of condition (\ref{eq:cinfcor3eq1}):}}}} Choosing $\Lambda=\Lambda_{opt}$
\begin{eqnarray}
\hspace{-.0in}\Lambda_{opt} & \triangleq & \Lambda_V\Lambda_U^T  =  -((I^{(l)})^T \bV^{\perp})^{-1} (I^{(l)})^T \bV \bU^T I^{(l)}((\bU^{\perp})^T I^{(l)})^{-1}, \label{eq:cinfanl8z1}
\end{eqnarray}
(where $(\cdot)^{-1}$ stands for the pseudo-inverse) one ensures
\begin{equation} \label{eq:cinfanl9z1}
  (I^{(l)})^T \lp \bV\bU^T+\bV^{\perp}\Lambda(\bU^{\perp})^T\rp I^{(l)}=0.
\end{equation}
Let $\lambda_{max}(\cdot)$ be the maximum eigenvalue of its symmetric matrix argument. A combination of (\ref{eq:cinfcor2eq1})-(\ref{eq:cinfanl9z1}) ensures that if
\begin{eqnarray}
\lambda_{max}(\Lambda_{opt}^T\Lambda_{opt}) \leq 1, \label{eq:cinfanl10z1}
\end{eqnarray}
then $\Lambda_{opt}$ satisfies (\ref{eq:cinfcor2eq1}). (\ref{eq:cinfanl10z1}) is implied by (\ref{eq:cinfcor3eq1}) since
\begin{equation}
 \lambda_{max}(\Lambda_{opt}^T\Lambda_{opt})  =  \lambda_{max}(\Lambda_U\Lambda_V^T\Lambda_V\Lambda_U^T)  = \lambda_{max}(\Lambda_V^T\Lambda_V\Lambda_U^T\Lambda_U)  \leq 1, \nonumber \label{eq:cinfanl11z1}
\end{equation}
which suffices to complete the proof of the ``if part".

\noindent \underline{\bl{\textbf{\emph{2) $\Longleftarrow$ -- The ``only if part" of condition (\ref{eq:cinfcor3eq1}):}}}}  Consider SVDs
\begin{equation}\label{eq:cinfanl11aa1z1}
B \triangleq (I^{(l)})^TV^\perp=U_B\Sigma_BV_B^T,\quad C \triangleq (I^{(l)})^TU^\perp=U_C\Sigma_CV_C^T
\end{equation}
with unitary $U_B,V_B,U_C,V_C$ and diagonal (with no zeros on the main diagonal) $\Sigma_B,\Sigma_C$. Any $\Lambda$ can be parameterized as
\begin{equation}\label{eq:cinfanl11aa2z1}
\Lambda = V_BH^T+V_B^\perp D^T,\quad H\triangleq V_C E+V_C^\perp F
\end{equation}
for some $E,F,D$ and unitary $V_B^\perp$ and $V_C^\perp$ such that $V_B^TV_B^\perp=V_C^TV_C^\perp=0$. Also, one can set $\Lambda_*$ and write the SVD of $E$
\begin{equation}\label{eq:cinfanl11aa3z1}
\Lambda_* \triangleq V_B E^TV_C^T, \quad E=U_E\Sigma_EV_E^T,
\end{equation}
where $U_E,V_E$ are unitary and $\Sigma_E$ is diagonal with entries on the main diagonal being nonzero and in ascending order. Let $\u_e$ be the last column of $U_E$ (i.e. the eigenvector of $EE^T$ that corresponds to its largest eigenvalue). Since $\|V_C\u_e\|_2=1$,
\begin{eqnarray}\label{eq:cinfanl11aa4z1}
\lambda_{max}(\Lambda^T\Lambda)& \geq  &   \u_e^TV_C^T \Lambda^T\Lambda V_C\u_e \nonumber \\
 & = & \u_e^TV_C^T HH^T V_C\u_e  +\u_e^TV_C^T DD^T V_C\u_e  \nonumber \\
 & \geq  & \u_e^TV_C^T (V_C E+V_C^\perp F)(V_C E+V_C^\perp F)^T V_C\u_e    \nonumber \\
 & =  & \u_e^T EE^T\u_e =\lambda_{max}(EE^T) \nonumber \\
 & = &  \lambda_{max}(V_CEE^TV_C^T) =\lambda_{max}(\Lambda_*^T\Lambda_*).
\end{eqnarray}
If $\Lambda$ satisfies the condition of (\ref{eq:cinfcor2eq1}) then a combination of (\ref{eq:cinfcor2eq1}) and (\ref{eq:cinfanl11aa1z1})-(\ref{eq:cinfanl11aa3z1}) gives
\begin{equation}\label{eq:cinfanl11aa5z1}
(I^{(l)})^T\bar{V}\bar{U}I^{(l)}+B\Lambda_*C^T=0,
\end{equation}
and a combination of (\ref{eq:cinfanl9}), (\ref{eq:cinfanl11aa1z1}), and (\ref{eq:cinfanl11aa5z1}) gives
\begin{equation}\label{eq:cinfanl11aa6z1}
\Lambda_V\Lambda_U^T  =   -B^{-1}(I^{(l)})^T\bar{V}\bar{U}I^{(l)}(C^T)^{-1} = \Lambda_*.
\end{equation}
Finally, for  $\Lambda$ that fits (\ref{eq:cinfcor2eq1}), from (\ref{eq:cinfanl11aa4z1}) and (\ref{eq:cinfanl11aa6z1}) one has
\begin{eqnarray}\label{eq:cinfanl11aa7z1}
1 & \geq & \lambda_{max}(\Lambda^T\Lambda)  >  \lambda_{max}(\Lambda_*^T\Lambda_*)=\lambda_{max}(\Lambda_*\Lambda_*^T) \nonumber \\
 & = & \lambda_{max}(\Lambda_V\Lambda_U^T\Lambda_U\Lambda_V^T) = \lambda_{max}(\Lambda_V^T\Lambda_V\Lambda_U^T\Lambda_U),
\end{eqnarray}
which completes the proof of the ``only if part".

\noindent \underline{\bl{\textbf{\emph{3) Suffciency of the condition (\ref{eq:cinfcor3eq2}):}}}}  Since
\begin{equation}
\lambda_{max}(\Lambda_V^T\Lambda_V\Lambda_U^T\Lambda_U)
\leq \lambda_{max}(\Lambda_V^T\Lambda_V)\lambda_{max}(\Lambda_U^T\Lambda_U), \label{eq:cinfanl12z1}
\end{equation}
one has that if the individual spectra of $\Lambda_V^T\Lambda_V$ and $\Lambda_U^T\Lambda_U$ do not exceed one then the $\ell_0^*-\ell_1^*$-equivalence holds. Due to symmetry we focus only on one of them. First we observe
\begin{equation*}
((I^{(l)})^T \bV^{\perp})^{-1}  = (\bV^{\perp})^T I^{(l)} \lp (I^{(l)})^T\bV^{\perp}(\bV^{\perp})^T I^{(l)}\rp^{-1}, \label{eq:cinfanl13z1}
\end{equation*}
and quickly find
{\small\begin{equation}
\lp ((I^{(l)})^T \bV^{\perp})^{-1} \rp^T ((I^{(l)})^T \bV^{\perp})^{-1}  =  \lp (I^{(l)})^T\bV^{\perp}(\bV^{\perp})^T I^{(l)}\rp^{-1}. \label{eq:cinfanl14z1}
\end{equation}}We also note
\begin{eqnarray}
 (I^{(l)})^T \bV \bV^T I^{(l)} & = &  (I^{(l)})^T \lp I - \bV^{\perp} (\bV^{\perp})^T \rp I^{(l)}, \label{eq:cinfanl15z1}
\end{eqnarray}
and
\begin{eqnarray}
Q_1 & \triangleq & \Lambda_V^T\Lambda_V \nonumber \\
& = &  \lp \lp ((I^{(l)})^T \bV^{\perp})^{-1} (I^{(l)})^T \bV\rp^T \lp ((I^{(l)})^T \bV^{\perp})^{-1} (I^{(l)})^T \bV \rp \rp \nonumber \\
& = &  \bV^T I^{(l)}\lp ((I^{(l)})^T \bV^{\perp})^{-1} \rp^T ((I^{(l)})^T \bV^{\perp})^{-1}   \lp(I^{(l)})^T \bV \rp.\label{eq:cinfanl16z1}
\end{eqnarray}
A combination of (\ref{eq:cinfanl14z1}), (\ref{eq:cinfanl15z1}), and (\ref{eq:cinfanl16z1}) produces
\begin{eqnarray}
Q & \triangleq &  \lp \lp ((I^{(l)})^T \bV^{\perp})^{-1} \rp^T ((I^{(l)})^T \bV^{\perp})^{-1}   \lp(I^{(l)})^T \bV \bV^T I^{(l)}\rp \rp \nonumber \\
& = &   \lp \lp (I^{(l)})^T\bV^{\perp}(\bV^{\perp})^T I^{(l)}\rp^{-1}  \lp (I^{(l)})^T \lp I - \bV^{\perp} (\bV^{\perp})^T \rp I^{(l)}\rp \rp\nonumber \\
& = &   \lp (I^{(l)})^T\bV^{\perp}(\bV^{\perp})^T I^{(l)}\rp^{-1} -I. \label{eq:cinfanl16a0z1}
\end{eqnarray}
Since all the nonzero eigenvalues of $Q_1$ and $Q$ are identical
\begin{equation}
\lambda_{max}(\Lambda_V^T\Lambda_V) = \lambda_{max}(Q_1) = \lambda_{max}(Q). \label{eq:cinfanl16a1z1}
\end{equation}
Repeating the above with $V$ replaced by $U$, $Q_1$ by $Q_1^{\perp}$, and $Q$ by $Q^{\perp}$ one arrives at the following analogue of (\ref{eq:cinfanl16a1z1})
\begin{equation}
\lambda_{max}(\Lambda_U^T\Lambda_U) = \lambda_{max}(Q_1^{\perp})= \lambda_{max}(Q^{\perp}). \label{eq:cinfanl16a2z1}
\end{equation}
A combination of (\ref{eq:cinfanl12z1}) and (\ref{eq:cinfanl16z1}) - (\ref{eq:cinfanl16a2z1}) completes the proof of the condition (\ref{eq:cinfcor3eq2})'s sufficiency  for the $\ell_0^*-\ell_1^*$-equivalence.
\end{proof}

All the three above corollaries provide useful characterization of the $\ell_0^*-\ell_1^*$-equivalence. Depending on what kind of scenario one faces and what kind of numerical/computational/statistical resources might be available each of them could be used. In the following sections we will focus on a particular type of analysis that will primarily relate to the spectral characterizations given in Corollary \ref{cor:cinfcor3}.

\section{Typical \prp{worst case} analysis of the $\ell_0^*-\ell_1^*$-equivalence}
\label{sec:typwcanl}

In this section we provide an analysis that sheds a bit more light on when the conditions from Corollary \ref{cor:cinfcor3} are indeed met. We will work in a \emph{typical} statistical scenario. We will first assume that $V$ and $U$ are statistical objects and under such an assumption we will try to see if there are regimes where the $\ell_0^*-\ell_1^*$-equivalence generically holds. There are of course many valid candidates for the statistics of $V$ and $U$. We will assume the most generic typical uniformly random scenario from the spectral theory. That means that both $\bar{V}$ and $\bar{U}$ will be Haar distributed. The experts in sparse recovery will quickly recognize that this is in a way analogous to assuming that the locations of the nonzero components of a $k$-sparse vector in the standard compressed sensing are uniformly randomly chosen. Apart from the RDT considerations from \cite{StojnicCSetam09,StojnicICASSP10var,StojnicICASSP09,StojnicISIT2010binary,StojnicRegRndDlt10,StojnicGenLasso10} and the high-dimensional geometry considerations from \cite{DonohoPol,DonohoSigned} we are unaware of any other techniques that can avoid assumptions of this type and still achieve the ultimate exact phase transition (PT) characterizations for the conditions of the type similar to the one appearing in Theorem \ref{thm:cinfthm1} and Corollaries \ref{cor:cinfcor1}-\ref{cor:cinfcor3}.

Conducting the analysis assuming the statistical nature of $V$ and $U$ we will uncover a very interesting and rather remarkable connection, among three a priori not necessarily related fields: 1) the compressed sensing (\bl{\textbf{CS}}), 2) the causal inference (\bl{\textbf{C-inf}}), and 3) the free probability theory (\bl{\textbf{FPT}}). As the connection between the former two exists even outside a statistical context we were able to partially incorporate it in our earlier discussion presented in the previous sections. On the other hand, in the sections that follow, we will deepen our understanding of such a connection while relating it to the free probability theory and a collection of very generic concepts from the modern spectral theory of random matrices.

\subsection{Free probability theory \bl{(FPT)} -- preliminaries}
\label{sec:fptprel}

Since this is the introductory paper where we are establishing the connection between the causal inference and the free probability theory (FPT) we will find it useful to first, in this subsection, recall on some FPT basics. As the FPT theory is mathematically very deep and involved we will restrict ourselves to the introduction of the basic FPT definitions, the explanations of the key technical results, and finally to a brief description related to the practical utilization of these results. After that, in the sections that follow, we will see how some of the introduced FPT concepts can be incorporated to strengthen our understanding of the causal inference itself.

The FPT is, of course, a very generic and abstract concept. Below we focus on some of its key implications related to the spectral theory of random matrices and start by sketching a bit of main motivation behind the FPT. That effectively means that we start with the simplest possible matrices which are of course scalars.

\subsubsection{Basics of FPT -- random scalar variables}
\label{sec:fptprelscalars}

It is well known that if one has two independent random variables $A$ and $B$ with respective pdfs $f_A(\cdot)$ and $f_B(\cdot)$, then the standard way of determining the distribution of their sum or product goes through the characteristic functions and the corresponding inverse Fourier transform considerations. To be a bit more concrete, one first recognizes that the individual characteristic functions for both variables are given as
\begin{eqnarray}
   F_A(jw) & \triangleq &  Ee^{_jwA}=\int e^{_jwa} f_A(a) da \triangleq {\cal F}(f_A(a)) \nonumber \\
   F_B(jw) & \triangleq &  Ee^{_jwB}=\int e^{_jwb} f_B(b) db \triangleq {\cal F}(f_B(b)).  \label{eq:typwcanl1}
\end{eqnarray}
Assuming that
\begin{eqnarray}
C=A+B,\label{eq:typwcanl2}
\end{eqnarray}
we analogously to (\ref{eq:typwcanl1}) also have
\begin{eqnarray}
    F_C(jw)  =   Ee^{_jwC}=\int e^{_jwc} f_C(c) dc.  \label{eq:typwcanl3}
\end{eqnarray}
Moreover,
\begin{eqnarray}
    F_C(jw)  =   Ee^{_jwC} = Ee^{_jw(A+B)} = Ee^{_jwA}Ee^{jwB} = F_A(jw) F_B(jw) = \int e^{_jwa} f_A(a) da \int e^{_jwb} f_B(b) db,  \label{eq:typwcanl4}
\end{eqnarray}
and finally
\begin{eqnarray}
    f_C(c)= {\cal F}^{-1}(F_C(jw)).  \label{eq:typwcanl5}
\end{eqnarray}
It is then easy to see that (\ref{eq:typwcanl4}) and (\ref{eq:typwcanl5}) are sufficient to determine the pdf of $C=A+B$ starting from the individual pdfs of $A$ and $B$. The key that leads to the success of the above mechanism is the introduction of the Fourier transform and the so-called characteristic function. It turns out that in the transform's domain the sum of random variables in a way corresponds to their product and, as a consequence, one can successfully  separate them and then rely on their individual pdfs. While this methodology is fairly simple and has been known for almost two centuries, the existence of a similar one for matrices was not discovered until only a couple of decades ago. Moreover, the path to its discovery turned out to be more thornier and unpredictable than one could have ever imagined. The work od Dan Voiculescu on group theories (see, e.g. \cite{Voic86,Voic87,Voic91}) uncovered it in an almost by-product type of way. Of course, due to an enormous practical importance it immediately drew a substantial interest and in the years that followed immediately after its discovery a few nice results appeared that helped make it presentable in a relatively simple and easily understandable way. We below follow into the same footsteps, leave all the abstractions out, and focus on presenting how the main FPT mechanism actually works (more details can be found in e.g. \cite{Voic86,Voic87,Voic91,NicaSpeich06,Speich14,Haag97,TulVer04}).

\subsubsection{Basics of FPT -- random matrix variables}
\label{sec:fptprelmatrices}

As was the case above for scalars, we here also start with two random variables, $A$ and $B$. This time though, these two variables are symmetric matrices, i.e. $A=A^T\in\mR^{n\times n}$ and $B=B^T\in\mR^{n\times n}$. We will also assume large $n$ regime and that the eigenspaces of these matrices are Haar distributed. Moreover, we will assume that their individual respective spectral laws are $f_A(\cdot)$ and $f_B(\cdot)$. Similarly to what we showed above in the scalar case, we will here also rely on introducing distributional transform. However, differently from the scalar case, here we will introduce not one but three different transforms. We start with the so-called Stieltjes transform (or as we will often call it G-transform) of a pdf $f(\cdot)$
\begin{eqnarray}
    G(z) & \triangleq & \int_{I_f} \frac{f(x)}{z-x} dx, \quad z\in\mC\setminus I_f,  \label{eq:typwcanl6}
\end{eqnarray}
where $I_f$ is the domain of $f(\cdot)$. One then also has the inverse relation (somewhat analogous to the above relation between the inverse Fourier and the underlying pdf of the sum of random variables)
\begin{eqnarray}
    f(x) =  \lim_{\epsilon\rightarrow 0^+} \frac{G(x-i\epsilon)-G(x+i\epsilon)}{2i\pi}
    \quad \mbox{or} \quad    f(x) =  -\lim_{\epsilon\rightarrow 0^+} \frac{\mbox{imag}(G(x+i\epsilon))}{\pi}.   \label{eq:typwcanl7}
\end{eqnarray}
For the above to hold it makes things easier to implicitly assume that $f(x)$ is continuous. We will, however, utilize it even in discrete (or semi-discrete) scenarios since the obvious asymptotic translation to continuity would make it fully rigorous. A bit later though, when we see some concrete examples where things of this nature may appear, we will say a few more words and explain more thoroughly what exactly can
be discrete and how one can deal with such a discreteness. In the meantime we proceed with general principles not necessarily worrying about all the underlying technicalities that may appear in scenarios deviating from the typically seen ones and potentially requiring additional separate addressing. To that end we continue by considering the $R(\cdot)$- and $S(\cdot)$-transforms that satisfy the following
\begin{eqnarray}
R(G(z))+\frac{1}{G(z)}=z,  \label{eq:typwcanl8}
\end{eqnarray}
and
\begin{eqnarray}
S(z)=\frac{1}{R(zS(z))} \quad \mbox{and}\quad R(z)=\frac{1}{S(zR(z))}.\label{eq:typwcanl9}
\end{eqnarray}
Let $f_A(\cdot)$ and $f_B(\cdot)$ be the spectral distributions of $A$ and $B$ and let $R_A(z)$/$S_A(z)$ and $R_B(z)$/$S_B(z)$ be their associated $R(\cdot)$-/$S(\cdot)$-transforms. One then has the following
\begin{center}
 \begin{tcolorbox}[beamer,title=\textbf{Key Voiculescu's FPT concepts \cite{Voic86,Voic87}:},lower separated=false, colback=yellow!95!green!40!white,
colframe=green!45!blue!60!black,coltext=black,fonttitle=\bfseries,width=5in]
\vspace{-.1in}\begin{eqnarray}
\begin{array}{r c l l r c l}
C & = & A+B & \quad \Longrightarrow \quad $ $ & R_C(z) & = & R_A(z)+R_B(z)\\
C & = & AB  & \quad \Longrightarrow \quad $ $ & S_C(z) & = & S_A(z)S_B(z).
\end{array}\label{eq:typwcanl10}
\end{eqnarray}
 \end{tcolorbox}
\end{center}
\noindent Now it is relatively easy to see that (\ref{eq:typwcanl6})-(\ref{eq:typwcanl10}) are sufficient to determine the spectral distribution of the sum or the product of two independent matrices with given spectral densities and the Haar distributed bases of eigenspaces. The above is of course generic principle. It can be applied pretty much always as long as one has access to the statistics of the underlying matrices $A$ and $B$. In the following section we will raise the bar a bit higher and show that in the case of the causal inference one can use all of the above in such a manner that eventually all the quantities of interest are explicitly determined. Moreover, although the methodology may, on occasion, seem a bit involved the final results will turn out to be presentable in  fairly neat and elegant closed forms.

\subsection{Uncovering the \bl{C-inf} $\longleftrightarrow$ \bl{FPT} connection}
\label{sec:cinffpt}

We are now in position to finally move to one of the key aspects of this paper, namely the uncovering of a rather unexpected connection between the C-inf and the FPT. The first part of the connection is in a way implicit and includes what we presented in in the previous section. Namely, utilizing the key compressed sensing concepts and the Random duality theory (RDT) we connected the success of the causal inference to behavior of ceratin algebraic structures. In particular, we have established the so-called $\ell_0^*-\ell_1^*$-equivalence as the key concept in determining the ultimate level of success of C-inf. The second part of the connection builds on the first and proceeds by analyzing the $\ell_0^*-\ell_1^*$-equivalence via the FPT machinery. In the sections that follow we provide such a very detailed and self-contained analysis.

\subsubsection{The spectral approach to the analysis of the $\ell_0^*-\ell_1^*$-equivalence}
\label{sec:fpteqv}

As mentioned earlier, we below rely on the spectral characterization of the $\ell_0^*-\ell_1^*$-equivalence provided in Corollary \ref{cor:cinfcor3}. Clearly, determining the spectrum of $\lambda_V^T\lambda_V\lambda_U^T\lambda_U$ would be sufficient to determine when the $\ell_0^*-\ell_1^*$-equivalence occurs (in fact, determining the edges of the spectrum is sufficient as well). While we will in the sections that follow below indeed determine the spectrum  of $\lambda_V^T\lambda_V\lambda_U^T\lambda_U$, here we note that in the \textbf{\emph{worst case}} that may not be necessary. Namely, in the worst case one has
\begin{eqnarray}
   \lambda_{max}(\lambda_V^T\lambda_V\lambda_U^T\lambda_U) \leq
       \lambda_{max}(\lambda_V^T\lambda_V)\lambda_{max}(\lambda_U^T\lambda_U).\label{eq:typwcanl11}
\end{eqnarray}
In the large $n$ limit due to the concentrations and identically Haar distributed $V$ and $U$ one also has
\begin{eqnarray}
   \lambda_{max}(\lambda_V^T\lambda_V\lambda_U^T\lambda_U) \leq
       \lambda_{max}(\lambda_V^T\lambda_V)\lambda_{max}(\lambda_U^T\lambda_U)\longrightarrow \lp\lambda_{max}(\lambda_V^T\lambda_V)\rp^2.\label{eq:typwcanl12}
\end{eqnarray}
Moreover, in the worst case, $U=V$, the equality is actually achieved since
\begin{eqnarray}
   \lambda_{max}(\lambda_V^T\lambda_V\lambda_V^T\lambda_V)=
   \lambda_{max}((\lambda_V^T\lambda_V)^2)=\lp\lambda_{max}(\lambda_V^T\lambda_V)\rp^2.\label{eq:typwcanl120a}
\end{eqnarray}

\noindent This basically means that in the worst case it is sufficient to consider only the spectrum of $Q$ with
\begin{eqnarray}
Q\triangleq\lambda_V^T\lambda_V.\label{eq:typwcanl12a}
\end{eqnarray}

\subsubsection{The FPT analysis of the spectrum of $Q$}
\label{sec:fpteqvspecwc}

We start by recalling from (\ref{eq:cinfanl16}) and (\ref{eq:cinfanl16a0})
\begin{eqnarray}
Q_1 & \triangleq &  \Lambda_V^T\Lambda_V \nonumber \\
Q   & \triangleq &   \lp (I^{(l)})^T\bV^{\perp}(\bV^{\perp})^T I^{(l)}\rp^{-1} -I \nonumber \\
Sp(Q_1) & \Longleftrightarrow_{\setminus 0} & Sp(Q), \label{eq:typwcanl13}
\end{eqnarray}
where $Sp(\cdot)$ stands for the spectrum of the matrix argument and $\Longleftrightarrow_{\setminus 0}$ means the equivalence of the parts of the spectra outside the zero eigenvalues. It is rather obvious that it will then be sufficient to handle the spectrum of
\begin{eqnarray}
D & \triangleq &  (I^{(l)})^T\bV^{\perp}(\bV^{\perp})^T I^{(l)}. \label{eq:typwcanl14}
\end{eqnarray}
Consider Haar distributed $\bU_D^{\perp}\in\mR^{n\times (n-l)}$ with $(\bU_D^{\perp})^T\bU_D^{\perp}=I_{(n-l)\times (n-l)}$ and let
\begin{eqnarray}
U_D & = &  \begin{bmatrix}
             \bU_D & \bU_D^{\perp}
           \end{bmatrix} \quad \mbox{with} \quad U_D^TU_D=I_{n\times n}. \label{eq:typwcanl15}
\end{eqnarray}
Also, we assume that $\bU_D^{\perp}$ (and $U_D$) are independent of $\bV^{\perp}$. After setting
\begin{eqnarray}
\bar{D} & \triangleq &  (I^{(l)})^TU_D^T\bV^{\perp}(\bV^{\perp})^T U_D I^{(l)}, \label{eq:typwcanl16}
\end{eqnarray}
we have that the spectra of $D$ and $\bar{D}$ are statistically identical, i.e.
\begin{eqnarray}
Sp(D) \triangleq  Sp((I^{(l)})^T\bV^{\perp}(\bV^{\perp})^T I^{(l)}) \Longleftrightarrow_\mP
  Sp((I^{(l)})^TU_D^T\bV^{\perp}(\bV^{\perp})^T U_D I^{(l)}) \triangleq Sp(\bar{D}), \label{eq:typwcanl17}
\end{eqnarray}
where $\Longleftrightarrow_\mP$ stands for the statistical/probabilistic equivalence. Two facts enable the above statistical identity: 1) the spectrum of the projector $\bV^{\perp}(\bV^{\perp})^T$ does not change under pre- and post-unitary multiplications on both sides; and 2) the Haar structure of $\bV^{\perp}$ remains preserved. Modulo zero eigenvalues, we then further have
\begin{eqnarray}
  Sp((I^{(l)})^TU_D^T\bV^{\perp}(\bV^{\perp})^T U_D I^{(l)}) \Longleftrightarrow_{\mP\setminus 0}
  Sp(\bV^{\perp}(\bV^{\perp})^T U_D I^{(l)}(I^{(l)})^TU_D^T) \Longleftrightarrow
  Sp(\bV^{\perp}(\bV^{\perp})^T \bU_D^{\perp}(\bU_D^{\perp})^T), \label{eq:typwcanl18}
\end{eqnarray}
where, similarly as above, $\Longleftrightarrow_{\mP\setminus 0}$ stands for the statistical/probabilistic equivalence in the part of the spectrum outside the zero eignevalues (introduced due to the non-square underlying matrices). Clearly, the key object of our interest below will be
\begin{eqnarray}
\tilde{D} & \triangleq & \bV^{\perp}(\bV^{\perp})^T \bU_D^{\perp}(\bU_D^{\perp})^T, \label{eq:typwcanl19}
\end{eqnarray}
where both $\bV^{\perp}$ and $\bU_D^{\perp}$ are Haar distributed and independent of each other. After setting
\begin{eqnarray}
{\cal V} & \triangleq & \bV^{\perp}(\bV^{\perp})^T \nonumber \\
{\cal U} & \triangleq & \bU_D^{\perp}(\bU_D^{\perp})^T, \label{eq:typwcanl20}
\end{eqnarray}
we easily have from (\ref{eq:typwcanl19})
\begin{eqnarray}
\tilde{D} & \triangleq & {\cal V}{\cal U}, \label{eq:typwcanl20a}
\end{eqnarray}
and below first focus on handling the spectrum and the  corresponding relevant transforms of ${\cal V}$. Since we will be working in the mathematically most challenging large $n$ linear regime, we find it useful to introduce the following large dimensional scalings
\begin{eqnarray}
\beta\triangleq \lim_{n\rightarrow \infty}\frac{k}{n}\quad \mbox{and} \quad \eta\triangleq \lim_{n\rightarrow \infty}\frac{l}{n} \quad \mbox{and} \quad \alpha\triangleq \lim_{n\rightarrow \infty}\frac{m}{n^2}=\lim_{n\rightarrow \infty}\frac{n^2-(n-l)^2}{n^2}=1-(1-\eta)^2. \label{eq:typwcanl21}
\end{eqnarray}

We start with a trivial observation. Let $f_{\calV}(\cdot)$ be the spectral distribution of $\calV$. Then
\begin{eqnarray}
 f_\calV(x)=(1-\beta)\delta(1-x)+\beta\delta(x), \label{eq:typwcanl22}
\end{eqnarray}
where $\delta(\cdot)$ stands for the standard delta function with nonzero value only when its argument takes value zero. Using the definition of the $G$-transform from (\ref{eq:typwcanl6}) we can find
\begin{eqnarray}
    G_\calV(z) = \int \frac{f_\calV(x)}{z-x} dx = \int \frac{(1-\beta)\delta(1-x)+\beta\delta(x)}{z-x} dx=
    \frac{1-\beta}{z-1}+\frac{\beta}{z}=\frac{z-\beta}{z^2-z}.  \label{eq:typwcanl23}
\end{eqnarray}
Also, from (\ref{eq:typwcanl18}) we have
\begin{eqnarray}
R_\calV(y)=z-\frac{1}{y} \quad \mbox{with} \quad y=G_\calV(z) \quad \mbox{and} \quad z=G_\calV^{-1}(y).  \label{eq:typwcanl24}
\end{eqnarray}
From (\ref{eq:typwcanl23}) and (\ref{eq:typwcanl24}) we further find
\begin{eqnarray}
 G_\calV(z)=y \quad  \Longleftrightarrow  \quad \frac{z-\beta}{z^2-z}=y \quad \Longleftrightarrow \quad  z^2y-z(y+1)+\beta=0.  \label{eq:typwcanl25}
\end{eqnarray}
Solving for $z$ gives
\begin{eqnarray}
z=\frac{y+1\pm\sqrt{(y+1)^2-4\beta y}}{2y}.  \label{eq:typwcanl26}
\end{eqnarray}
Combining (\ref{eq:typwcanl24}) and (\ref{eq:typwcanl26}) we obtain for the $R$-transform
\begin{eqnarray}
R_\calV(y)=z-\frac{1}{y}= \frac{y-1\pm\sqrt{(y+1)^2-4\beta y}}{2y},  \label{eq:typwcanl27}
\end{eqnarray}
where we for the completeness adopt the strategy to keep both $\pm$ signs. To determine the $S$-transform we start by combining (\ref{eq:typwcanl9}) and (\ref{eq:typwcanl27})
\begin{eqnarray}
S_\calV(z)=\frac{1}{R_\calV(zS_\calV(z))}= \frac{1}{\frac{zS_\calV(z)-1\pm\sqrt{(zS_\calV(z)+1)^2-4\beta zS_\calV(z)}}{2zS_\calV(z)}}.  \label{eq:typwcanl28}
\end{eqnarray}
After a bit of algebraic transformations we have
\begin{eqnarray}
\begin{array}{c r c l}
& zS_\calV(z)-1-2z &  =  & \mp\sqrt{(zS_\calV(z)+1)^2-4\beta zS_\calV(z)} \\
\Longleftrightarrow & (zS_\calV(z)-1-2z)^2 &  =  & (zS_\calV(z)+1)^2-4\beta zS_\calV(z) \\
\Longleftrightarrow & (zS_\calV(z))^2-2(2z+1)zS_\calV(z)+(2z+1)^2 &  =  & (zS_\calV(z))^2+2zS_\calV+1-4\beta zS_\calV(z) \\
\Longleftrightarrow & -2(2z+1)zS_\calV(z)+4z^2+4z &  =  &  2zS_\calV(z)-4\beta zS_\calV(z) \\
\Longleftrightarrow & 4z^2+4z &  =  &  (4z^2+4z)S_\calV(z)-4\beta zS_\calV(z) \\
\Longleftrightarrow & z+1 &  =  &  S_\calV(z)(z+1-\beta).
\end{array}\label{eq:typwcanl29}
\end{eqnarray}
From (\ref{eq:typwcanl29}) we finally have
\begin{eqnarray}
S_\calV(z)=\frac{z+1}{z+1-\beta}.\label{eq:typwcanl30}
\end{eqnarray}
As this is a very generic result it is useful to have it formalized in the following lemma.
\begin{lemma}
  Let $\bV^{\perp}\in\mR^{n\times (n-k)}$ be Haar distributed unitary basis of an $(n-k)$-dimensional subspace of $\mR^n$. Let $\calV$ be as in (\ref{eq:typwcanl20}), i.e.
   \begin{eqnarray}
  {\cal V}  \triangleq  \bV^{\perp}(\bV^{\perp})^T.\label{eq:typwclemma1eq1}
\end{eqnarray}
In the large $n$ linear regime, with $\beta\triangleq\lim_{n\rightarrow\infty}\frac{k}{n}$, the $S$-transform of the spectral density of $\calV$, $f_\calV(\cdot)$, is
\begin{eqnarray}
S_\calV(z)=\frac{z+1}{z+1-\beta}.\label{eq:typwclemma1eq2}
\end{eqnarray}\label{lemma:typwclemma1}
\end{lemma}
\begin{proof}
  Follows from the above discussion.
\end{proof}

Since $\calV$ and $\calU$  are  structurally identical (with the only difference being one of their dimensions) we easily have
\begin{eqnarray}
S_\calU(z)=\frac{z+1}{z+1-\eta}.\label{eq:typwcanl31}
\end{eqnarray}
A combination of (\ref{eq:typwcanl10}), (\ref{eq:typwcanl20a}), (\ref{eq:typwcanl30}), and (\ref{eq:typwcanl31}) gives
\begin{eqnarray}
S_{\tilde{D}}(z)=\frac{(z+1)^2}{(z+1-\beta)(z+1-\eta)}.\label{eq:typwcanl32}
\end{eqnarray}
From (\ref{eq:typwcanl9}) we also have
\begin{eqnarray}
R_{\tilde{D}}(z)=\frac{1}{S_{\tilde{D}}(zR_{\tilde{D}}(z))}=
\frac{1}{\frac{(zR_{\tilde{D}}(z)+1)^2}{(zR_{\tilde{D}}(z)+1-\beta)(zR_{\tilde{D}}(z)+1-\eta)}}=
\frac{(zR_{\tilde{D}}(z)+1-\beta)(zR_{\tilde{D}}(z)+1-\eta)}{(zR_{\tilde{D}}(z)+1)^2}.\label{eq:typwcanl33}
\end{eqnarray}
Moreover, (\ref{eq:typwcanl7}) gives
\begin{eqnarray}
R_{\tilde{D}}(G_{\tilde{D}}(z))+\frac{1}{G_{\tilde{D}}(z)}=z,\label{eq:typwcanl34}
\end{eqnarray}
and
\begin{eqnarray}
G_{\tilde{D}}(z)R_{\tilde{D}}(G_{\tilde{D}}(z))=zG_{\tilde{D}}(z)-1,\label{eq:typwcanl35}
\end{eqnarray}
From (\ref{eq:typwcanl33}) one further finds
\begin{eqnarray}
R_{\tilde{D}}(G_{\tilde{D}}(z))=
\frac{(G_{\tilde{D}}(z)R_{\tilde{D}}(G_{\tilde{D}}(z))+1-\beta)(G_{\tilde{D}}(z)R_{\tilde{D}}(G_{\tilde{D}}(z))+1-\eta)}
{(G_{\tilde{D}}(z)R_{\tilde{D}}(G_{\tilde{D}}(z))+1)^2}.\label{eq:typwcanl36}
\end{eqnarray}
After plugging (\ref{eq:typwcanl35}) in (\ref{eq:typwcanl36}) we have
\begin{eqnarray}
R_{\tilde{D}}(G_{\tilde{D}}(z))=
\frac{(zG_{\tilde{D}}(z)-1+1-\beta)(zG_{\tilde{D}}(z)-1+1-\eta)}
{(zG_{\tilde{D}}(z)-1+1)^2}=\frac{(zG_{\tilde{D}}(z)-\beta)(zG_{\tilde{D}}(z)-\eta)}
{(zG_{\tilde{D}}(z))^2}.\label{eq:typwcanl37}
\end{eqnarray}
A combination of (\ref{eq:typwcanl34}) and (\ref{eq:typwcanl37}) further gives
\begin{eqnarray}
z-\frac{1}{G_{\tilde{D}}(z)}=
 \frac{(zG_{\tilde{D}}(z)-\beta)(zG_{\tilde{D}}(z)-\eta)}
{(zG_{\tilde{D}}(z))^2}.\label{eq:typwcanl38}
\end{eqnarray}
From (\ref{eq:typwcanl38})  we quickly find
\begin{eqnarray}
z^3(G_{\tilde{D}}(z))^2- z^2G_{\tilde{D}}(z)=
 z^2(G_{\tilde{D}}(z))^2-(\beta+\eta)zG_{\tilde{D}}(z)+\beta\eta,\label{eq:typwcanl39}
\end{eqnarray}
and
\begin{eqnarray}
(G_{\tilde{D}}(z))^2(z^3-z^2)- G_{\tilde{D}}(z)(z^2-z(\beta+\eta))-\beta\eta=0.\label{eq:typwcanl40}
\end{eqnarray}
Solving for $G_{\tilde{D}}(z)$ finally gives
\begin{eqnarray}
G_{\tilde{D}}^{\pm}(z)=\frac{z^2-z(\beta+\eta)\pm\sqrt{(z^2-z(\beta+\eta))^2+4\beta\eta(z^3-z^2)}}{2(z^3-z^2)},\label{eq:typwcanl41}
\end{eqnarray}
or
\begin{eqnarray}
G_{\tilde{D}}^{\pm}(z)=\frac{z-(\beta+\eta)\pm\sqrt{(z-(\beta+\eta))^2+4\beta\eta(z-1)}}{2(z^2-z)}.\label{eq:typwcanl42}
\end{eqnarray}

The above is sufficient to establish the following lemma.
\begin{lemma}
  Let $\bV^{\perp}\in\mR^{n\times (n-k)}$ and $\bU_D^{\perp}\in\mR^{n\times (n-k)}$ be Haar distributed unitary bases of $(n-k)$-dimensional subspaces of $\mR^n$. Let $\calV$ and $\calU$ be as in (\ref{eq:typwcanl20}) and $\tilde{D}$ as in (\ref{eq:typwcanl20a}), i.e.
   \begin{eqnarray}
  {\cal V} & \triangleq & \bV^{\perp}(\bV^{\perp})^T \nonumber \\
    {\cal U} & \triangleq & \bU_D^{\perp}(\bU_D^{\perp})^T\nonumber \\
     \tilde{D} & \triangleq & {\cal V}{\cal U}.\label{eq:typwclemma1aeq1}
\end{eqnarray}
In the large $n$ linear regime, with $\beta\triangleq\lim_{n\rightarrow\infty}\frac{k}{n}$, the $G$-transform of the spectral density of $\tilde{D}$, $f_{\tilde{D}}(\cdot)$, is
\begin{eqnarray}
G_{\tilde{D}}^{\pm}(z)=\frac{z-(\beta+\eta)\pm\sqrt{(z-(\beta+\eta))^2+4\beta\eta(z-1)}}{2(z^2-z)}.\label{eq:typwclemma1aeq2}
\end{eqnarray}\label{lemma:typwclemma1a}
\end{lemma}
\begin{proof}
  Follows from the above discussion. The ``$+$/$-$" signs are taken for negative/positive imaginary part under the root.
\end{proof}

One then relies on (\ref{eq:typwcanl7}) to determine $f_{\tilde{D}}(x)$ as
\begin{eqnarray}
     f_{\tilde{D}}(x) =  -\lim_{\epsilon\rightarrow 0^+} \frac{\mbox{imag}(G_{\tilde{D}}(x+i\epsilon))}{\pi}.   \label{eq:typwcanl43}
\end{eqnarray}
The above is a generic procedure and we in Figure \ref{fig:cinfspecGplusG} show the results that one can get for two concrete values $\beta=0.2$ and $\eta=0.6$. One should note that it is not clear \emph{a priori} which of the two $\pm$ signs should be used. As Figure \ref{fig:cinfspecGplusG} indicates one most definitely has to be fairly careful and account for both signs. From Figure \ref{fig:cinfspecGplusG} one further observes that there are four critical points in the spectrum itself: the locations of the two delta functions, zero and one, and two edges of the spectrum's bulk, $x_l$ and $x_u$. The values of these points are shown in the plots on the right hand side. In general one can actually determine their closed forms as well. Moreover, it turns out that one can determine the closed form of the entire spectral function. The section that follows analyzes the spectrum of $\tilde{D}$ in more details and eventually provides the closed form expressions for all the relevant spectral features.

\begin{figure}[htb]
\begin{minipage}[b]{.5\linewidth}
\centering
\centerline{\epsfig{figure=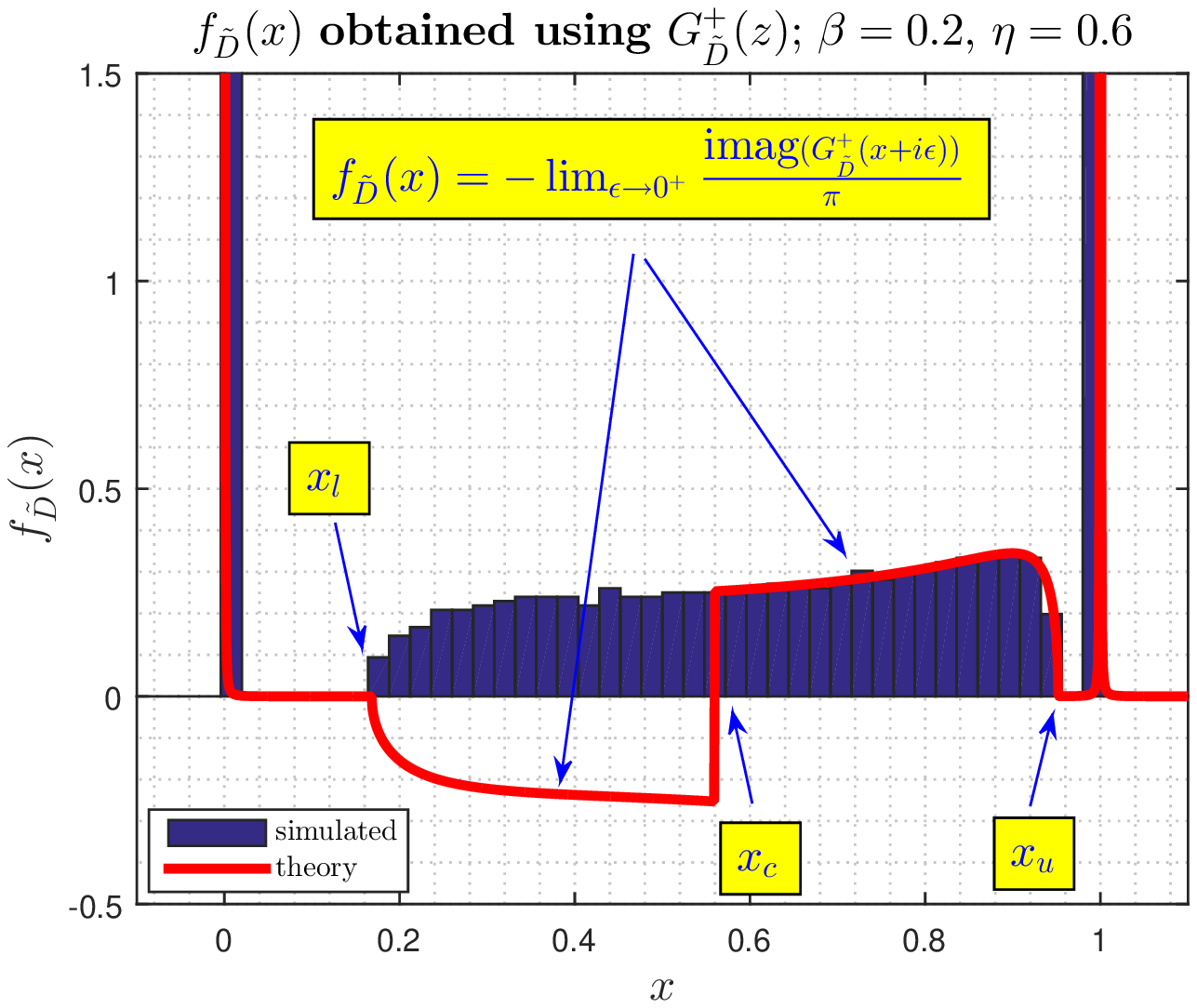,width=9cm,height=7cm}}
\end{minipage}
\begin{minipage}[b]{.5\linewidth}
\centering
\centerline{\epsfig{figure=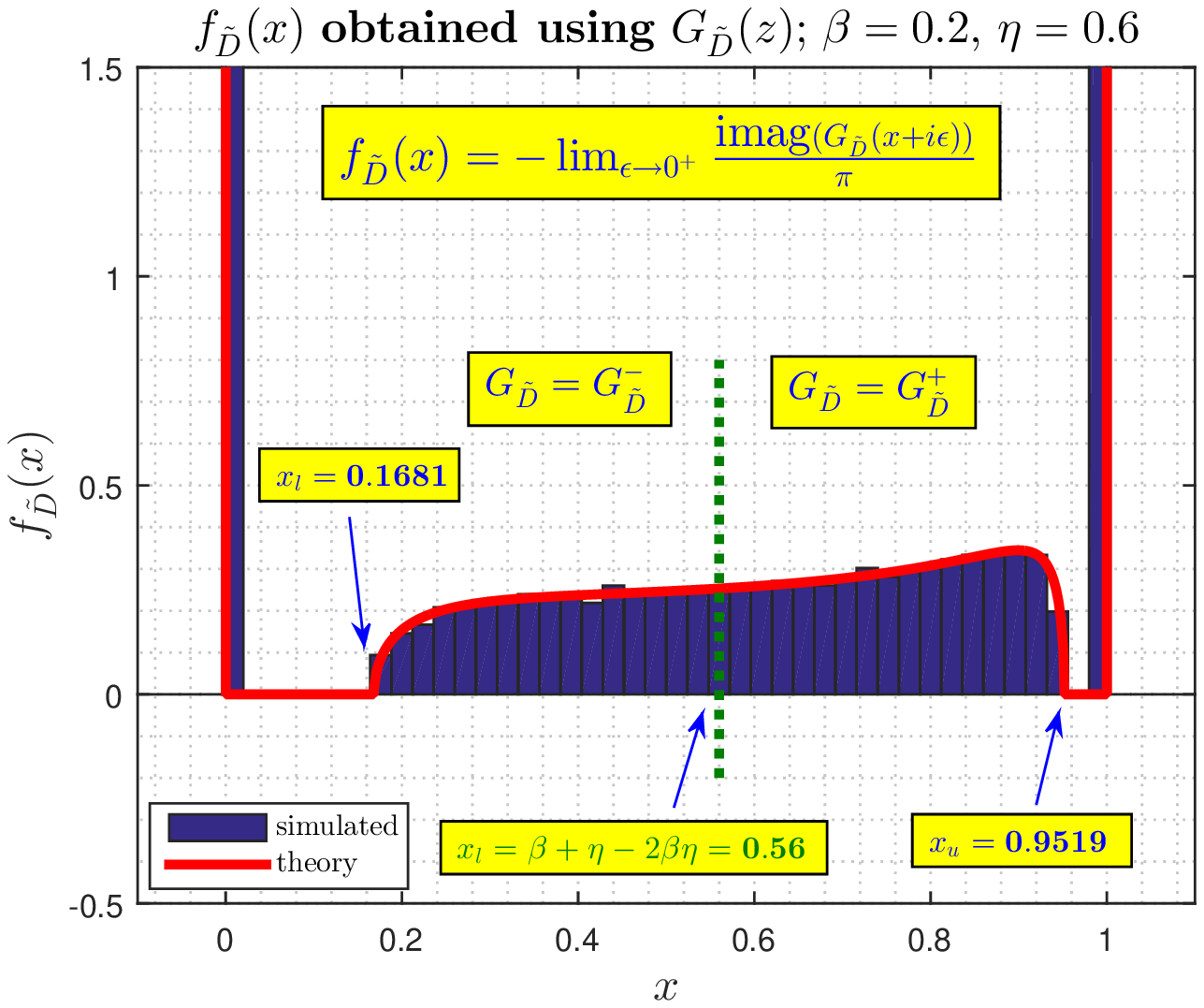,width=9cm,height=7cm}}
\end{minipage}
\caption{Both $G_{\tilde{D}}^+(z)$ and $G_{\tilde{D}}^-(z)$ need to be taken into account}
\label{fig:cinfspecGplusG}
\end{figure}

\subsubsubsection{The spectrum of $\tilde{D}$ -- closed form expressions}
\label{sec:fpteqvspecwctildeD}

As one of our main concerns in this paper is the utilization of the final results that we will get in this section and not necessarily the presentation of the tiny details needed to get them, we will sketch all the key arguments and leave out all the unnecessary minute details. However, we do emphasize that the sketch will contain all the key pointers so that with a little bit of effort one, if in a need, can fill in all the missing pieces of the overall mosaic.

As mentioned above, looking at the denominator of (\ref{eq:typwcanl42}) and keeping in mind the $f\longleftrightarrow G$ connection from (\ref{eq:typwcanl43}) one observes that the pdf of interest, $f_{\tilde{D}}(x)$, potentially has two delta functions, one at zero and the other one at one. Moreover, the bulk of the spectrum will be in the range where the real part under the root is negative. It also goes almost without saying that the entire spectrum will be located between zero and one. Finally, the breaking point, $x_c$, where one needs to switch from $G_{\tilde{D}}^+(z)$ to $G_{\tilde{D}}^-(z)$ in (\ref{eq:typwcanl43}) is determined as the value where the imaginary part under the root changes its sign. Equipped with these observations one can then proceed to actually concretely determine some of the relevant quantities.

Based on what we have just observed above, we first express $f_{\tilde{D}}(x)$ as the sum of its three key constitutive parts (two delta functions and the bulk)
\begin{eqnarray}
     f_{\tilde{D}}(x) = f_0\delta (x-0) +f_{\tilde{D}}^{(b)}(x)+f_1\delta(x-1) .   \label{eq:typwcanl44}
\end{eqnarray}
From (\ref{eq:typwcanl44}) one has that $f_{\tilde{D}}(x)$ will be fully specified if one can determine the delta multipliers $f_0$ and $f_1$, and the bulk pdf $f_{\tilde{D}}^{(b)}(\cdot)$.

\underline{\bl{\textbf{ 1) Finding $f_0$:}}} To determine $f_0$ we start by observing from (\ref{eq:typwcanl43}) for $x=0$
\begin{eqnarray}
     f_{\tilde{D}}(0) =  -\lim_{\epsilon\rightarrow 0^+} \frac{\mbox{imag}(G_{\tilde{D}}(i\epsilon))}{\pi}.   \label{eq:typwcanl45}
\end{eqnarray}
Utilizing (\ref{eq:typwcanl42}) we further have
\begin{eqnarray}
     f_{\tilde{D}}(0) &  =  & -\frac{1}{\pi}\lim_{\epsilon\rightarrow 0^+} \mbox{imag}\lp \frac{i\epsilon-(\beta+\eta)\pm\sqrt{(i\epsilon-(\beta+\eta))^2+4\beta\eta(i\epsilon-1)}}{2((i\epsilon)^2-i\epsilon)}\rp \nonumber \\
      &  =  & -\frac{1}{\pi}\lim_{\epsilon\rightarrow 0^+} \mbox{imag}\lp \frac{-(\beta+\eta)\pm\sqrt{-\epsilon^2+(\beta+\eta)^2-4\beta\eta -2i\epsilon (\beta+\eta-2\beta\eta)}}{2(-\epsilon^2-i\epsilon)}\rp \nonumber \\
      &  =  & -\frac{1}{\pi}\lim_{\epsilon\rightarrow 0^+} \mbox{imag}\lp \frac{-(\beta+\eta)\pm\sqrt{(\beta-\eta)^2  }}{2(-i\epsilon)}\rp \nonumber \\
      &  =  & -\frac{1}{\pi}\lim_{\epsilon\rightarrow 0^+} \mbox{imag}\lp \frac{-(\beta+\eta)-|\beta-\eta| }{-2i\epsilon}\rp \nonumber \\
      &  =  & \lp \beta+\eta+|\beta-\eta|  \rp \lp -\frac{1}{\pi}\lim_{\epsilon\rightarrow 0^+} \mbox{imag}\lp \frac{1}{i\epsilon}\rp \rp\nonumber \\
      &  =  & \max(\beta,\eta)\delta(0),   \label{eq:typwcanl46}
\end{eqnarray}
where the fourth equality (the choice of the ``$-$" sign in $\pm$) follows since $0\leq x_c$ (the spectrum belongs to the interval $[0,1]$ and $x_c$ must be in the spectrum) and the last equality follows since by convention
\begin{eqnarray}
 \delta(0)=\lp -\frac{1}{\pi}\lim_{\epsilon\rightarrow 0^+} \mbox{imag}\lp \frac{1}{i\epsilon}\rp \rp.   \label{eq:typwcanl47}
\end{eqnarray}
To see the rationale behind (\ref{eq:typwcanl47}) we briefly digress and start with
\begin{equation}
g(x)=\delta(x). \label{eq:typwcanl48}
\end{equation}
Then from (\ref{eq:typwcanl6})
\begin{equation}
G(z)=\int_x \frac{\delta(x)dx}{z-x}=\frac{1}{z}, \label{eq:typwcanl49}
\end{equation}
and from (\ref{eq:typwcanl7})
\begin{equation}
\delta(x)=-\lim_{\epsilon\rightarrow 0^+}\frac{\mbox{imag}(G(x+i\epsilon))}{\pi}. \label{eq:typwcanl50}
\end{equation}
For $x=0$ then
\begin{equation}
\delta(0)=-\lim_{\epsilon\rightarrow 0^+}\frac{\mbox{imag}(G(i\epsilon))}{\pi}
=-\lim_{\epsilon\rightarrow 0^+}\mbox{imag}\lp\frac{1}{\pi i\epsilon}\rp
=-\frac{1}{\pi}\lim_{\epsilon\rightarrow 0^+}\mbox{imag}\lp\frac{1}{i\epsilon}\rp, \label{eq:typwcanl51}
\end{equation}
which is identical to (\ref{eq:typwcanl47}). The above description of the delta function may not necessarily be the most adequate one. However, for what we need here it is conceptually sufficient. Namely, we are here interested in determining the proportionality constants that multiply the delta functions rather than the functions' expressions themselves. One way to make everything more adequate would be to translate everything into the continuous domain by choosing a continuous function as an asymptotic replacement for $\delta(x)$. For example, one can use the Gaussian continual approximation
\begin{equation}
\delta(x)=\lim_{\sigma\rightarrow 0^+}\frac{e^{-\frac{x^2}{2\sigma^2}}}{\sqrt{2\pi\sigma^2}}. \label{eq:typwcanl52}
\end{equation}
Then all the above holds for small $\sigma=\epsilon\sqrt{\pi}/2$ and
\begin{equation}
\delta(x)\rightarrow \lim_{\sigma\rightarrow 0^+}\frac{e^{-\frac{x^2}{2\sigma^2}}}{\sqrt{2\pi\sigma^2}}
\rightarrow \lim_{\sigma\rightarrow 0^+}\frac{e^{-\frac{x^2}{\pi\epsilon^2}}}{\pi\epsilon}\quad \mbox{and} \quad
\delta(0)\rightarrow \lim_{\sigma\rightarrow 0^+}\frac{1}{\pi\epsilon}. \label{eq:typwcanl53}
\end{equation}
The difference though would be that when computing and maneuvering with all the above transforms one would need to account for the resulting/induced $\epsilon$-differences. These are of course practically and conceptually negligible and all the results that we presented would continue to hold in the limit of small $\sigma$ or $\epsilon$. However, the writing would be substantially more tedious and a tone of additional minute details would need to be added to express all the $\epsilon$-type of modifications and to show that their contributions are indeed marginal. These things are conceptually highly trivial but require a tedious detail-oriented work. Since, on the other hand, they contribute exactly nothing to the essence of the arguments and final results we chose to operate in a semi-discrete domain with the delta functions. As a consequence one has the expressions given in (\ref{eq:typwcanl47}) and (\ref{eq:typwcanl51}). We believe that a little bit of conventional inadequacy is better than to overwhelm the presentation with a tone of details which would avoid it but at the same time make the overall content less accessible and potentially even less understandable.

\underline{\bl{\textbf{ 2) Finding $f_1$:}}} To determine $f_1$ we follow the above methodology and start by observing from (\ref{eq:typwcanl43}) for $x=1$
\begin{eqnarray}
     f_{\tilde{D}}(1) =  -\lim_{\epsilon\rightarrow 0^+} \frac{\mbox{imag}(G_{\tilde{D}}(1+i\epsilon))}{\pi}.   \label{eq:typwcanl54}
\end{eqnarray}
Further utilization of (\ref{eq:typwcanl42}) gives
\begin{eqnarray}
     f_{\tilde{D}}(1) &  =  & -\frac{1}{\pi}\lim_{\epsilon\rightarrow 0^+} \mbox{imag}\lp \frac{1+i\epsilon-(\beta+\eta)\pm\sqrt{(1+i\epsilon-(\beta+\eta))^2+4\beta\eta i\epsilon}}{2((1+i\epsilon)^2-1-i\epsilon)}\rp \nonumber \\
      &  =  & -\frac{1}{\pi}\lim_{\epsilon\rightarrow 0^+} \mbox{imag}\lp \frac{1-(\beta+\eta)\pm\sqrt{-\epsilon^2+(1-(\beta+\eta))^2 -2i\epsilon (-1+\beta+\eta-2\beta\eta)}}{2(-\epsilon^2+i\epsilon)}\rp \nonumber \\
      &  =  & -\frac{1}{\pi}\lim_{\epsilon\rightarrow 0^+} \mbox{imag}\lp \frac{1-(\beta+\eta)\pm\sqrt{(1-(\beta-\eta))^2  }}{2(i\epsilon)}\rp \nonumber \\
      &  =  & -\frac{1}{\pi}\lim_{\epsilon\rightarrow 0^+} \mbox{imag}\lp \frac{1-(\beta+\eta)+|1-(\beta+\eta)| }{2i\epsilon}\rp \nonumber \\
      &  =  & \lp 1-(\beta+\eta)+|1-(\beta+\eta)|  \rp \lp -\frac{1}{\pi}\lim_{\epsilon\rightarrow 0^+} \mbox{imag}\lp \frac{1}{i\epsilon}\rp \rp\nonumber \\
      &  =  & \max(1-(\beta+\eta),0)\delta(0),   \label{eq:typwcanl55}
\end{eqnarray}
where the fourth equality (the choice of the ``$+$" sign in $\pm$) follows since now $x_c\leq 1$ and the last equality follows by the above discussed $\delta(0)$ convention.

\underline{\bl{\textbf{ 3) Finding $f_{\tilde{D}}^{(b)}(x)$:}}} To determine $f_{\tilde{D}}^{(b)}(x)$ for $x\notin \{0,1\}$ we again start with (\ref{eq:typwcanl43}) and wrte the following for a general $x$ from the bulk of the spectrum
\begin{eqnarray}
     f_{\tilde{D}}^{(b)}(x) =  -\lim_{\epsilon\rightarrow 0^+} \frac{\mbox{imag}(G_{\tilde{D}}(x+i\epsilon))}{\pi}.   \label{eq:typwcanl56}
\end{eqnarray}
Relying once again on (\ref{eq:typwcanl42}) we, for $x\notin \{0,1\}$, have
\begin{eqnarray}
     f_{\tilde{D}}^{(b)}(x) &  =  & -\frac{1}{\pi}\lim_{\epsilon\rightarrow 0^+} \mbox{imag}\lp \frac{x+i\epsilon-(\beta+\eta)\pm\sqrt{(x+i\epsilon-(\beta+\eta))^2+4\beta\eta (x+i\epsilon-1}}{2((x+i\epsilon)^2-x-i\epsilon)}\rp \nonumber \\
      &  =  & -\frac{1}{\pi}\lim_{\epsilon\rightarrow 0^+} \mbox{imag}\lp \frac{x-(\beta+\eta)\pm\sqrt{-\epsilon^2+(x-(\beta+\eta))^2+4\beta\eta(x-1)-2i\epsilon (-x+\beta+\eta-2\beta\eta)}}{2(x^2-x-\epsilon^2+i\epsilon(2x-1))}\rp \nonumber \\
      &  =  & -\frac{1}{\pi}\lim_{\epsilon\rightarrow 0^+} \mbox{imag}\lp \frac{x-(\beta+\eta)\pm\sqrt{(x-(\beta+\eta))^2+4\beta\eta(x-1)-2i\epsilon (-x+\beta+\eta-2\beta\eta)}}{2(x^2-x)}\rp \nonumber \\
      &  =  & -\frac{1}{\pi}\lim_{\epsilon\rightarrow 0^+} \mbox{imag}\lp \frac{\pm\sqrt{(x-(\beta+\eta))^2+4\beta\eta(x-1)-2i\epsilon (-x+\beta+\eta-2\beta\eta)}}{2(x^2-x)}\rp.   \label{eq:typwcanl56}
\end{eqnarray}
Now, since one is interested in the imaginary part of interest is the region of $x$ where the real part under the root is negative (outside that region, i.e in the region of $x$ where the real part under the root is nonnegative $f_{\tilde{D}}^{(b)}(x)$ is zero). To determine the region of interest we start by setting
\begin{eqnarray}
      T_{\tilde{D}}  \triangleq   \{x\in\mR| (x-(\beta+\eta))^2+4\beta\eta(x-1) \leq 0\} \quad \mbox{and}\quad
      x_c\triangleq  \beta+\eta-2\beta\eta.
          \label{eq:typwcanl57}
\end{eqnarray}
To explicitly characterize $T_{\tilde{D}}$ we look at the following
\begin{eqnarray}
\begin{array}{r r c l}
  $ $  & (x-(\beta+\eta))^2+4\beta\eta(x-1) & = &  0  \\
 \Longleftrightarrow  &  x^2-2x(\beta+\eta-2\beta\eta)+(\beta+\eta)^2-4\beta\eta & = & 0 \\
 \Longleftrightarrow  &  x^2-2x(\beta+\eta-2\beta\eta)+(\beta-\eta)^2 & = & 0.
 \end{array}\label{eq:typwcanl58}
\end{eqnarray}
Solving for $x$ one finds
\begin{equation}
 x  =  \frac{2(\beta+\eta-2\beta\eta)\pm\sqrt{(2(\beta+\eta-2\beta\eta))^2-4(\beta-\eta)^2}}{2} = \beta+\eta-2\beta\eta \pm \sqrt{(\beta+\eta-2\beta\eta)^2-(\beta-\eta)^2}.
 \label{eq:typwcanl59}
\end{equation}
Setting
\begin{eqnarray}
 x_l & \triangleq & \beta+\eta-2\beta\eta  -  \sqrt{(\beta+\eta-2\beta\eta)^2-(\beta-\eta)^2} \nonumber \\
 x_u & \triangleq & \beta+\eta-2\beta\eta  +  \sqrt{(\beta+\eta-2\beta\eta)^2-(\beta-\eta)^2},
 \label{eq:typwcanl60}
\end{eqnarray}
one has
\begin{eqnarray}
T_{\tilde{D}}=\{x\in\mR| x\in [x_l,x_u]\}.
 \label{eq:typwcanl61}
\end{eqnarray}
Moreover, from (\ref{eq:typwcanl60}), one also has
\begin{eqnarray}
 0\leq x_l \leq x_u\leq 1,
 \label{eq:typwcanl62}
\end{eqnarray}
with
\begin{eqnarray}
 x_l=0 \quad \mbox{if} \quad \beta=\eta \quad \mbox{and} \quad x_u=1  \quad \mbox{if} \quad \beta=\eta=0.5.
 \label{eq:typwcanl63}
\end{eqnarray}
The first two inequalities in (\ref{eq:typwcanl60}) are trivial, whereas the third one follows after noting
\begin{eqnarray}
\beta+\eta-2\beta\eta\leq \max(\beta,1-\beta)\leq 1,
 \label{eq:typwcanl64}
\end{eqnarray}
and observing the following sequence
\begin{eqnarray}
\begin{array}{r r c l}
  &  \beta+\eta-2\beta\eta  +  \sqrt{(\beta+\eta-2\beta\eta)^2-(\beta-\eta)^2} & \leq & 1 \\
 \Longleftrightarrow   &    (\beta+\eta-2\beta\eta)^2-(\beta-\eta)^2 & \leq & (1-(\beta+\eta-2\beta\eta))^2 \\
 \Longleftrightarrow   &    -(\beta-\eta)^2 & \leq & 1-2(\beta+\eta-2\beta\eta) \\
 \Longleftrightarrow   &    -(\beta-\eta)^2+2(\beta+\eta-2\beta\eta)-1 & \leq & 0 \\
 \Longleftrightarrow   &    -(\beta+\eta)^2+2(\beta+\eta)-1 & \leq & 0 \\
 \Longleftrightarrow   &    -(1-(\beta+\eta))^2 & \leq & 0.
\end{array}
 \label{eq:typwcanl65}
\end{eqnarray}
Returning to (\ref{eq:typwcanl56}) we further have for $x\in T_{\tilde{D}}=[x_l,x_u]$
\begin{eqnarray}
     f_{\tilde{D}}^{(b)}(x)
      &  =  & -\frac{1}{\pi}\lim_{\epsilon\rightarrow 0^+} \mbox{imag}\lp \frac{\pm\sqrt{(x-(\beta+\eta))^2+4\beta\eta(x-1)-2i\epsilon (-x+\beta+\eta-2\beta\eta)}}{2(x^2-x)}\rp \nonumber \\
            &  =  & -\frac{1}{\pi}\lim_{\epsilon\rightarrow 0^+} \mbox{imag}\lp \frac{\pm\sqrt{(x-(\beta+\eta))^2+4\beta\eta(x-1)-2i\epsilon (x_c-x)}}{2(x^2-x)}\rp \nonumber \\
                        &  =  & -\frac{1}{\pi}\lim_{\epsilon\rightarrow 0^+} \mbox{imag}\lp \frac{i\sqrt{-(x-(\beta+\eta))^2-4\beta\eta(x-1)}}{2(x^2-x)}\rp,   \label{eq:typwcanl66}
\end{eqnarray}
where the ``$+$" plus sign is chosen if $x_c\leq x\leq x_u$ and the ``$-$" sign is chosen if $x_l\leq x\leq x_c$. Finally, from (\ref{eq:typwcanl66}) one easily finds
\begin{eqnarray}
     f_{\tilde{D}}^{(b)}(x)   &  =  &
               \frac{\sqrt{-(x-(\beta+\eta))^2-4\beta\eta(x-1)}}{2\pi(x-x^2)} \quad \mbox{ if } \quad x_l\leq x\leq x_u.
     \label{eq:typwcanl68}
\end{eqnarray}

 The above is then sufficient to completely characterize the spectral distribution $f_{\tilde{D}}(x)$. We summarize the results in the following lemma.

 \begin{lemma}\label{lemma:typwclemma2}
Assume large $n$ linear regime with
\begin{eqnarray}
\beta\triangleq \lim_{n\rightarrow \infty}\frac{k}{n}\quad \mbox{and} \quad \eta\triangleq \lim_{n\rightarrow \infty}\frac{l}{n}. \label{eq:typwclemma2eq1}
\end{eqnarray}
Let $\bV^{\perp}\in\mR^{n\times (n-k)}$ be a Haar distributed basis of an $n-k$-dimensional subspace of $R^n$. Analogously, let $\bU_D^{\perp}\in\mR^{n\times (n-k)}$ be a Haar distributed basis of an $n-l$-dimensional subspace of $R^n$. Moreover, let $\bV^{\perp}\in\mR^{n\times (n-k)}$ and $\bU_D^{\perp}\in\mR^{n\times (n-k)}$ be independent of each other. Also, let ${\cal V}$, ${\cal U}$, and $\tilde{D}$ be as defined in (\ref{eq:typwcanl20}) and (\ref{eq:typwcanl20a}), i.e. let
\begin{eqnarray}
{\cal V} & \triangleq & \bV^{\perp}(\bV^{\perp})^T \nonumber \\
{\cal U} & \triangleq & \bU_D^{\perp}(\bU_D^{\perp})^T\nonumber \\
 \tilde{D} & \triangleq & {\cal V}{\cal U}. \label{eq:typwclemma2eq2}
\end{eqnarray}
Set $x_l$ and $x_u$ as in (\ref{eq:typwcanl60}), i.e.
\begin{eqnarray}
 x_l & \triangleq & \beta+\eta-2\beta\eta  -  \sqrt{(\beta+\eta-2\beta\eta)^2-(\beta-\eta)^2} \nonumber \\
 x_u & \triangleq & \beta+\eta-2\beta\eta  +  \sqrt{(\beta+\eta-2\beta\eta)^2-(\beta-\eta)^2},
 \label{eq:typwclemma2eq3}
\end{eqnarray}
Then the limiting spectral distribution of $\tilde{D}$, $f_{\tilde{D}}(x)$, is
\begin{equation}
     f_{\tilde{D}}(x)     =   f_0\delta(x) +  f_{\tilde{D}}^{(b)}(x) + f_1f_0\delta(x-1)=
              \max(\beta,\eta)\delta(x) + f_{\tilde{D}}^{(b)}(x)
              +\max(1-(\beta+\eta),0)\delta(x-1),
     \label{eq:typwclemma2eq4}
\end{equation}
with
\begin{eqnarray}
     f_{\tilde{D}}^{(b)}(x)   &  =  &
            \begin{cases}
              \frac{\sqrt{-(x-(\beta+\eta))^2-4\beta\eta(x-1)}}{2\pi(x-x^2)}, & \mbox{if }  x_l\leq x\leq x_u. \\
              0, & \mbox{otherwise}.
            \end{cases}
     \label{eq:typwclemma2eq5}
\end{eqnarray}
 \end{lemma}
\begin{proof}
  Follows through a combination of (\ref{eq:typwcanl44}), (\ref{eq:typwcanl46}),  (\ref{eq:typwcanl55}),  (\ref{eq:typwcanl68}), and the above discussion.
\end{proof}

In Figure \ref{fig:ftildeDbeta01eta08} we show the spectral function obtained based on the above lemma for $\beta=0.1$ and $\eta=0.8$. We observe a very strong agreement between the simulated results and the above theoretical predictions. Simulation results were obtained using moderately large $n=4000$.

\begin{figure}[htb]
\centering
\centerline{\epsfig{figure=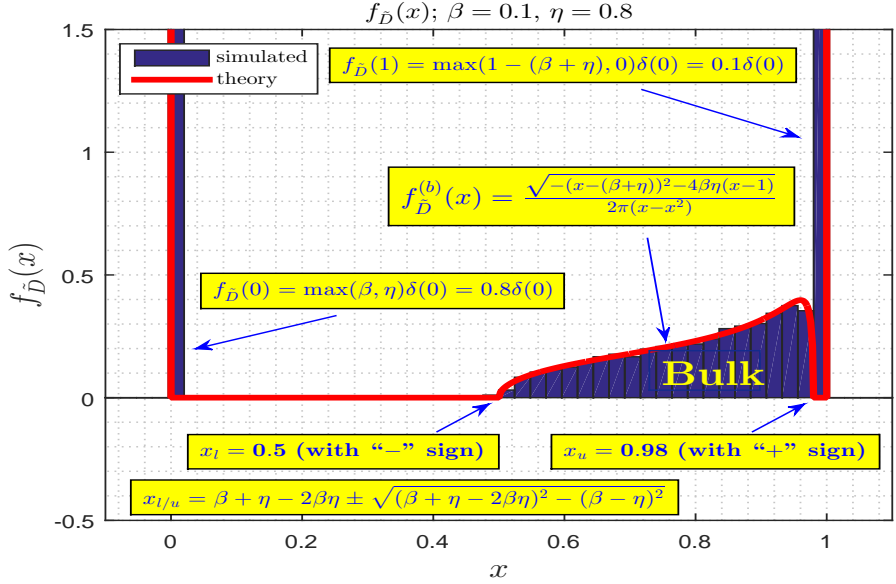,width=13.5cm,height=8cm}}
\caption{$f_{\tilde{D}}(x)$ -- spectral function of $\tilde{D}$; $\beta=0.1$ and $\eta=0.8$}
\label{fig:ftildeDbeta01eta08}
\end{figure}

In Figure \ref{fig:ftildeDbeta02eta09} we show the spectral function obtained based on the above lemma for $\beta=0.2$ and $\eta=0.9$. Due to a remarkable property of the underlying functions the spectrum is identical as in Figure \ref{fig:ftildeDbeta01eta08} apart from the fact that the multiplier of the delta function at zero is increased from $0.8$ to $0.9$ at the expense of removing the delta function at one. We also again observe a very strong agreement between the simulated results and the theoretical predictions. As in Figure \ref{fig:ftildeDbeta01eta08}, Simulation results were again obtained for $n=4000$.

\begin{figure}[htb]
\centering
\centerline{\epsfig{figure=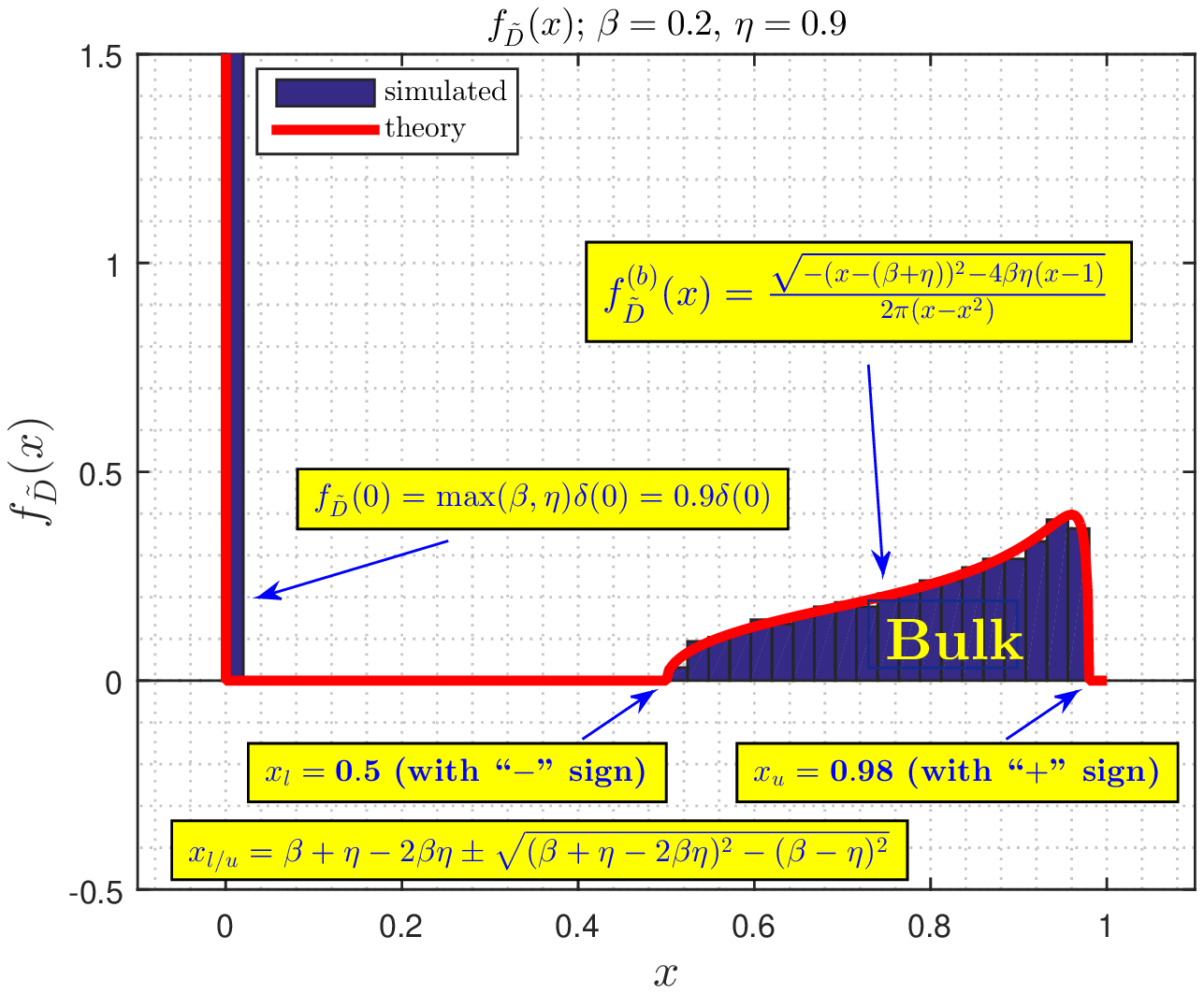,width=13.5cm,height=8cm}}
\caption{$f_{\tilde{D}}(x)$ -- spectral function of $\tilde{D}$; $\beta=0.2$ and $\eta=0.9$}
\label{fig:ftildeDbeta02eta09}
\end{figure}

\subsubsubsection{The spectrum of $\bar{D}$/$D$ -- closed form expressions}
\label{sec:fpteqvspecwcbarD}

We recall on (\ref{eq:typwcanl15}) and (\ref{eq:typwcanl16}) to set
\begin{eqnarray}
\bar{D} & \triangleq &  (I^{(l)})^TU_D^T\bV^{\perp}(\bV^{\perp})^T U_D I^{(l)}
= (\bU_D^{(\perp)})^T\bV^{\perp}(\bV^{\perp})^T \bU_D^{(\perp)}, \label{eq:typwcanl69}
\end{eqnarray}
Comparing (\ref{eq:typwcanl19}) and (\ref{eq:typwcanl69}) we observe that their spectra are modulo scalings basically identical. To be a bit more precise, $\bar{D}$ has all eigenvalues that $\tilde{D}$ has with $\eta n$ zeros less. That basically means that one needs to adjust the multiplier of $\delta(x)$ in $f_{\tilde{D}}(x)$ and to scale everything by $(1-\eta)$. The following lemma summarizes the final results of such a procedure.

\begin{lemma}\label{lemma:typwclemma3}
 Assume the setup of Lemma \ref{lemma:typwclemma2}. Let $\bar{D}$ be as in (\ref{eq:typwcanl69}), i.e.
\begin{eqnarray}
\bar{D} & \triangleq &    (\bU_D^{(\perp)})^T\bV^{\perp}(\bV^{\perp})^T \bU_D^{(\perp)}, \label{eq:typwclemma3eq1}
\end{eqnarray}
and let $x_l$ and $x_u$ be as in (\ref{eq:typwclemma2eq3}). Then the limiting spectral distribution of $\bar{D}$, $f_{\bar{D}}(x)$, is
\begin{equation}
     f_{\bar{D}}(x)     =
              \frac{(\max(\beta,\eta)-\eta)}{1-\eta}\delta(x) + f_{\bar{D}}^{(b)}(x)
              +\frac{\max(1-(\beta+\eta),0)}{1-\eta}\delta(x-1),
     \label{eq:typwclemma3eq2}
\end{equation}
with
\begin{eqnarray}
     f_{\bar{D}}^{(b)}(x)   &  =  &
            \begin{cases}
              \frac{\sqrt{-(x-(\beta+\eta))^2-4\beta\eta(x-1)}}{2\pi(x-x^2)(1-\eta)}, & \mbox{if }  x_l\leq x\leq x_u. \\
              0, & \mbox{otherwise}.
            \end{cases}
     \label{eq:typwclemma3eq3}
\end{eqnarray}
Moreover, let $D$ be as in (\ref{eq:typwcanl14}), i.e.
\begin{eqnarray}
D & \triangleq &   (I^{(l)})^T\bV^{\perp}(\bV^{\perp})^T I^{(l)}. \label{eq:typwclemma3eq4}
\end{eqnarray}
The limiting spectral distribution of $D$, $f_{D}(x)$, is
\begin{equation}
     f_{D}(x)=f_{\bar{D}}(x).
      \label{eq:typwclemma3eq5}
\end{equation}
  \end{lemma}
\begin{proof}
 The part that relates to the spectral distribution of $\bar{D}$ follows by removing $l=\eta n$ zeros from the spectrum of $\tilde{D}$ and appropriately scaling the residual pdf by $(1-\eta)$. The part that relates to the spectral distribution of $D$ follows from (\ref{eq:typwcanl17}) which itself is a consequence of the spectral invariance under unitary multiplications.
\end{proof}

In Figure \ref{fig:fDbeta01eta08} we show the spectral function obtained based on the above lemma for $\beta=0.1$ and $\eta=0.8$. We observe a very strong agreement between the simulated results and the above theoretical predictions. Simulation results were obtained using moderately large $n=4000$.

\begin{figure}[htb]
\centering
\centerline{\epsfig{figure=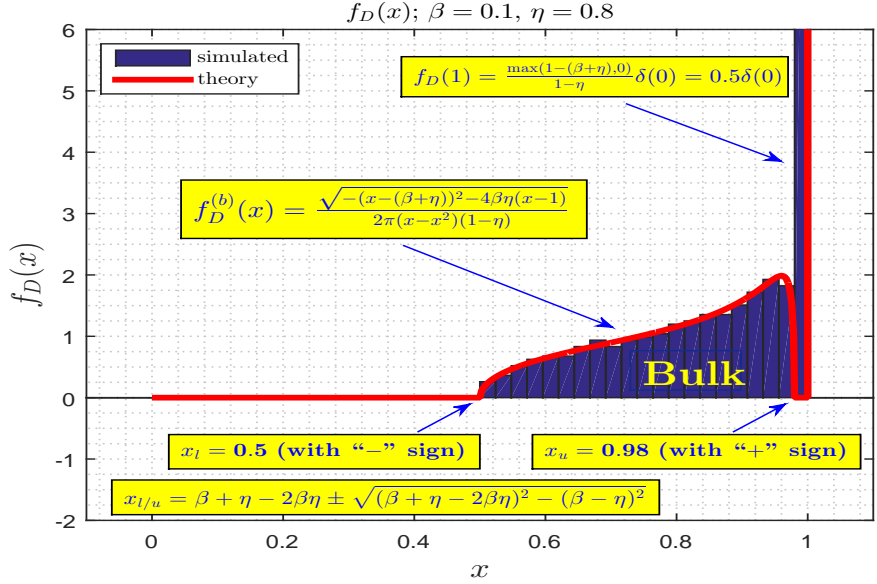,width=13.5cm,height=8cm}}
\caption{$f_{D}(x)$ -- spectral function of $D$; $\beta=0.1$ and $\eta=0.8$}
\label{fig:fDbeta01eta08}
\end{figure}

\subsubsubsection{The spectrum of $Q$ -- closed form expressions}
\label{sec:fpteqvspecwcQ}

We start by recalling on (\ref{eq:typwcanl14})
\begin{eqnarray}
Q=D^{-1}-I . \label{eq:typwcanl70}
\end{eqnarray}
Almost all of the above holds without explicitly assuming $k\leq l$. From this point on such an assumption is needed. Since $k\leq l$ means $\beta\leq \eta$ one has that $D$ has no zeros in its spectrum and can be inverted. Since the portion of the spectrum at one remains the same after the inversion, one then basically has that the spectrum of $Q$ is the same as the spectrum of $D$ modulo the inversion of the bulk of $D$. After a change of variables $y=\frac{1}{x}$ one has for the spectral distribution of the inverted bulk
\begin{eqnarray}
     f_{\bar{D}^{-1}}^{(b)}(y)   &  =  &
            \begin{cases}
              \frac{-\sqrt{-(\frac{1}{y}-(\beta+\eta))^2-4\beta\eta(\frac{1}{y}-1)}}{2\pi(\frac{1}{y}-(\frac{1}{y})^2)(1-\eta)y^2}, & \mbox{if }  \frac{1}{x_u}\leq y\leq \frac{1}{x_l}. \\
              0, & \mbox{otherwise},
            \end{cases}
     \label{eq:typwcanl71}
\end{eqnarray}
which after elementary algebraic transformations becomes
\begin{eqnarray}
     f_{\bar{D}^{-1}}^{(b)}(x)   &  =  &
            \begin{cases}
              \frac{\sqrt{-(1-x(\beta+\eta))^2-4\beta\eta x(1-x)}}{2\pi(1-x)(1-\eta)x}, & \mbox{if }  \frac{1}{x_u}\leq x\leq \frac{1}{x_l}. \\
              0, & \mbox{otherwise},
            \end{cases}
     \label{eq:typwcanl72}
\end{eqnarray}
Noting that $\beta\leq \eta$ implies $\max(\beta,\eta)-\eta=0$, one can combine (\ref{eq:typwcanl72} together with (\ref{eq:typwclemma3eq2}) to obtain
\begin{equation}
     f_{\bar{D}^{-1}}(x)     =
  f_{\bar{D}^{-1}}^{(b)}(x)
              +\frac{\max(1-(\beta+\eta),0)}{1-\eta}\delta(x-1).
     \label{eq:typwcanl73}
\end{equation}
Finally adjusting for a subtracted identity matrix corresponds to subtracting one from any point in the spectrum (or basically to shifting the entire spectral distribution to the left by one). In other words
\begin{equation}
     f_{Q}(x) = f_{\bar{D}^{-1}-I}(x)  =
  f_{\bar{D}^{-1}-I}^{(b)}(x)
              +\frac{\max(1-(\beta+\eta),0)}{1-\eta}\delta(x),
     \label{eq:typwcanl75}
\end{equation}
with
\begin{eqnarray}
    f_{Q}^{(b)}(x) \triangleq f_{\bar{D}^{-1}-I}^{(b)}(x)   &  =  &
            \begin{cases}
              \frac{\sqrt{-(1-(x+1)(\beta+\eta))^2-4\beta\eta x(x+1)}}{2\pi x(x+1)(1-\eta)}, & \mbox{if }  \frac{1}{x_u}-1\leq x\leq \frac{1}{x_l}-1. \\
              0, & \mbox{otherwise},
            \end{cases}
     \label{eq:typwcanl76}
\end{eqnarray}

The following lemma summarizes the key results of this section.

\begin{lemma}\label{lemma:typwclemma4}
 Assume the setup of Lemma \ref{lemma:typwclemma3}. Let $Q$ be as in (\ref{eq:typwcanl14}), i.e.
\begin{eqnarray}
Q\triangleq D^{-1}-I = \lp  (I^{(l)})^T\bV^{\perp}(\bV^{\perp})^T I^{(l)} \rp^{-1} - I, \label{eq:typwclemma4eq1}
\end{eqnarray}
and let $x_l$ and $x_u$ be as in (\ref{eq:typwclemma2eq3}). Then the limiting spectral distribution of $Q$, $f_{Q}(x)$, is
\begin{equation}
     f_{Q}(x)     =
                f_{Q}^{(b)}(x)
              +\frac{\max(1-(\beta+\eta),0)}{1-\eta}\delta(x),
     \label{eq:typwclemma4eq2}
\end{equation}
with
\begin{eqnarray}
    f_{Q}^{(b)}(x)     &  =  &
            \begin{cases}
              \frac{\sqrt{-(1-(x+1)(\beta+\eta))^2-4\beta\eta x(x+1)}}{2\pi x(x+1)(1-\eta)}, & \mbox{if }  \frac{1}{x_u}-1\leq x\leq \frac{1}{x_l}-1, \\
              0, & \mbox{otherwise}.
            \end{cases}
     \label{eq:typwclemma4eq3}
\end{eqnarray}
   \end{lemma}
\begin{proof}
 Follows from the above discussion.
 \end{proof}

In Figure \ref{fig:fQbeta01eta08} we show the spectral function, $f_{Q}(x)$, obtained based on the above lemma for $\beta=0.1$ and $\eta=0.8$. One observes a very strong agreement between what the theory predicts and what the simulations produce. As in all earlier experiments $n=4000$ was used here again.

\begin{figure}[htb]
\centering
\centerline{\epsfig{figure=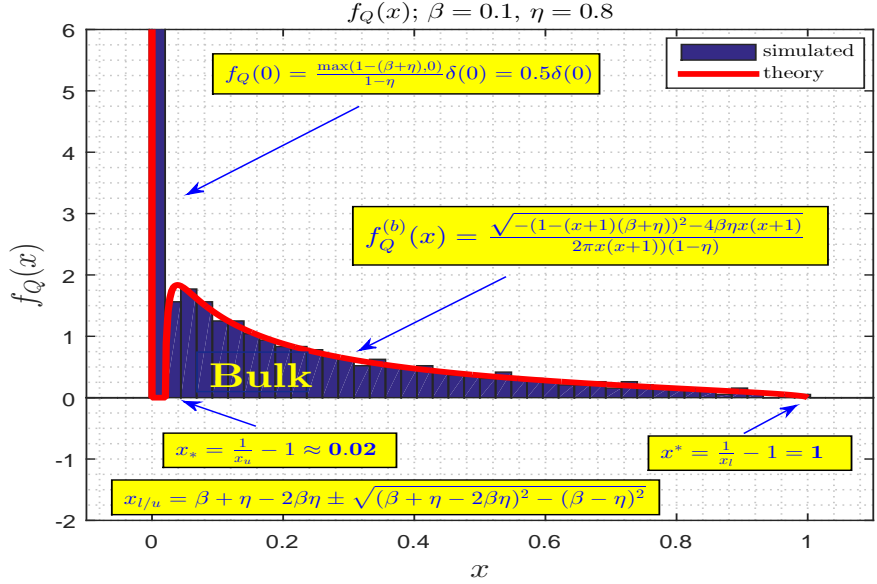,width=13.5cm,height=8cm}}
\caption{$f_{Q}(x)$ -- spectral function of $Q$; $\beta=0.1$ and $\eta=0.8$}
\label{fig:fQbeta01eta08}
\end{figure}

\subsubsection{$\ell_0^*-\ell_1^*$-equivalence via the spectral limit}
\label{sec:fpteqvsplim}

From Corollary \ref{cor:cinfcor3}, (\ref{eq:cinfcor3eq1}), (\ref{eq:cinfcor3eq2}), and (\ref{eq:typwcanl12}) one has in the \emph{\textbf{worst case}}
\begin{equation}\label{eq:typwcanl77}
  \ell_0^*-\ell_1^*-\mbox{equivalence} \quad\Longleftrightarrow \quad \lambda_{max}(Q)\leq 1.
\end{equation}
From Lemma \ref{lemma:typwclemma4} we have
\begin{equation}\label{eq:typwcanl78}
 \lambda_{max}(Q)=\frac{1}{x_l}-1,
\end{equation}
with $x_l$ as in (\ref{eq:typwclemma2eq3}). Moreover, from (\ref{eq:typwcanl77}) and (\ref{eq:typwcanl78}), one arrives at the following necessary and sufficient condition to achieve the $\ell_0^*-\ell_1^*$-equivalence
\begin{equation}\label{eq:typwcanl79}
  \ell_0^*-\ell_1^*-\mbox{equivalence} \quad\Longleftrightarrow \quad \frac{1}{x_l}-1\leq 1,
 \end{equation}
 or
\begin{equation}\label{eq:typwcanl80}
  \ell_0^*-\ell_1^*-\mbox{equivalence} \quad\Longleftrightarrow \quad \frac{1}{2}\leq x_l.
 \end{equation}
 Recalling on (\ref{eq:typwclemma2eq3})
\begin{eqnarray}
  x_l & \triangleq & \beta+\eta-2\beta\eta  -  \sqrt{(\beta+\eta-2\beta\eta)^2-(\beta-\eta)^2},
   \label{eq:typwcanl81}
\end{eqnarray}
 and combining further with (\ref{eq:typwcanl80}) we have
\begin{eqnarray}
\begin{array}{rrcl}
 & \frac{1}{2} & \leq & x_l \\
\Longleftrightarrow & \frac{1}{2} & = & \beta+\eta-2\beta\eta  -  \sqrt{(\beta+\eta-2\beta\eta)^2-(\beta-\eta)^2} \\
\Longleftrightarrow &  (\beta+\eta-2\beta\eta)^2-(\beta-\eta)^2 & \leq & \lp\beta+\eta-2\beta\eta  - \frac{1}{2}\rp^2 \\
\Longleftrightarrow &   (\beta+\eta-2\beta\eta)^2-(\beta-\eta)^2 & \leq & \lp\beta+\eta-2\beta\eta \rp^2 - \lp\beta+\eta-2\beta\eta \rp + \frac{1}{4} \\
\Longleftrightarrow &    -(\beta-\eta)^2  & \leq & - \lp\beta+\eta-2\beta\eta \rp + \frac{1}{4} \\
\Longleftrightarrow &  0 & \leq & \beta^2+\eta^2 - \lp\beta+\eta \rp + \frac{1}{4}  \\
\Longleftrightarrow &  \eta-\eta^2  & \leq &  \lp \frac{1}{2} - \beta \rp^2\\
\Longleftrightarrow &   \beta & \leq & \frac{1}{2} - \sqrt{\eta-\eta^2}.
\end{array}
  \label{eq:typwcanl82}
\end{eqnarray}
From (\ref{eq:typwcanl80}) and (\ref{eq:typwcanl82}) we finally have
\begin{equation}\label{eq:typwcanl83}
  \ell_0^*-\ell_1^*-\mbox{equivalence} \quad\Longleftrightarrow \quad \beta  \leq  \frac{1}{2} - \sqrt{\eta-\eta^2}.
 \end{equation}

We are now in position to formalize the key causal inference results that establish the so-called phase-transition phenomenon as well as its a precise \emph{worst case} location in a \emph{typical} statistical scenario.

\begin{theorem}(\textbf{\bl{$\ell_1^*$ -- phase transition -- C-inf (typical \prp{\underline{worst case}})}})
  Consider a rank-$k$ matrix  $X_{sol}=X\in\mR^{n\times n}$ with the Haar distributed (\emph{not necessarily independent}) bases of its orthogonal row and column spans $\bU^{\perp}\in\mR^{n\times (n-k)}$ and $\bV^{\perp}\in\mR^{n\times (n-k)}$ ($X_{sol}^T\bU^{\perp}=X_{sol}\bV^{\perp}=\0_{n\times (n-k)}$). Let $M\triangleq M^{(l)}\in\mR^{n\times n}$ be as defined in (\ref{eq:cinf2}) or (\ref{eq:cinfanl2a}). Assume a large $n$ linear regime with $\beta\triangleq\lim_{n\rightarrow\infty}\frac{k}{n}$ and $\eta\triangleq\lim_{n\rightarrow\infty}\frac{l}{n}$ and let $\beta_{wc}$ and $\eta$ satisfy the  following
\begin{center}
 \begin{tcolorbox}[beamer,title=\textbf{C-inf $\ell_1^*$ \yellow{worst case} phase transition (PT)} characterization,lower separated=false, colback=yellow!95!green!40!white,
colframe=red!75!blue!60!black,fonttitle=\bfseries,width=5in]
  \begin{equation}\label{eq:typwcthm1eq1}
    \xi_{\eta}^{(wc)}(\beta) \triangleq \beta-\frac{1}{2}+\sqrt{\eta-\eta^2}=0.
  \end{equation}
 \end{tcolorbox}
\end{center}
 \noindent \textbf{If and only if} $\beta\leq \beta_{wc}$
    \begin{equation}\label{eq:typwcthm1eq2}
   \lim_{n\rightarrow\infty} \mP(\ell_0^*\Longleftrightarrow \ell_1^*)=\ \lim_{n\rightarrow\infty} \mP(\mathbf{RMSE}=0)=1,
  \end{equation}
and the solutions of (\ref{eq:genmcl0posmmt}) and (\ref{eq:genmcl1posmmt}) coincide with overwhelming probability.
  \label{thm:typwcthm1}
\end{theorem}
\begin{proof}
The ``if" part follows through a combination of Theorem \ref{thm:cinfthm1}, Corollary \ref{cor:cinfcor3}, Lemma \ref{lemma:typwclemma4}, and the above discussion from  (\ref{eq:typwcanl77}) to (\ref{eq:typwcanl83}). The ``only if" part additionally assumes $\bU^{\perp}=\bV^{\perp}$ and then relying on (\ref{eq:typwcanl12}) and (\ref{eq:typwcanl120a}) ensures that the results are in the worst case achievable.
\end{proof}

The results obtained based on the above theorem are shown in Figure \ref{fig:cinfspectypwcPTbetaeta}. As can be seen from the figure, the phase transition curve splits the entire $(\beta,\eta)$ region into two subregions. The first of the subregions is below (or to the right of) the curve and in that region the $\ell_0^*-\ell_1^*$-equivalence phenomenon occurs. This means that one can recover $X_{sol}$ masked by $M$ as in (\ref{eq:genmcl0posmmt}) via the $\ell_1^*$ heuristic from (\ref{eq:genmcl1posmmt}) with the residual mean square error ($\mathbf{RMSE}$) equal to zero. In other words, for the system parameters $(\beta,\eta)$ that belong to the subregion below the curve one has a perfect recovery with $X_{sol}$ and $\hat{X}$ (the respective solutions of (\ref{eq:genmcl0posmmt}) and (\ref{eq:genmcl1posmmt})) being equal to each other and consequently with $\mathbf{RMSE}=\|\vecw(\hat{X})-\vecw(X_{sol})\|_2=0$. On the other hand, in the subregion above the curve, the $\ell_1^*$ heuristic fails and one can even find an $X_{sol}$ for which $\mathbf{RMSE}\rightarrow\infty$.

\begin{figure}[htb]
\centering
\centerline{\epsfig{figure=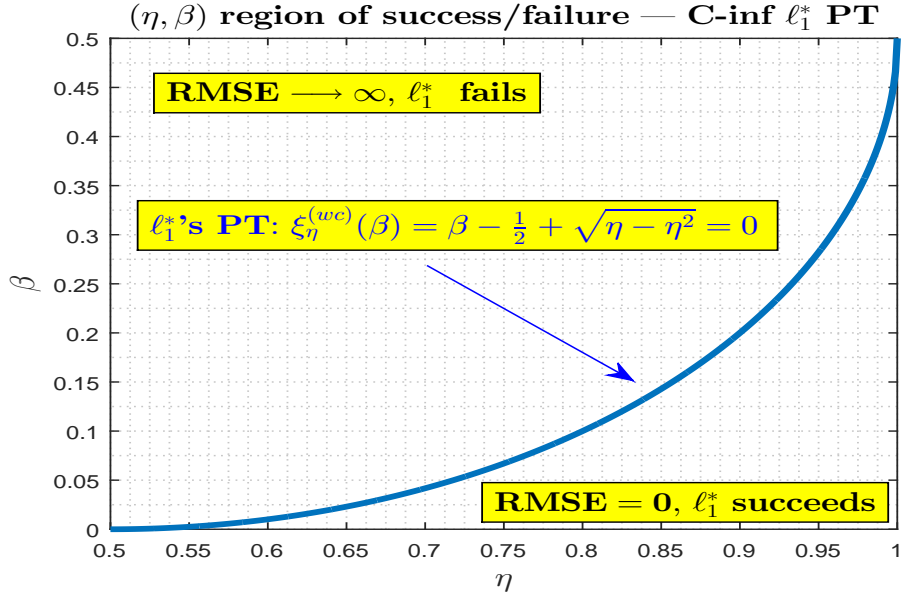,width=13.5cm,height=8cm}}
\caption{Causal inference (\bl{\textbf{C-inf}}) -- typical \emph{\textbf{worst case}} $\ell_1^*$ phase transition}
\label{fig:cinfspectypwcPTbetaeta}
\end{figure}

We make a few remarks that relate to some general features of matrix completion, and their differences with respect with standard compressed sensing. Even in the so-called random-masking scenario (where elements of $M$ take values $0/1$ with probability one half), one may have troubles recovering low-rank matrices. For example, in such a scenario, it is statistically unlikely that one can guarantee universal recovery even of rank-$1$ matrices. To see this, one can choose rank-$1$ $X_{sol}$ with one in the upper left corner and all other elements equal to zero. Then such a matrix will be recoverable from $M\circ X_{sol}$ only if the element in the upper left corner of $M$ is one. Since that happens with probability 1/2, one can not have a reliable recovery with probability going to one as $n\rightarrow \infty$. It is the discreteness in the process of acquiring observations $Y$ that enables scenarios like this, and makes the matrix completion (MC) substantially different from its compressed sensing vector analogue. In that light, it is somewhat surprising that any form of the phase-transition can be established in the linear regime. The key behind the success of the above machinery is the focus on the \textbf{\emph{typical}} worst case and the causal inference internal structure. In the vector compressed sensing setup, the above described scenario cannot happen and one consequently does not need to resort to the typicality and instead can formulate more generic phase transition concepts (more on the \emph{non-typical} compressed sensing approaches can be found in, e.g. \cite{DonohoPol,StojnicCSetam09,StojnicICASSP10var} and on the corresponding \emph{typical} ones in, e.g. \cite{BayMon10,DonMalMon09}).

\begin{corollary}(\textbf{\bl{$\ell_1^*$ -- phase transition -- C-inf (typical \prp{\underline{worst case}}; standard $(\alpha,\beta)$ representation)}})
   Assume the setup of Theorem \ref{thm:typwcthm1}. Let $m$ be the total number of ones in matrix $M$ and let $\alpha\triangleq\lim_{n\rightarrow\infty}\frac{m}{n^2}$. Let $\beta$ and $\alpha_w$ satisfy the
\begin{center}
 \begin{tcolorbox}[beamer,title=\textbf{C-inf $\ell_1^*$ \yellow{worst case} PT (standard $(\alpha,\beta)$ representation)},lower separated=false, colback=yellow!95!green!40!white,
colframe=red!75!blue!60!black,coltext=black,fonttitle=\bfseries,width=5in]
  \begin{equation}\label{eq:typwccor1eq1}
    \xi_{\beta}^{(wc,s)}(\alpha) \triangleq \beta-\frac{1}{2}+\sqrt{\sqrt{1-\alpha}-1+\alpha}=0.
  \end{equation}
 \end{tcolorbox}
\end{center}
\noindent \textbf{If and only if} $\alpha\geq \alpha_{w}$
    \begin{equation}\label{eq:typwcthm1eq2}
   \lim_{n\rightarrow\infty} \mP(\ell_0^*\Longleftrightarrow \ell_1^*)=\ \lim_{n\rightarrow\infty} \mP(\mathbf{RMSE}=0)=1,
  \end{equation}
and the solutions of (\ref{eq:genmcl0posmmt}) and (\ref{eq:genmcl1posmmt}) coincide with overwhelming probability.
  \label{cor:typwccor1}
\end{corollary}

\begin{proof}
  Follows immediately from Theorem \ref{thm:typwcthm1} after observing that $m=n^2-(n-l)^2$ and consequently $\alpha=1-(1-\eta)^2$.
\end{proof}

 Figure \ref{fig:cinfspectypwcPTbetaeta} shows the results obtained based on the above corollary in the standard $(\alpha,\beta)$ region format. As earlier, in the subregion to the right of (or below) the curve $\mathbf{RMSE}=\|\vecw(\hat{X})-\vecw(X_{sol})\|_2=0$. In the subregion to the left (or above) the curve, the $\ell_1^*$ heuristic generally fails and $\mathbf{RMSE}\rightarrow\infty$ can even be achieved.

\begin{figure}[htb]
\centering
\centerline{\epsfig{figure=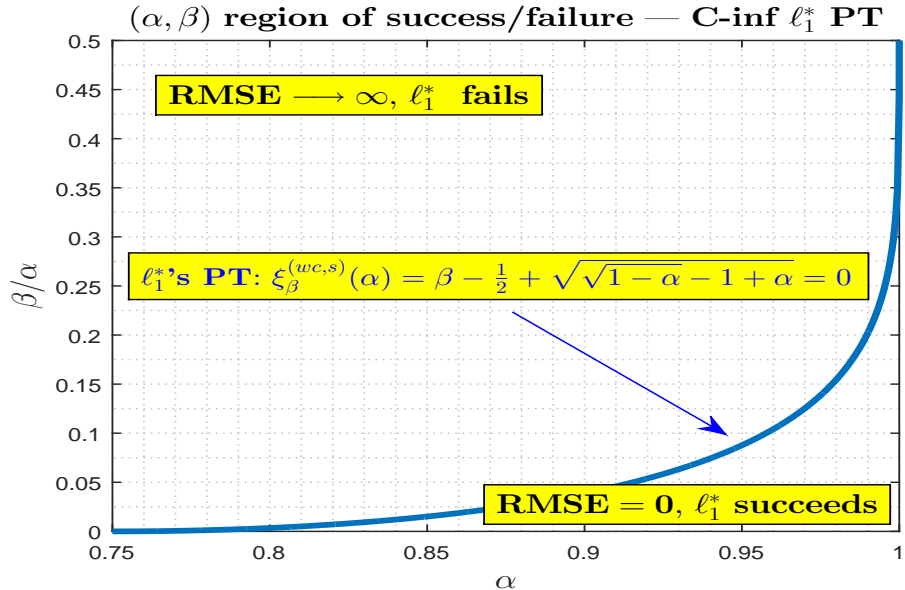,width=13.5cm,height=8cm}}
\caption{Causal inference (\bl{\textbf{C-inf}}) -- typical \emph{\textbf{worst case}} $\ell_1^*$ phase transition ($(\alpha,\beta)$ region)}
\label{fig:cinfspectypwcPTbetaeta}
\end{figure}

\subsection{Numerical results}
\label{sec:cpmcnumres}

We conducted a set of numerical experiments to complement the above theoretical findings and see how well the entire theoretical machinery characterizes the utilization of the $\ell_1^*$-minimization heuristic in the causal inference problems. Figure \ref{fig:numrestypacbetaetaptsim} shows the performance obtained through the numerical experiments as well as the corresponding theoretical \emph{worst case} predictions discussed above. We clearly observe the existence of the phase transition and a solid agreement between the theoretical predictions and the results obtained through the numerical experiments.

\begin{figure}[htb]
\centering
\centerline{\epsfig{figure=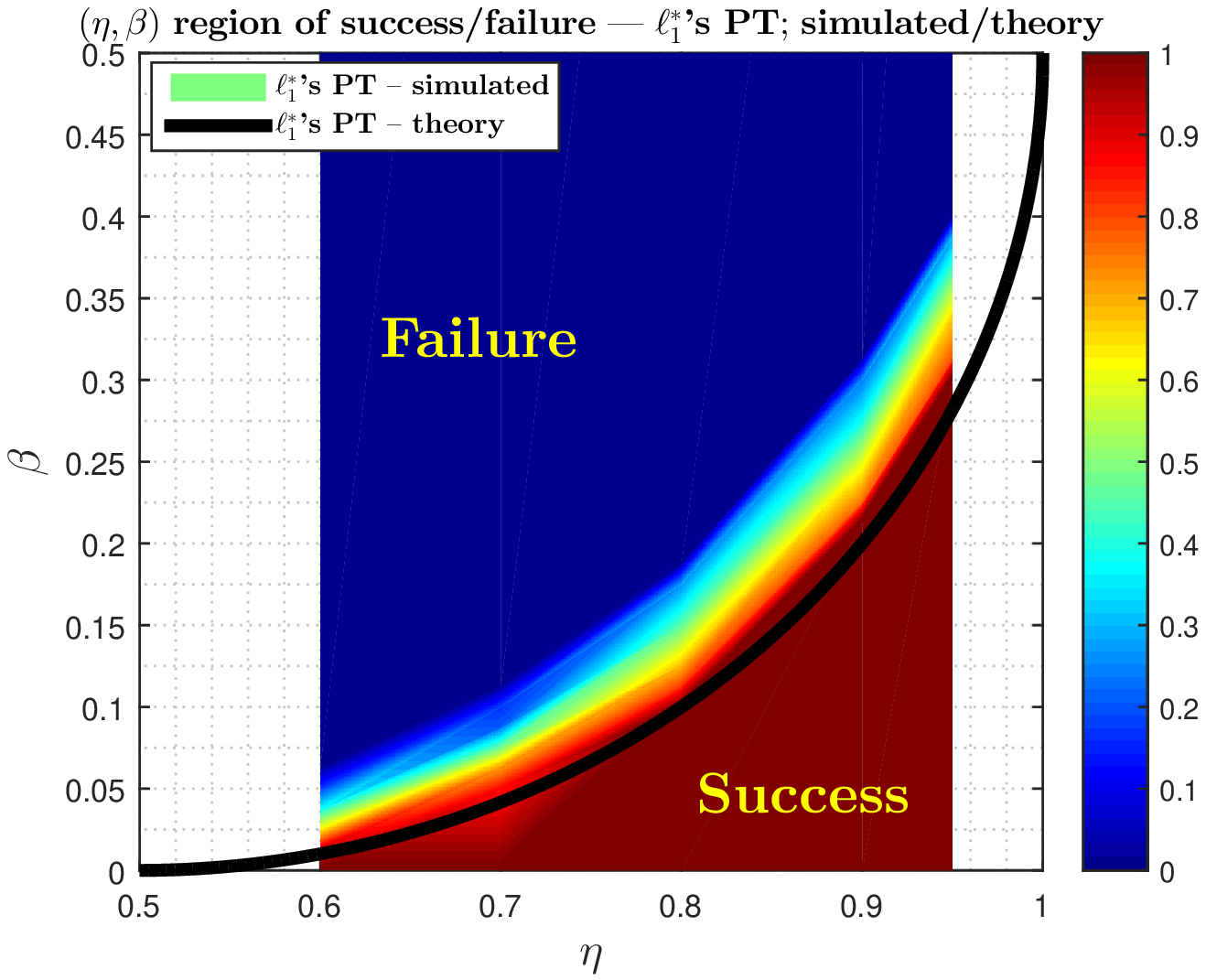,width=13.5cm,height=8cm}}
\caption{C-inf $\ell_1^{*}$'s phase transition (PT)}
\label{fig:numrestypacbetaetaptsim}
\end{figure}

We should also add that we conducted the numerical experiments for fairly small matrix sizes. On the other hand, the theoretical predictions assume large $n$ (basically an $n\rightarrow\infty$). In particular, we chose $n=80$ and $\eta$ in the range $[0.6,0.95]$. Such a choice shows that even though the theory is predicated on the large $n$ assumption, its conclusions may be applicable for smaller values of $n$ as well. Viewed a bit alternatively, the large $n$ regime, needed for the theory to properly operate, practically may start ro kick in already for not necessarily super large values of $n$. This also means that the presented theory may actually be of a practical use as well. We do, however, mention that for larger values of $n$ an even better fit between the theoretical and simulated results might be expected.

Finally, we should emphasize that in the numerical experiments here (as well as in a significant portion of the paper) we considered the so-called \emph{typical} behavior. Also, it should be mentioned that we followed into the footsteps of our theoretical analyses from the previous sections and presented the simulations results for the square matrices. With a little bit of extra effort, all of our theoretical considerations can be repeated in the non-square scenarios as well. Since the writings would be a bit more involved we found it useful to preserve the clarity of the presentation at the expense of rather trivial generalizations. Finally, in all numerical experiments that we present we chose the unknown matrices with the singular values equal to one. While a thorough discussion regarding this choice goes well beyond the scope of this paper, we just briefly recall that these types of structures typically serve as the worst case examples in establishing the reversal $\ell_0-\ell_1$-equivalence conditions. In other words, they are usually the examples that make many of our key results/theorems hold as equivalences rather than just as implications. We also conducted numerical experiments where singular values were randomly chosen with results being identical to the ones shown in Figure \ref{fig:numrestypacbetaetaptsim}
or better.

\section{Conclusion}
\label{sec:conc}

In this paper, we have explored the \prp{\textbf{\emph{Causal inference (C-inf) $\leftrightarrow$  \textbf{low-rank recovery (LRR)}}}} connection. We have analyzed how the best known convex type of heuristic, called $\ell_1^*$-minimization, fairs when used for solving the C-inf. We have shown both theoretically and numerically that in a \emph{typical} statistical context, causal inference exhibits the so-called phase transition (PT) phenomenon. This ensures that for certain range of system parameters $\ell_1^*$ succeeds in recovering the unobserved potential outcomes,
and outside such a range it fails. Moreover, we have obtained the \bl{\textbf{\emph{exact explicit}}} functional characterization for the location of the worst case phase transition (PT) curve.  As a byproduct of our analysis, we have obtained (somewhat surprisingly) that the underlying functional characterization admits a fairly simple form, which elegantly pins down the relation between the low rankness of the target C-inf matrix and the time when the treatment is applied. We also emphasize that, while establishing the mathematical methodology to provide the C-inf PT characterizations, we uncovered a rather interesting connection between \bl{\textbf{\emph{C-inf}}} via LRR and the free probability theory (\bl{\textbf{\emph{FPT}}}) from modern spectral theory of random matrices. After establishing the connection between C-inf and the compressed sensing via the Random duality theory (\bl{\textbf{\emph{RDT}}}), we have proceeded to recognize the role that FPT plays in the overall mosaic that ultimately enables handling the C-inf problem.

On a path to achieving complete handling of the causal inference we have created quite a few mathematical results that are of independent interest. To ensure a completeness of the overall treatment, we for all of them also ran the corresponding numerical experiments and again observed a rather overwhelming agreement between the theoretical predictions and simulations.

While there exists different approaches that can be explored to attack the problems considered here (with some of them being even conceptually simpler), our choice is partly motivated by the ability to handle more complicated problems in future extensions of this work. Being this the introductory paper, we stopped short of showcasing how the introduced theory fairs when applied to many such specific, more complex instances (the simplest among them would be the noisy and approximately low-rank corresponding ones). In our companion papers, we will establish results along these directions, which all rely on the mathematical framework presented in this paper.

\begin{singlespace}
\bibliographystyle{plain}
\bibliography{cinfspecidealRefs1}

\begin{thebibliography}{10}

\bibitem{Abadsynth19}
A.~Abadie.
\newblock Using synthetic controls: Feasibility, data requirements, and
  methodological aspects.
\newblock {\em Journal of Economic Literature}, 2019.

\bibitem{ADHsynth10}
A.~Abadie, A.~Diamond, and J.~Hainmueller.
\newblock Synthetic control methods for comparative case studies: Estimating
  the effect of california’s tobacco control program.
\newblock {\em Journal of the American Statistical Association}, 105:493--505,
  2010.

\bibitem{Agarwal2021}
A.~Agarwal, M.~Dahel, D.~Shah, and D.~Shen.
\newblock Causal matrix completion.
\newblock available online at arxiv.

\bibitem{ABDIK21}
S.~Athey, M.~Bayati, N.~Doudchenko, G.~Imbens, and K.~Khosravi.
\newblock Matrix completion methods for causal panel data models.
\newblock {\em Journal of the American Statistical Association},
  116(536):1716--1730, 2021.

\bibitem{AthImb18}
S.~Athey and G.~W. Imbens.
\newblock Design-based analysis in differencein- differences settings with
  staggered adoption.
\newblock {\em Technical Report, National Bureau of Economic Research}, 2018.

\bibitem{AthSte02}
S.~Athey and S.~Stren.
\newblock The impact of information technology on emergency health care
  outcomes.
\newblock {\em The RAND Journal of Economics}, 33:399--432, 2002.

\bibitem{BayMon10}
M.~Bayati and A.~Montanari.
\newblock The dynamics of message passing on dense graphs, with applications to
  compressed sensing.
\newblock available online at \bl{\url{http://arxiv.org/abs/1001.3448}}.

\bibitem{CRT}
E.~Candes, J.~Romberg, and T.~Tao.
\newblock Robust uncertainty principles: exact signal reconstruction from
  highly incomplete frequency information.
\newblock {\em IEEE Trans. on Information Theory}, 52:489--509, December 2006.

\bibitem{CR09matcomp}
E.~J. Candes and B.~Recht.
\newblock Exact matrix completion via convex optimization.
\newblock {\em Foundations of Computational Mathematics}, (9):717, 2009.

\bibitem{CT10matcomp}
E.~J. Candes and T.~Tao.
\newblock The power of convex relaxation: Near-optimalmatrix completion.
\newblock {\em IEEE Transactions on Information Theory}, 56:2053--2080, 2010.

\bibitem{CPmatcomp10}
E.~J. Candès and Y.~Plan.
\newblock Matrix completion with noise.
\newblock {\em Proceedings of the IEEE}, 98:925--936, 2010.

\bibitem{Cinfidealpc22}
A.~Capponi and M.~Stojnic.
\newblock Causal inference ({C}-inf) --- asymmetric scenario of typical phase
  transitions.
\newblock available online at arxiv.

\bibitem{DonohoPol}
D.~Donoho.
\newblock High-dimensional centrally symmetric polytopes with neighborlines
  proportional to dimension.
\newblock {\em Disc. Comput. Geometry}, 35(4):617--652, 2006.

\bibitem{DonMalMon09}
D.~Donoho, A.~Maleki, and A.~Montanari.
\newblock Message-passing algorithms for compressed sensing.
\newblock {\em Proc. National Academy of Sciences}, 106(45):18914--18919, Nov.
  2009.

\bibitem{DonohoSigned}
D.~Donoho and J.~Tanner.
\newblock Sparse nonnegative solutions of underdetermined linear equations by
  linear programming.
\newblock {\em Proc. National Academy of Sciences}, 102(27):9446--9451, 2005.

\bibitem{DoudImb16}
N.~Doudchenko and G.~W. Imbens.
\newblock Balancing, regression, difference-in-differences and synthetic
  control methods: A synthesis.
\newblock {\em Technical Report, National Bureau of Economic Research}, 2016.

\bibitem{Haag97}
U.~Haagerup.
\newblock On {V}oiculescu{’}s {R}- and {S}-transforms for free non-commuting
  random variables.
\newblock {\em In: Free Probability Theory, Fields Institute Communications},
  12:127--148, 1997.

\bibitem{HerRob10}
M.~A. Hernan and J.~M. Robins.
\newblock {\em Causal Inference}.
\newblock Boca Raton, FL: CRC Press, 2010.

\bibitem{ImbRub15}
G.~W. Imbens and D.~B. Rubin.
\newblock {\em Causal Inference in Statistics, Social, and Biomedical
  Sciences}.
\newblock New York: Cambridge University Press, 2015.

\bibitem{KallusNIPS}
N.~Kallus, X.~Mao, and M.~Udell.
\newblock Causal inference with noisy and missing covariates via matrix
  factorization.
\newblock {\em Proceedings of the 32rd International Conference on Neural
  Information Processing}, pages 6921–--6932, December 2018.
\newblock available online.

\bibitem{KMO10matcomp}
R.~H. Keshavan, A.~Montanari, and S.~Oh.
\newblock Matrix completion from a few entries.
\newblock {\em IEEE Transactions on Information Theory}, 56:2980--2998, 2010.

\bibitem{KMO10matcomp1}
R.~H. Keshavan, A.~Montanari, and S.~Oh.
\newblock Matrix completion from noisy entries.
\newblock {\em Journal of Machine Learning Research}, 11:2057--2078, 2010.

\bibitem{Klopp14matcomp}
O.~Klopp.
\newblock Noisy low-rank matrix completion with general sampling distribution.
\newblock {\em Bernoulli}, 20:282--303, 2014.

\bibitem{KLT11matcomp}
V.~Koltchinskii, K.~Lounici, and A.~B. Tsybakov.
\newblock Nuclear-norm penalization and optimal rates for noisy low-rank matrix
  completion.
\newblock {\em The Annals of Statistics}, 39:2302--2329, 2011.

\bibitem{MHT10}
R.~Mazumder, T.~Hastie, and R.~Tibshirani.
\newblock Spectral regularization algorithms for learning large incomplete
  matrices.
\newblock {\em Journal of Machine Learning Research}, 11:2287--2322, 2010.

\bibitem{NW11matcomp}
S.~N. Negahban and M.~J. Wainwright.
\newblock Estimation of (near) low-rank matriceswith noise and high-dimensional
  scaling.
\newblock {\em The Annals of Statistics}, 39:1069--1097, 2011.

\bibitem{NW12matcomp}
S.~N. Negahban and M.~J. Wainwright.
\newblock Restricted strong convexity and weighted matrix completion: Optimal
  bounds with noise.
\newblock {\em Journal of Machine Learning Research}, 13:1685--1697, 2012.

\bibitem{NicaSpeich06}
A.~Nica and R.~Speicher.
\newblock {\em London Mathematical Society Lecture Note Series, vol. 335}.
\newblock Cambridge University Press, 2006.

\bibitem{OH10}
S.~Oymak and B.~Hassibi.
\newblock New null space results and recovery thresholds for matrix rank
  minimization.
\newblock Nov 2010.
\newblock available online at \bl{\url{http://arxiv.org/abs/1011.6326}}.

\bibitem{PearlSurv09}
J.~Pearl.
\newblock Causal inference in statistics: An overview.
\newblock {\em Statistics Surveys}, 3:96--146, 2009.

\bibitem{PearlCausBook09}
J.~Pearl.
\newblock {\em Causality: Models, Reasoning, and Inference}.
\newblock 2nd. Cambridge University Press, New York, 2009.

\bibitem{PearlBar19}
J.~Pearl and E.~Bareinboim.
\newblock A note on {‘}generalizability of study results{’}.
\newblock {\em Journal of Epidemiology}, 30:186--188, 2019.

\bibitem{PearlSMack18}
J.~Pearl and D.~Mackenzie.
\newblock {\em The Book of Why}.
\newblock Basic Books, New York, 2018.

\bibitem{Rechtmatcomp11}
B.~Recht.
\newblock A simpler approach to matrix completion.
\newblock {\em Journal of Machine Learning Research}, 12:3413--3430, 2011.

\bibitem{RFPrank}
B.~Recht, M.~Fazel, and P.~A. Parrilo.
\newblock Guaranteed minimum-rank solution of linear matrix equations via
  nuclear norm minimization.
\newblock 2007.
\newblock available online at \bl{\url{http://www.dsp.ece.rice.edu/cs/}}.

\bibitem{RT11matcomp}
A.~Rohde and A.~B. Tsybakov.
\newblock Estimation of high-dimensional low-rank matrices.
\newblock {\em The Annals of Statistics}, 39:887--930, 2011.

\bibitem{RoseRub83}
P.~R. Rosenbaum and D.~B. Rubin.
\newblock Thecentral role of the propensity score in observational studies for
  causal effects.
\newblock {\em Biometrika}, 70, 1983.

\bibitem{Rub06}
D.~B. Rubin.
\newblock {\em Matched Sampling for Causal Effects}.
\newblock Cambridge: Cambridge University Press, 2006.

\bibitem{ShaTou19}
A.~Shaikh and P.~Toulis.
\newblock Randomization tests in observational studies with staggered adoption
  of treatment.
\newblock {\em University of Chicago, Becker Friedman Institute for Economics
  Working Paper}, (144), 2019.

\bibitem{Speich14}
R.~Speicher.
\newblock Free probability and random matrices.
\newblock {\em In: Proc. ICM}, III:477--501, 2014.

\bibitem{SAT05}
N.~Srebro, N.~Alon, and T.~S. Jaakkola.
\newblock Generalization error bounds for collaborative prediction with
  low-rank matrices.
\newblock {\em in Advances in Neural Information Processing Systems},
  17:1321--1328, 2005.
\newblock eds. L. K. Saul, Y.Weiss, and L. Bottou.

\bibitem{StojnicGenLasso10}
M.~Stojnic.
\newblock A framework for perfromance characterization of \emph{LASSO}
  algortihms.
\newblock available online at \bl{\url{http://arxiv.org/abs/1303.7291}}.

\bibitem{StojnicRegRndDlt10}
M.~Stojnic.
\newblock Regularly random duality.
\newblock available online at \bl{\url{http://arxiv.org/abs/1303.7295}}.

\bibitem{StojnicUpper10}
M.~Stojnic.
\newblock Upper-bounding $\ell_1$-optimization weak thresholds.
\newblock available online at \bl{\url{http://arxiv.org/abs/1303.7289}}.

\bibitem{StojnicCSetam09}
M.~Stojnic.
\newblock Various thresholds for $\ell_1$-optimization in compressed sensing.
\newblock available online at \bl{\url{http://arxiv.org/abs/0907.3666}}.

\bibitem{StojnicICASSP09}
M.~Stojnic.
\newblock A simple performance analysis of $\ell_1$-optimization in compressed
  sensing.
\newblock {\em ICASSP, International Conference on Acoustics, Signal and Speech
  Processing}, pages 3021--3024, 15-19 April 2009.
\newblock Taipei, Taiwan.

\bibitem{StojnicICASSP10block}
M.~Stojnic.
\newblock Block-length dependent thresholds for $\ell_2/\ell_1$-optimization in
  block-sparse compressed sensing.
\newblock {\em ICASSP, IEEE International Conference on Acoustics, Signal and
  Speech Processing}, pages 3918--3921, 14-19 March 2010.
\newblock Dallas, TX.

\bibitem{StojnicICASSP10var}
M.~Stojnic.
\newblock $\ell_1$ optimization and its various thresholds in compressed
  sensing.
\newblock {\em ICASSP, IEEE International Conference on Acoustics, Signal and
  Speech Processing}, pages 3910--3913, 14-19 March 2010.
\newblock Dallas, TX.

\bibitem{StojnicJSTSP09}
M.~Stojnic.
\newblock $\ell_2/\ell_1$-optimization in block-sparse compressed sensing and
  its strong thresholds.
\newblock {\em IEEE Journal of Selected Topics in Signal Processing},
  4(2):350--357, 2010.

\bibitem{StojnicISIT2010binary}
M.~Stojnic.
\newblock Recovery thresholds for $\ell_1$ optimization in binary compressed
  sensing.
\newblock {\em ISIT, IEEE International Symposium on Information Theory}, pages
  1593 -- 1597, 13-18 June 2010.
\newblock Austin, TX.

\bibitem{StojnicICASSP10knownsupp}
M.~Stojnic.
\newblock Towards improving $\ell_1$ optimization in compressed sensing.
\newblock {\em ICASSP, IEEE International Conference on Acoustics, Signal and
  Speech Processing}, pages 3938--3941, 14-19 March 2010.
\newblock Dallas, TX.

\bibitem{SPH}
M.~Stojnic, F.~Parvaresh, and B.~Hassibi.
\newblock On the reconstruction of block-sparse signals with an optimal number
  of measurements.
\newblock {\em IEEE Trans. on Signal Processing}, 57(8):3075--3085, August
  2009.

\bibitem{TulVer04}
A.~M. Tulino and S.~Verdú.
\newblock {\em Random Matrix Theory and Wireless Communications. Foundations
  and Trends® in Communications and Information Theory: Vol. 1}.
\newblock Now Publishers, Hanover, 2004.

\bibitem{Voic86}
D.~Voiculescu.
\newblock Addition of certain non-commuting random variables.
\newblock {\em J. Funct. Anal.}, 66(3):323--346, 1986.

\bibitem{Voic87}
D.~Voiculescu.
\newblock Multiplication of certain noncommuting random variables.
\newblock {\em J. Operator Theory}, 18:2223--2235, 1987.

\bibitem{Voic91}
D.~Voiculescu.
\newblock Limit laws for random matrices and free products.
\newblock {\em Invent. Math.}, 104(1):201--220, 1991.

\bibitem{XPlatfac10}
R.~Xiong and M.~Pelger.
\newblock Large dimensional latent factor modeling with missing observations
  and applications to causal inference.
\newblock 2018.
\newblock Electronic copy available at:
  \bl{\url{https://ssrn.com/abstract=3465357}}.

\bibitem{Xucinf17}
Y.~Xu.
\newblock Generalized synthetic control method: Causal inference with
  interactive fixed effects models.
\newblock {\em Political Analysis}, 25:57--76, 2017.

\end{thebibliography}
\end{singlespace}

\appendix

\section{Proof of Theorem \ref{thm:cinfthm1}}
\label{sec:appA}

As mentioned earlier, the proof of Theorem \ref{thm:cinfthm1} is conceptually identical to the corresponding proof when matrix $X$ is symmetric. A detailed proof for the symmetric matrices is given below. Before being able to present the proof we need a couple of technical lemmas.

\begin{lemma}
Let $C=C^T\in\mR^{n\times n}$. Also let all eigenvalues of $C$ belong to the interval $[-1,1]$. Finally, let the first $k$ entries on the main diagonal, $C_{i,i},1\leq i\leq k$, be larger than or equal to 1. Then the upper $k\times k$  left block of $C$, $C_{1:k,1:k}$, is an identity matrix, i.e.
  \begin{eqnarray}
     C_{1:k,1:k} & = & I_{k\times k}.\label{eq:genmcproof8}
  \end{eqnarray} \label{lemma:genmclemma3}
\end{lemma}
\begin{proof}
  Let $\lambda_{max}(C)$ be the maximum eigenvalue of $C$. Then
  \begin{eqnarray}
     \lambda_{max}(C) & \triangleq & \max_{\|\c\|_2=1}\c^TC\c.\label{eq:genmcproof9}
  \end{eqnarray}
Since by assumptions $1\leq  C_{i,i},1\leq i\leq k$ and $\lambda_{max}(C)\leq 1$ we also have for any $1\leq i\leq k$
  \begin{eqnarray}
   1\leq  C_{i,i} \leq\max_{\|\c\|_2=1} \c^TC\c\triangleq \lambda_{max}(C) \leq 1,\label{eq:genmcproof10}
  \end{eqnarray}
which implies $C(i,i)=1,1\leq i\leq k$. The proof that all other elements of $C_{1:k,1:k}$ are equal to zero proceeds inductively.

\underline{\textbf{1) Induction move from $l=1$ to $l=2$:}}
First we look at the upper block of size $2\times 2$, i.e. at $C_{1:2,1:2}$. We then have
  \begin{eqnarray}
     1 \geq \max_{\|\c\|_2=1}\c^TC\c \geq \max_{\|\c_{1:2}\|_2=1} \c_{1:2}^TC_{1:2,1:2}\c_{1:2}
     & \geq  & \max_{\|\c_{1:2}\|_2=1} \lp \|\c_{1:2}\|_2+2|\c_1\c_2C_{1,2}|\rp \nonumber \\
     & \geq  & \max_{\|\c_{1:2}\|_2=1} \lp 1+2|\c_1\c_2C_{1,2}|\rp \geq 1,\label{eq:genmcproof11}
  \end{eqnarray}
which implies $C_{1,2}=0$.

\underline{\textbf{2) Induction move from $l$ to $l+1$:}}
Now we look at the upper block of size $(l+1)\times (l+1)$, i.e. at $C_{1:l+1,1:l+1}$ while assuming that $C_{1:l,1:l}=I_{l\times l}$. We then have
  \begin{eqnarray}
     1 & \geq  & \max_{\|\c\|_2=1}\c^TC\c \nonumber \\
      & \geq & \max_{\|\c_{1:l+1}\|_2=1}\c_{1:l+1}^TC_{1:l+1,1:l+1}\c_{1:l+1} \nonumber \\
      & \geq & \max_{\|\c_{1:l+1}\|_2=1} \lp \|\c_{1:l+1}\|_2+2|\c_{1:l}^TC_{1:l,l+1}\c_{l+1}| \rp \nonumber \\
      & \geq & \max_{\|\c_{1:l+1}\|_2=1} \lp 1+2|\c_{1:l}^TC_{1:l,l+1}\c_{l+1}| \rp \nonumber \\
      & \geq & 1,\label{eq:genmcproof12}
  \end{eqnarray}
which implies $C_{1:l,l+1}=\0_{l\times 1}$ and completes the proof.
\end{proof}

\begin{lemma}
Assume the setup of Lemma \ref{lemma:genmclemma3}. Then the upper $k\times k$  left block of $C$, $C_{1:k,1:k}$, is an identity matrix and the upper $k\times (n-k)$  right  block of $C$, $C_{1:k,n-k+1:n}$ is a zero matrix, i.e.
  \begin{eqnarray}
     C_{1:k,1:k} & = & I_{k\times k} \nonumber \\
     C_{1:k,n-k+1:n} & = & \0_{k\times (n-k)}.\label{eq:genmcproof18}
  \end{eqnarray} \label{lemma:genmclemma4}
\end{lemma}
\begin{proof}
The first part follows by Lemma \ref{lemma:genmclemma3}. We now focus on the second part. Consider the following partition of matrix $C$
  \begin{eqnarray}
     C & = & \begin{bmatrix}
               C_{1:k,1:k} & C_{1:k,n-k+1:n} \\
               C_{n-k+1:n,1:k} & C_{n-k+1:n,n-k+1:n}
             \end{bmatrix} =\begin{bmatrix}
               I_{k\times k} & C_{1:k,n-k+1:n} \\
               C_{n-k+1:n,1:k} & C_{n-k+1:n,n-k+1:n}
             \end{bmatrix}.\label{eq:genmcproof14}
  \end{eqnarray}
Then assuming that the largest nonzero singular value of $C_{1:k,n-k+1:n}$ is equal to $b>0$, we have
{\small   \begin{eqnarray}
     1 & \geq  & \max_{\|\c\|_2=1}\c^TC\c \nonumber \\
      & \geq & \max_{\|\c_{1:k}\|_2=a,\c_{n-k+1:n}} \lp \c_{1:k}^TC_{1:k,1:k}\c_{1:k} +2|\c_{1:k}^TC_{1:k,n-k+1:n}\c_{n-k+1:n}| + \c_{n-k+1:n}^TC_{n-k+1:n,n-k+1:n}\c_{n-k+1:n}\rp\nonumber \\
      & \geq & \max_{\|\c_{1:k}\|_2=a,\c_{n-k+1:n}} \lp a^2 +2|\c_{1:k}^TC_{1:k,n-k+1:n}\c_{n-k+1:n}| + \c_{n-k+1:n}^TC_{n-k+1:n,n-k+1:n}\c_{n-k+1:n}\rp\nonumber \\
      & \geq & \max_{\|\c_{1:k}\|_2=a,\c_{n-k+1:n}} \lp a^2 +2|\c_{1:k}^TC_{1:k,n-k+1:n}\c_{n-k+1:n}| - \c_{n-k+1:n}^T\c_{n-k+1:n}\rp \nonumber \\
      & \geq & \max_{a\in[0,1]} \lp a^2 +2ba\sqrt{1-a^2} - (1-a^2) \rp\nonumber \\
            & = & \max_{a\in[0,1]} \lp 2a^2-1 +2ba\sqrt{1-a^2} \rp,\label{eq:genmcproof15}
  \end{eqnarray}}where the fourth inequality follows since the minimum eigenvalue of $C_{n-k+1:n,n-k+1:n}$ is larger than or equal to the minimum eigenvalue of $C$ which is by the lemma's assumption larger than or equal to -1. Now, we further have
  \begin{eqnarray}
  c\triangleq 2a\sqrt{1-a^2} \quad \mbox{and} \quad 2a^2-1 +2ba\sqrt{1-a^2}=\sqrt{1-c^2}+b c,\label{eq:genmcproof16}
  \end{eqnarray}
and
  \begin{eqnarray}
  \frac{d(\sqrt{1-c^2}+b c)}{dc}=\frac{-c}{\sqrt{1-c^2}}+b=0.\label{eq:genmcproof17}
  \end{eqnarray}
From (\ref{eq:genmcproof17}) we then easily obtain
  \begin{eqnarray}
c=\frac{b}{\sqrt{1+b^2}}.\label{eq:genmcproof18}
  \end{eqnarray}
A combination of (\ref{eq:genmcproof15}), (\ref{eq:genmcproof16}), and (\ref{eq:genmcproof18}) gives
  \begin{eqnarray}
     1  \geq   \max_{\|\c\|_2=1}\c^TC\c \geq \max_{a\in[0,1]} \lp 2a^2-1 +2ba\sqrt{1-a^2}\rp  = &\sqrt{1+b^2},\label{eq:genmcproof19}
  \end{eqnarray}
which implies $b=0$ and automatically $C_{1:k,n-k+1:n}=\0_{k\times 1}$. This completes the proof.
\end{proof}

Now we can consider the above mentioned theorem that adapts the general $\ell_1$ equivalence condition result from \cite{StojnicCSetam09,StojnicUpper10,StojnicICASSP09} to the corresponding one for the $\ell_1$ norm of the singular/eigenvalues (similar adaptation can also be found in \cite{OH10}).

\begin{theorem}(\textbf{\bl{$\ell_0^*-\ell_1^*$-equivalence condition (LRR)}} -- \textbf{symmetric}  $X$)
 Consider a $\bU\in\mR^{n\times k}$ such that $\bU^T\bU=I_{k\times k}$ and a $\rankw-k$ \textbf{a priori known to be} symmetric matrix $X_{sol}=X\in\mR^{n\times n}$  with all of its columns belonging to the span of $\bU$. For concreteness, and without loss of generality, assume that $X$ has only positive nonzero eigenvalues. For a given matrix $A\in\mR^{m\times n^2}$ ($m\leq n^2$) assume that $\y=A\vecw(X)\in \mR^m$. If
\begin{equation}
(\forall W\in \mR^{n\times n} | A\vecw(W)=\0_{m\times 1},W=W^T\neq \0_{n\times n}) \quad  -\tr(\bU^TW\bU)< \ell_1^*((\bU^{\perp})^TW\bU^{\perp}),
\label{eq:genmcposthmcond1}
\end{equation}
then the solutions of (\ref{eq:genmcl0pos}) and (\ref{eq:genmcl1pos}) coincide. Moreover, if
\begin{equation}
(\exists  W\in \mR^{n\times n} | A\vecw(W)=\0_{m\times 1},W=W^T\neq \0_{n\times n}) \quad  -\tr(\bU^TW\bU)\geq \ell_1^*((\bU^{\perp})^TW\bU^{\perp}),
\label{eq:genmcposthmcond2}
\end{equation}
then there is an $X$ from the above set of the symmetric matrices with columns belonging to the span of $\bU$  such that the solutions of (\ref{eq:genmcl0pos}) and (\ref{eq:genmcl1pos}) are different.
\label{thm:genmcthmregposcond}
\end{theorem}
\begin{proof}
The proof follows literally step-by-step the proof of the corresponding theorem in \cite{StojnicCSetam09,StojnicICASSP09,StojnicUpper10} and adapts it to matrices or their singular/eigenvalues. For experts in the field this adaptation is highly likely to be viewed as trivial and certainly doesn't need to be as detailed as we will make it to be. Nonetheless, to ensure a perfect clarity of all arguments we provide a step-by-step instructional derivation. For concreteness and without loss of generality we also assume that the eigen-decomposition of $X$ is

\begin{eqnarray}\label{eq:genmcrec1}
X=U\Lambda U^T=\begin{bmatrix}
                \bU & \bU^{\perp}
              \end{bmatrix}
              \begin{bmatrix}
                \bar{\Lambda}_X & \0_{k\times (n-k)} \\
                \0_{(n-k)\times k} & \bar{\Lambda}_X^{\perp}
              \end{bmatrix}
              \begin{bmatrix}
                \bU & \bU^{\perp}
              \end{bmatrix}^T.
\end{eqnarray}

\bl{ \underline{\textbf{(i) $\Longrightarrow$ (the if part):}}} Following step-by-step the proof of Theorem $2$ in \cite{StojnicICASSP09}, we start by  assuming that  $\hat{X}$ is the solution of (\ref{eq:genmcl1pos}). Then we want to show that if (\ref{eq:genmcposthmcond1}) holds then $\hat{X}=X$. As usual, we instead of that, assume opposite, i.e. we assume that (\ref{eq:genmcposthmcond1}) holds but $\hat{X}\neq X$. Then since $\y=A\vecw(\hat{X})$ and $\y=A\vecw(X)$ must hold simultaneously there must exist $W$ such that $\hat{X} =X+W$ with $W\neq 0$, $A\vecw(W)=0$. Moreover, since $\hat{X}$ is the solution of (\ref{eq:genmcl1pos}) one must also have
\begin{eqnarray}
\begin{array}{r r r l@{\ }}
   & \ell_1^*(X+W) = \ell_1^*(\hat{X}) & \leq & \ell_1^*(X) \\
\Longleftrightarrow \hspace{.3in} $ $ & \ell_1^*(\begin{bmatrix}
                \bU & \bU^{\perp}
              \end{bmatrix}^T(X+W)\begin{bmatrix}
                \bU & \bU^{\perp}
              \end{bmatrix}) & \leq & \ell_1^*(X) \\
\Longrightarrow \hspace{.3in} $ $ &\ell_1^*(\bU^T (X+W)\bU)+\ell_1^*((\bU^{\perp})^T (X+W)\bU^{\perp}) & \leq & \ell_1^*(X).
\end{array}\nonumber \\
\label{eq:genmcproof1}
\end{eqnarray}
The last implication follows after one trivially notes
\begin{eqnarray}
 \ell_1^*(\begin{bmatrix}
                \bU & \bU^{\perp}
              \end{bmatrix}^T(X+W)\begin{bmatrix}
                \bU & \bU^{\perp}
              \end{bmatrix}) & = & \max_{\Lambda_*=\Lambda_*^T\in {\cal L}_*}
                            \tr(\Lambda_*
              \begin{bmatrix}
                \bU & \bU^{\perp}
              \end{bmatrix}^T(X+W)\begin{bmatrix}
                \bU & \bU^{\perp}
              \end{bmatrix})\nonumber \\
  & \geq & \max_{\Lambda_*=\Lambda_*^T\in {\cal L}_{*}^0}
              \tr(\Lambda_*
              \begin{bmatrix}
                \bU & \bU^{\perp}
              \end{bmatrix}^T(X+W)\begin{bmatrix}
                \bU & \bU^{\perp}
              \end{bmatrix})\nonumber \\
& = & \ell_1^*(\bU^T (X+W)\bU)+\ell_1^*((\bU^{\perp})^T (X+W)\bU^{\perp}),\label{eq:genmcproof1a}
\end{eqnarray}
where
\begin{eqnarray}
{\cal L}_{*}^0 &\triangleq& \left \{\Lambda_*\in\mR^{n\times n} | \Lambda_*=\Lambda_*^T,\Lambda_*\Lambda_*^T\leq I, \Lambda_*=\begin{bmatrix}
                                                                                                                 \Lambda_{*,1} & 0_{k\times (n-k)} \\
                                                                                                                 0_{(n-k)\times k} & \Lambda_{*,2}
                                                                                                               \end{bmatrix} \right \} \nonumber \\
   &\subseteq& \left \{\Lambda_*\in\mR^{n\times n} | \Lambda_*=\Lambda_*^T,\Lambda_*\Lambda_*^T\leq I\right \} \triangleq   {\cal L}_{*}.\label{eq:genmcproof1b}
\end{eqnarray}

\tcbset{beamer,lower separated=false, fonttitle=\bfseries,width=3.4in, coltext=white,
colback=yellow!70!orange!40!white,title style={left color=cyan!40!black!80!purple, right color=red!60!yellow!40!orange!80!white},
width=(\linewidth-4pt)/4, equal height group=AT,before=,after=\hfill,fonttitle=\bfseries}
\tcbset{colback=red!25!white!70!green!15!yellow,colframe=red!95!white,width=(\linewidth-4pt)/4,
equal height group=AT,before=,after=\hfill,fonttitle=\bfseries,
interior style={left color=cyan!40!black!80!purple, right color=red!60!yellow!40!orange!80!white}}
\noindent\begin{tcolorbox}[width=4.3in, height=.25in]
\vspace{-.27in}
\begin{equation*}
\hspace{-.0in}\mbox{\textbf{The key observation -- \bl{``\emph{Removing the absolute values}"}:}}
\vspace{-0in}
\end{equation*}
\end{tcolorbox}

\vspace{.1in}

\noindent Now, the key observation made in \cite{StojnicICASSP09} comes into play. Namely, one notes that the absolute values can be removed in the nonzero part and that the $\ell_1^*(\cdot)$ can be ``\emph{replaced}" by $\tr(\cdot)$.  Such a simple observation is the most fundamental reason for all the success of the \bl{\textbf{RDT}} when used for the \textbf{exact} performance characterization of the structured objects' recovery. From (\ref{eq:genmcproof1}) we then have
\begin{eqnarray}
\begin{array}{r r r l@{\ }}
 & \ell_1^*(\bU^T (X+W)\bU)+\ell_1^*((\bU^{\perp})^T (X+W)\bU^{\perp}) & \leq & \ell_1^*(X)\\
 \Longrightarrow   \hspace{.3in} $ $ & \tr(\bU^T (X+W)\bU)+\ell_1^*((\bU^{\perp})^T (W)\bU^{\perp}) & \leq & \ell_1^*(X)\\
 \Longleftrightarrow   \hspace{.3in} $ $ & \tr(\bU^T W\bU)+\ell_1^*((\bU^{\perp})^T W\bU^{\perp}) & \leq & 0.
\end{array}\label{eq:genmcproof2}
\end{eqnarray}
We have arrived at a contradiction as the last inequality in (\ref{eq:genmcproof2}) is exactly the opposite of (\ref{eq:genmcposthmcond1}). This implies that our initial assumption $\hat{X}\neq X$ cannot hold and we therefore must have $\hat{X}=X$. This is precisely the claim of the first part of the theorem.

\bl{ \underline{ \textbf{ (ii) $\Longleftarrow$ (the only if part):}}}  We now assume that (\ref{eq:genmcposthmcond2}) holds, i.e.
\begin{equation}
(\exists  W\in \mR^{n\times n} | A\vecw(W)=\0_{m\times 1},W\neq \0_{n\times n}) \quad  -\tr((\bU)^TW\bU)\geq \ell_1^*((\bU^{\perp})^TW\bU^{\perp})\label{eq:genmcproof3}
\end{equation}
and would like to show that for such a $W$ there is a symmetric rank-$k$ matrix $X$ with the columns belonging to the span of $\bU$ such that $\y=A\vecw(X)$,  and the following holds
\begin{equation}
\ell_1^*(X+W)<\ell_1^*(X).\label{eq:genmcproof4}
\end{equation}

Existence of such an $X$ would ensure that it both, satisfies all the constraints in (\ref{eq:genmcl1pos}) and is not the solution of (\ref{eq:genmcl1pos}). Following the strategy of  \cite{StojnicUpper10} one can reverse all the above steps from (\ref{eq:genmcproof3}) to (\ref{eq:genmcproof1}) with strict inequalities and arrive at the first inequality in (\ref{eq:genmcproof1}) which is exactly (\ref{eq:genmcproof4}). There are two implications that cause problems in such a reversal process, the one in (\ref{eq:genmcproof3}) and the one in (\ref{eq:genmcproof1}). If these implications were equivalences everything would be fine. We address these two implications separately.

\underline{1) \textbf{the implication in (\ref{eq:genmcproof2}) -- particular $X$ to ``overwhelm" $W$:}} Assume $X=\bU\Lambda_x\bU^T$ with $\Lambda_x>0$ being a diagonal matrix with arbitrarily large elements on the main diagonal (here it is sufficient even to choose diagonal of $\Lambda_x$ so that its smallest element is larger than the maximum eigenvalue of $\bU^TW\bU$). Now one of course sees the main idea behind the ``removing the absolute values" concept from \cite{StojnicICASSP09,StojnicUpper10}. Namely, for such an $X$ one has that $\ell_1^*(\bU^TX+W)\bU)=tr(\ell_1^*(\bU^TX+W)\bU))$ since for symmetric matrices  the $\ell_1^*(\cdot)$ (as the sum of the argument's \emph{absolute} eigenvalues) and $\tr(\cdot)$ (as the sum of the argument's eigenvalues) are equal. That basically means that when going backwards the second inequality in (\ref{eq:genmcproof2}) not only follows from the first one but also implies it as well. In other words, for $X=\bU\Lambda_x\bU^T$ (with $\Lambda_x>0$ and arbitrarily large)
\begin{eqnarray}
\begin{array}{r r r l@{\ }}
   & \tr(\bU^T W\bU)+\ell_1^*((\bU^{\perp})^T W\bU^{\perp}) & \leq & 0 \\
 \Longleftrightarrow  \hspace{.3in}  $ $ & \tr(\bU^T (X+W)\bU^)+\ell_1^*((\bU^{\perp})^T (W)\bU^{\perp}) & \leq & \ell_1^*(X)\\
 \Longleftrightarrow \hspace{.3in} $ $ & \ell_1^*(\bU^T (X+W)\bU)+\ell_1^*((\bU^{\perp})^T (X+W)\bU^{\perp}) & \leq & \ell_1^*(X),
\end{array}\label{eq:genmcproof5}
\end{eqnarray}
which basically mans that there is an $X$ that can ``overwhelm" $W$ (in the span of $\bU$) and ensures that the \bl{``\textbf{\emph{removing the absolute values}}"} is not only a \textbf{\emph{sufficient}} but also a \textbf{\emph{necessary}} concept for creating the relaxation  equivalence condition.

\underline{2) \textbf{the implication in (\ref{eq:genmcproof1}):}} One would now need to somehow show that the third inequality in (\ref{eq:genmcproof1}) not only follows from the second one but also implies it as well. This boils down to showing that inequality in (\ref{eq:genmcproof1a}) can be replaced with an equality or, alternatively, that ${\cal L}^0$ and ${\cal L}$ are provisionally equivalent. Neither of these statements is generically true. However, since we have a set of $X$ at our disposal there might be an $X$ for which they actually hold. We continue to assume $X=\bU\Lambda_x\bU^T$ with $\Lambda_x>0$ being a diagonal matrix with arbitrarily large entries on the main diagonal. Then the last equality in (\ref{eq:genmcproof1a}) gives
\begin{eqnarray}
\begin{array}{r r r l@{\ }}
  $ $ & \ell_1^*(\bU^T (X+W)\bU)+\ell_1^*((\bU^{\perp})^T (X+W)\bU^{\perp}) & \leq & \ell_1^*(X) \\
  \Longleftrightarrow \hspace{.3in} $ $ & \max_{\Lambda_*=\Lambda_*^T\in {\cal L}_{*}^0}
                            \tr(\Lambda_*
              \begin{bmatrix}
                \bU & \bU^{\perp}
              \end{bmatrix}^T(X+W)\begin{bmatrix}
                \bU & \bU^{\perp}
              \end{bmatrix}) & \leq & \ell_1^*(X).
\end{array}\label{eq:genmcproof6}
\end{eqnarray}
Also, one has
\begin{eqnarray}
\begin{array}{r r r l@{\ }}
  & \max_{\Lambda_*=\Lambda_*^T\in {\cal L}_{*}^0}
                            \tr(\Lambda_*
              \begin{bmatrix}
                \bU & \bU^{\perp}
              \end{bmatrix}^T(X+W)\begin{bmatrix}
                \bU & \bU^{\perp}
              \end{bmatrix}) & \leq & \ell_1^*(X)\\
                 \Longleftrightarrow \hspace{.3in} $ $ & \max_{\Lambda_{*,i}=\Lambda_{*,i}^T,\Lambda_{*,i}\Lambda_{*,i}^T\leq I,i\in\{1,2\}}
                            \tr(\Lambda_{*,1}\bU^T X\bU +\Lambda_{*,2}(\bU^{\perp})^T W\bU^{\perp}) & \leq & \ell_1^*(X) \\
                                             \Longleftrightarrow \hspace{.3in} $ $ & \max_{\Lambda_{*,i}=\Lambda_{*,i}^T,\Lambda_{*,i}\Lambda_{*,i}^T\leq I,i\in\{1,2\}}
                            \tr(\Lambda_{*,1}\Lambda_x +\Lambda_{*,2}(\bU^{\perp})^T W\bU^{\perp}) & \leq & \tr(\Lambda_x).
\end{array}\label{eq:genmcproof6}
\end{eqnarray}
Now, if at least one of the elements on the main diagonal of $\Lambda_{*,1}$, $\diag(\Lambda_{*,1})$, is smaller than 1, then the corresponding element on the diagonal of $\Lambda_x$ can be made arbitrarily large compared to the other elements of $\Lambda_x$ and one would have
\begin{eqnarray}
\begin{array}{r r r l@{\ }}
  & \max_{\Lambda_{*,i}=\Lambda_{*,i}^T,\Lambda_{*,i}\Lambda_{*,i}^T\leq I,i\in\{1,2\}}
                            \tr(\Lambda_{*,1}\Lambda_x +\Lambda_{*,2}(\bU^{\perp})^T W\bU^{\perp}) & < & \tr(\Lambda_x) \\
                            \Longleftrightarrow \hspace{.3in} $ $   & \max_{\Lambda_*=\Lambda_*^T\in {\cal L}_{*}^0}
                            \tr(\Lambda_*
              \begin{bmatrix}
                \bU & \bU^{\perp}
              \end{bmatrix}^T(X+W)\begin{bmatrix}
                \bU & \bU^{\perp}
              \end{bmatrix}) & < & \ell_1^*(X)\\
                                          \Longleftrightarrow \hspace{.3in} $ $   & \max_{\Lambda_*=\Lambda_*^T\in {\cal L}_{*}}
                            \tr(\Lambda_*
              \begin{bmatrix}
                \bU & \bU^{\perp}
              \end{bmatrix}^T(X+W)\begin{bmatrix}
                \bU & \bU^{\perp}
              \end{bmatrix}) & < & \ell_1^*(X),
\end{array}\label{eq:genmcproof7}
\end{eqnarray}
where the last equivalence holds since the difference of the terms on the left-hand side in the last two inequalities is bounded independently of $X$. Also, the last inequality in (\ref{eq:genmcproof7}) together with the first equality in (\ref{eq:genmcproof1a}) and the first inequality in (\ref{eq:genmcproof1}) produces (\ref{eq:genmcproof4}). Therefore the only scenario that is left as potentially not producing (\ref{eq:genmcproof4}) is when all the elements on the main diagonal are larger than or equal to 1. However, the two lemmas preceding the theorem show that in such a scenario ${\cal L}^0={\cal L}$ and one consequently has an equality instead of the inequality in (\ref{eq:genmcproof1a}) which then, together with
(\ref{eq:genmcproof1}), implies (\ref{eq:genmcproof4}). This completes the proof of the second (``the only if") part of the theorem and therefore of the entire theorem.
 \end{proof}

\end{document}